\def\isarxiv{1}

\documentclass[11pt]{article}

\usepackage[utf8]{inputenc} % allow utf-8 input
\usepackage[T1]{fontenc}    % use 8-bit T1 fonts
\usepackage{hyperref}       % hyperlinks
\usepackage{url}            % simple URL typesetting
\usepackage{booktabs}       % professional-quality tables
\usepackage{amsfonts}       % blackboard math symbols
\usepackage{nicefrac}       % compact symbols for 1/2, etc.
\usepackage{microtype}      % microtypography

\usepackage{amsmath}
\usepackage{amsthm}
\usepackage{amssymb}
\usepackage{algorithm}
\usepackage{subfig}
\usepackage{color}
\usepackage[english]{babel}
\usepackage{graphicx}
\usepackage{grffile}
\usepackage{wrapfig,epsfig}
\usepackage{epstopdf}
\usepackage{url}
\usepackage{color}
\usepackage{epstopdf}
\usepackage{algpseudocode}
\usepackage[T1]{fontenc}
\usepackage{bbm}
\usepackage{comment}
\usepackage{dsfont}

\let\C\relax
\usepackage{tikz}
\usepackage{hyperref}  %%% arxiv don't allow this.
\hypersetup{colorlinks=true,citecolor=red,linkcolor=blue}
\usetikzlibrary{arrows}
\ifdefined\isarxiv
\usepackage[margin=1in]{geometry}
\else 

\fi
\graphicspath{{./figs/}}
\usepackage{mathtools}

\newtheorem{theorem}{Theorem}[section]
\newtheorem{lemma}[theorem]{Lemma}
\newtheorem{definition}[theorem]{Definition}

\newtheorem{corollary}[theorem]{Corollary}

\newtheorem{assumption}[theorem]{Assumption}

\newtheorem{remark}[theorem]{Remark}
\newtheorem{claim}[theorem]{Claim}

\newcommand{\wt}{\widetilde}
\newcommand{\ov}{\overline}

\newcommand{\N}{\mathcal{N}}
\newcommand{\R}{\mathbb{R}}

\renewcommand{\i}{\mathbf{i}}
\renewcommand{\varepsilon}{\epsilon}
\renewcommand{\tilde}{\wt}

\renewcommand{\bar}{\overline}

\renewcommand{\d}{\mathrm{d}}

\DeclareMathOperator*{\E}{{\mathbb{E}}}

\DeclareMathOperator*{\C}{\mathbb{C}}

\DeclareMathOperator{\poly}{poly}
\DeclareMathOperator{\Tr}{tr}

\DeclareMathOperator{\dis}{dis}

\DeclareMathOperator{\tr}{tr}

\DeclareMathOperator{\diag}{diag}

\DeclareMathOperator{\cts}{cts}

\DeclareMathOperator{\ntk}{ntk}
\DeclareMathOperator{\nn}{nn}
\DeclareMathOperator{\test}{test}
\DeclareMathOperator{\train}{train}

\DeclareMathOperator{\init}{init}
\DeclareMathOperator{\unif}{unif}
\renewcommand{\k}{\mathsf{K}}

\definecolor{b2}{RGB}{51,153,255}
\definecolor{mygreen}{RGB}{80,180,0}

\makeatletter
\newcommand*{\RN}[1]{\expandafter\@slowromancap\romannumeral #1@}
\makeatother

\title{Generalized Leverage Score Sampling for Neural Networks\thanks{The authors would like to thank Michael Kapralov for suggestion of this topic.}}
\ifdefined\isarxiv
\author{Jason D. Lee\thanks{\texttt{jasonlee@princeton.edu} Princeton University.}
\and Ruoqi Shen\thanks{\texttt{shenr3@cs.washington.edu} University of Washington. Work done while visiting Institute for Advanced Study.}
\and Zhao Song\thanks{\texttt{magic.linuxkde@gmail.com} Columbia University, Princeton University and Institute for Advanced Study.}
\and Mengdi Wang\thanks{\texttt{mengdiw@princeton.edu} Princeton University.}
\and Zheng Yu\thanks{\texttt{zhengy@princeton.edu} Princeton University.}
}
\date{}
\else
\author{
}
\fi

\begin{document}

\ifdefined\isarxiv
\begin{titlepage}
  \maketitle
  \begin{abstract}
  
Leverage score sampling is a powerful technique that originates from theoretical computer science, which can be used to speed up a large number of fundamental questions, e.g. linear regression, linear programming, semi-definite programming, cutting plane method, graph sparsification, maximum matching and max-flow. Recently, it has been shown that leverage score sampling helps to accelerate kernel methods [Avron, Kapralov, Musco, Musco, Velingker and Zandieh 17]. 

In this work, we generalize the results in [Avron, Kapralov, Musco, Musco, Velingker and Zandieh 17] to a broader class of kernels. We further bring the leverage score sampling into the field of deep learning theory.
\begin{itemize}
\item We show the connection between the initialization for neural network training and approximating the neural tangent kernel with random features.
\item We prove the equivalence between regularized neural network and neural tangent kernel ridge regression under the initialization of both classical random Gaussian and leverage score sampling.
\end{itemize}

  \end{abstract}
  \thispagestyle{empty}
\end{titlepage}

\pagenumbering{roman}
{\small
\hypersetup{linkcolor=black}
\tableofcontents
}

\newpage
\else
\maketitle
\begin{abstract}

\end{abstract}
\fi

\pagenumbering{arabic}
\setcounter{page}{1}
\section{Introduction}

Kernel method is one of the most common techniques in various machine learning problems. One classical application is the kernel ridge regression (KRR). Given training data $X = [x_1,\cdots,x_n]^\top \in \R^{n \times d}$, corresponding labels $Y=[y_1, \cdots, y_n] \in \R^n$ and regularization parameter $\lambda>0$, the output estimate of KRR for any given input $z$ can be written as:
\begin{align}\label{eq:intro_krr}
f(z) = \k( z , X )^\top ( K + \lambda I_n )^{-1} Y,
\end{align}
where $\k(\cdot,\cdot)$ denotes the kernel function and $K \in \R^{n\times n}$ denotes the kernel matrix.

Despite being powerful and well-understood, the kernel ridge regression suffers from the costly computation when dealing with large datasets, since generally implementation of Eq.~\eqref{eq:intro_krr} requires $O(n^3)$ running time. Therefore, intensive research have been dedicated to the scalable methods for KRR \cite{b13,am15,zdw15,acw17,mm17,znvkr20}. One of the most popular approach is the random Fourier features sampling originally proposed by \cite{rr08} for shift-invariant kernels. They construct a finite dimensional random feature vector $\phi : \R^d \to \C^s$ through sampling that approximates the kernel function $\k(x,z) \approx \phi(x)^* \phi(z)$ for data $x,z\in\R^d$. The random feature helps approximately solves KRR in $O(ns^2+n^2)$ running time, which improves the computational cost if $s \ll n$. The work \cite{akmmvz17} advanced this result by introducing the leverage score sampling to take the regularization term into consideration. 

In this work, we follow the the approach in \cite{akmmvz17} and naturally generalize the result to a broader class of kernels, which is of the form
\begin{align*}
  \k(x,z) = \E_{w \sim p}[\phi(x,w)^\top\phi(z,w)],
\end{align*}
% \begin{align}\label{eq:lev_kernel_matrix_intro}
%   K_{i,j} = k(x_i,x_j) = \E_{w\sim p}[\phi(x_i,w)^\top\phi(x_j,w)],
% \end{align}
where $\phi : \R^{d} \times \R^{d_1} \to \R^{d_2}$ is a finite dimensional vector and $p : \R^{d_1} \to \R_{\geq 0}$ is a probability distribution. We apply the leverage score sampling technique in this generalized case to obtain a tighter upper-bound on the dimension of random features.

Further, We discuss the application of our theory in neural network training. Over the last two years, there is a long line of over-parametrization theory works on the convergence results of deep neural network \cite{ll18,dzps19,als19a,als19b,dllwz19,adhlw19,adhlsw19,sy19,bpsw20}, all of which either explicitly or implicitly use the property of neural tangent kernel \cite{jgh18}. However, most of those results focus on neural network training without regularization, while in practice regularization (which is originated from classical machine learning) has been widely used in training deep neural network. Therefore, in this work we rigorously build the equivalence between training a ReLU deep neural network with $\ell_2$ regularization and neural tangent kernel ridge regression.
We observe that the initialization of training neural network corresponds to approximating the neural tangent kernel with random features, whose dimension is proportional to the width of the network. Thus, it motivates us to bring the leverage score sampling theory into the neural network training. We present a new equivalence between neural net and kernel ridge regression under the initialization using leverage score sampling, which potentially improves previous equivalence upon the upper-bound of network width needed.

We summarize our main results and contribution as following:
\begin{itemize}
	\item Generalize the leverage score sampling theory for kernel ridge regression to a broader class of kernels.
	\item Connect the leverage score sampling theory with neural network training.
	\item Theoretically prove the equivalence between training regularized neural network and kernel ridge regression under both random Gaussian initialization and leverage score sampling initialization.
\end{itemize}

\section{Related work}
% \vspace{-3mm}
\paragraph{Leverage scores}

Given a $m \times n$ matrix $A$. Let $a_i^\top$ be the $i$-th rows of $A$ and the leverage score of the $i$-th row of $A$ is $\sigma_i(A) = a_i^\top (A^\top A)^{\dagger} a_i$. A row's leverage score measures how important it is in composing the row space of $A$. If a row has a component orthogonal to all other rows, its leverage score is $1$. Removing it would decrease the rank of $A$, completely changing its row space. The coherence of $A$ is $\| \sigma(A) \|_{\infty}$. If $A$ has low coherence, no particular row is especially important. If $A$ has high coherence, it contains at least one row whose removal would significantly affect the composition of $A$'s row space.

Leverage score is a fundamental concept in graph problems and numerical linear algebra. There are many works on how to approximate leverage scores \cite{ss11,dmmw12,cw13,nn13} or more general version of leverages, e.g. Lewis weights \cite{l78,blm89,cp15}. From graph perspective, it has been applied to maximum matching \cite{bln+20,lsz20}, max-flow \cite{ds08,m13_flow,m16,ls20_stoc,ls20_focs}, generate random spanning trees \cite{s18}, and sparsify graphs \cite{ss11}. From matrix perspective, it has been used to give matrix CUR decomposition \cite{bw14,swz17,swz19} and tensor CURT decomposition \cite{swz19}. From optimization perspective, it has been used to approximate the John Ellipsoid \cite{ccly19}, linear programming \cite{ls14,blss20,jswz20}, semi-definite programming \cite{jklps20}, and cutting plane methods \cite{v89,lsw15,jlsw20}.

% \vspace{-4mm}
\paragraph{Kernel methods}

Kernel methods can be thought of as instance-based learners: rather than learning some fixed set of parameters corresponding to the features of their inputs, they instead ``remember'' the $i$-th training example $(x_i,y_i)$ and learn for it a corresponding weight $w_i$. Prediction for unlabeled inputs, i.e., those not in the training set, is treated by the application of similarity function $\k$, called a kernel, between the unlabeled input $x'$ and each of the training inputs $x_i$. 

There are three lines of works that are closely related to our work. First, our work is highly related to the recent discoveries of the connection between deep learning and kernels \cite{dfs16,d17,jgh18,cb18}. Second, our work is closely related to development of connection between leverage score and kernels \cite{rr08,cw17,cmm17,mw17_focs,mw17_nips,ltos18,akmmvz17,akmmvz19,acss20}. Third, our work is related to kernel ridge regression \cite{b13,am15,zdw15,acw17,mm17,znvkr20}.

 %\vspace{-4mm}

\paragraph{Convergence of neural network}

There is a long line of work studying the convergence of neural network with random input assumptions \cite{bg17,t17,zsjbd17,s17,ly17,zsd17,dltps18,glm18,brw19}. For a quite while, it is not known to remove the randomness assumption from the input data points. Recently, there is a large number of work studying the convergence of neural network in the over-parametrization regime \cite{ll18,dzps19,als19a,als19b,dllwz19,adhlw19,adhlsw19,sy19,bpsw20}. These results don't need to assume that input data points are random, and only require some much weaker assumption which is called ``data-separable''. Mathematically, it says for any two input data points $x_i$ and $x_j$, we have $\| x_i - x_j \|_2 \geq \delta$. Sufficiently wide neural network requires the width $m$ to be at least $\poly(n,d,L,1/\delta)$, where $n$ is the number of input data points, $d$ is the dimension of input data point, $L$ is the number of layers.

%\vspace{-4mm}
\paragraph{Continuous Fourier transform}
The continuous Fourier transform is defined as a problem \cite{jls20} where you take samples $f(t_1), \cdots , f(t_m)$ from the time domain $f(t) := \sum_{j=1}^n v_j e^{2\pi \i \langle x_j , t \rangle }$, and try to reconstruct function $f : \R^d \rightarrow \C$ or even recover $\{(v_j,x_j)\} \in \C \times \R^d$. The data separation connects to the sparse Fourier transform in the continuous domain. We can view the $n$ input data points \cite{ll18,als19a,als19b} as $n$ frequencies in the Fourier transform \cite{m15,ps15}. The separation of the data set is equivalent to the gap of the frequency set ($\min_{i \neq j} \| x_i - x_j \|_2 \geq \delta$). In the continuous Fourier transform, there are two families of algorithms: one requires to know the frequency gap \cite{m15,ps15,cm20,jls20} and the other doesn't \cite{ckps16}. However, in the over-parameterized neural network training, all the existing work requires a gap for the data points.

\paragraph{Notations}
We use $\i$ to denote $\sqrt{-1}$.
For vector $x$, we use $\|x\|_2$ to denote the $\ell_2$ norm of $x$. For matrix $A$, we use $\|A\|$ to denote the spectral norm of $A$ and $\|A\|_F$ to denote the Frobenius norm of $A$. For matrix $A$ and $B$, we use $A \preceq B$ to denote that $B-A$ is positive semi-definite.  For a square matrix, we use $\tr[A]$ to denote the trace of $A$. We use $A^{-1}$ to denote the true inverse of an invertible matrix. We use $A^\dagger$ to denote the pseudo-inverse of matrix $A$. We use $A^\top$ to denote the transpose of matrix $A$. 

\section{Main results}

% {\color{red} 
% \begin{enumerate}
% 	\item Clean the repeated definition
% 	\item Explain the generalization to broader class of models
% 	\item Explain the benefits of introducing leverage score sampling
% 	\item 8 page
% \end{enumerate}
% }

In this section, we state our results. In Section~\ref{subsubsec:leverage_sampling}, we consider the large-scale kernel ridge regression (KRR) problem. We generalize the Fourier transform result \cite{akmmvz17} of accelerating the running time of solving KRR using the tool of leverage score sampling to a broader class of kernels. In Section~\ref{subsec:app}, we discuss the interesting application of leverage score sampling for training deep learning models due to the connection between regularized neural nets and kernel ridge regression.

% to show that leverage score can be used in general kernel ridge regression.

% In Section~\ref{subsubsec:equivalence1}, we show the equivalence between training regularized neural network and (neural tangent) kernel ridge regression. In Section~\ref{subsubsec:leverage_sampling}, we generalize the result of \cite{akmmvz17} to show that leverage score can be used in general kernel ridge regression. In Section~\ref{subsubsec:equivalence2}, we further extend the result of Section~\ref{subsubsec:leverage_sampling} to show the equivalence between training leverage score initialized neural network and kernel ridge regression. In Section~\ref{sec:extend}, we discuss the extension of our results to a broader class of neural network models.

\subsection{Kernel approximation with leverage score sampling}
\label{subsubsec:leverage_sampling}
In this section, we generalize the leverage score theory in \cite{akmmvz17}, which analyzes the number of random features needed to approximate kernel matrix under leverage score sampling regime for the kernel ridge regression task. In the next a few paragraphs, we briefly review the settings of classical kernel ridge regression.

Given training data given training data matrix $X=[x_1, \cdots, x_n]^\top \in \R^{n \times d}$, corresponding labels $Y=[y_1,\cdots,y_n]^\top \in \R^n$ and feature map $\phi : \R^d \to \mathcal{F}$, a classical kernel ridge regression problem can be written as\footnote{Strictly speaking, the optimization problem should be considered in a hypothesis space defined by the reproducing kernel Hilbert space associated with the feature/kernel. Here, we use the notation in finite dimensional space for simplicity.}
\begin{align*}
	\min_{\beta} \frac{1}{2}\| Y - \phi(X)^\top \beta \|_2^2 + \frac{1}{2}\lambda\|\beta\|_2^2
\end{align*}
where $\lambda>0$ is the regularization parameter. By introducing the corresponding kernel function $\k(x,z) = \langle \phi(x), \phi(z) \rangle$ for any data $x , z \in \R^d$, the output estimate of the kernel ridge regression for any data $x \in \R^d$ can be denoted as $f^*(x) = \k(x,X)^\top \alpha$,
% \begin{align*}
% u^*(x) = k(x,X)^\top \alpha,
% \end{align*}
where $\alpha \in \R^n$ is the solution to 
\begin{align*}
(K + \lambda I_n) \alpha = Y.
\end{align*}
Here $K \in \R^{n \times n}$ is the kernel matrix with $K_{i,j} = \k(x_i,x_j)$, $\forall i, j \in [n] \times [n]$.

Note a direct computation involves $(K+\lambda I_n)^{-1}$, whose $O(n^3)$ running time can be fairly large in tasks like neural network due to the large number of training data. Therefore, we hope to construct feature map $\phi : \R^d \to \R^{s}$, such that the new feature approximates the kernel matrix well in the sense of 
\begin{align}\label{eq:leverage_goal}
 (1-\epsilon) \cdot (K + \lambda I_n) \preceq \Phi \Phi^\top + \lambda I_n \preceq (1+\epsilon) \cdot ( K + \lambda I_n ),
\end{align}
where $\epsilon\in(0,1)$ is small and $\Phi =  [ \phi(x_1),\cdots,\phi(x_n)]^\top \in\R^{n\times s}$. Then by Woodbury matrix equality, we can approximate the solution by $u^*(z)= \phi (z)^\top ( \Phi^\top \Phi +\lambda I_s)^{-1} \Phi^\top Y $, which can be computed in $O(ns^2+n^2)$ time. In the case $s= o(n)$, computational cost can be saved.

%To make the problem achievable, 
In this work, we consider a generalized setting of \cite{akmmvz17} as a kernel ridge regression problem with positive definite kernel matrix $\k: \R^d \times \R^d \to \R$ of the form
\begin{align}\label{eq:lev_kernel_intro}
  \k(x,z) = \E_{w \sim p}[\phi(x,w)^\top\phi(z,w)],
\end{align}
% \begin{align}\label{eq:lev_kernel_matrix_intro}
%   K_{i,j} = k(x_i,x_j) = \E_{w\sim p}[\phi(x_i,w)^\top\phi(x_j,w)],
% \end{align}
where $\phi:\R^{d} \times \R^{d_1}\to \R^{d_2}$ denotes a finite dimensional vector and $p:\R^{d_1}\to\R_{\geq 0}$ denotes a probability density function. 
% Then given training data $\{(x_i,y_i),~i\in[n]\}$ and regularization parameter $\lambda>0$, the optimal predictor for any data $z$ can be written as $u^*(z) = k(z,X)^\top \alpha$, where $\alpha\in\R^n$ solves $(K+\lambda I_n)\alpha = Y$ and $K\in\R^{n\times n}$ is the kernel matrix with $K_{i,j} = k(x_i,x_j)$.

% Note a direct computation involves $(K+\lambda I_n)^{-1}$, whose $O(n^3)$ running time can be fairly large in tasks like neural network due to the large number of training data. Therefore, we hope to construct feature map $\bar{\phi}:\R^d\to\R^{s}$, such that the new feature approximates the kernel matrix well in the sense of 
% \begin{align*}
%  (1-\epsilon) \cdot (K + \lambda I_n) \preceq \bar{\Psi} \bar{\Psi}^\top + \lambda I_n \preceq (1+\epsilon) \cdot ( K + \lambda I_n ),
% \end{align*}
% where $\epsilon\in(0,1)$ and $\bar{\Psi} =  [\bar{\varphi}(x_1),\cdots,\bar{\varphi}(x_n)]^\top \in\R^{n\times s}$. Then by Woodbury matrix equality, we can approximate the solution by $u^*(z)=\bar{\phi}(z)^\top (\bar{\Phi}^\top \bar{\Phi} +\lambda I_s)^{-1}\bar{\Phi}^\top Y $, which can be computed in $O(ns^2)$ time. In the case $s\ll n$, computational cost can be saved.
% \begin{align*}
% u^*(z) = k(z,X)^\top \alpha,
% \end{align*}
% where $\alpha$ solves $(K+\lambda I_n)\alpha = Y$
% \begin{align*}
% (K+\lambda I_n)\alpha = Y
% \end{align*}

Due to the regularization $\lambda>0$ in this setting, instead of constructing the feature map directly from the distribution $q$, we consider the following ridge leveraged distribution:

% While calculating the optimal solution directly using kernel matrix involves the calculation of $(K+\lambda I_n)^{-1}$, whose $O(n^3)$ running time can be fairly large in tasks like neural network due to the large number of training data. We hope to construct feature $\Psi\in\R^{n\times s}$, such that 

% kernel function $k: \R^d \times \R^d \to \R$ of the form
% \begin{align}\label{eq:lev_kernel_intro}
%   \k(x,z) = \E_{w \sim p}[\phi(x,w)^\top\phi(z,w)],
% \end{align}
% for any data $x,z\in\R^d$, where $\phi:\R^{d} \times \R^{d_1}\to \R^{d_2}$ denotes a finite dimensional vector and $p:\R^{d_1}\to\R_{\geq 0}$ denotes probability density function. Given data $x_1,\cdots,x_n\in\R^d$, we define the corresponding kernel matrix $K \in \R^{n\times n}$ as
% \begin{align}\label{eq:lev_kernel_matrix_intro}
%   K_{i,j} = k(x_i,x_j) = \E_{w\sim p}[\phi(x_i,w)^\top\phi(x_j,w)].
% \end{align}

% Under this setting, we define the leverage score as
\begin{definition}[Ridge leverage function]\label{def:leverage_score_intro}
Given data $x_1,\cdots,x_n\in\R^d$ and parameter $\lambda >0$, we define the ridge leverage function as
\begin{align*}
  q_\lambda(w) = p(w) \cdot \Tr[\Phi(w)^\top ( K + \lambda I_n )^{-1} \Phi(w)],
\end{align*}
where $p(\cdot)$,~$\phi$ are defined in Eq.~\eqref{eq:lev_kernel_intro}, and $\Phi(w) = [\phi(x_1,w)^\top, \cdots, \phi(x_n,w)^\top]^\top\in\R^{n\times d_2}$. Further, we define statistical dimension $s_{\lambda}(K)$ as
\begin{align}\label{eq:stat_dim_lev}
	s_{\lambda}(K) = \int q_{\lambda}(w) \d w = \Tr[(K+\lambda I_n)^{-1} K].
\end{align}
\end{definition}
The leverage score sampling distribution $q_{\lambda}/s_{\lambda}(K)$ takes the regularization term into consideration and achieves Eq.~\eqref{eq:leverage_goal} using the following modified random features vector:
\begin{definition}[Modified random features]\label{def:modify_random_feature_lev}
Given any probability density function $q(\cdot)$ whose support includes that of $p(\cdot)$. Given $m$ vectors $w_1, \cdots, w_m \in \R^{d_1}$, we define modified random features $\bar{\Psi}\in\R^{n\times md_2}$ as
$
  \bar{\Psi} :=  [\bar{\varphi}(x_1),\cdots,\bar{\varphi}(x_n)]^\top,
$
where 
\begin{align*}
 \bar{\varphi}(x) = \frac{1}{\sqrt{m}} \left[ \frac{\sqrt{p(w_1)}}{ \sqrt{q(w_1)} }\phi(x,w_1)^\top , \cdots, \frac{\sqrt{p(w_m)}}{ \sqrt{q(w_m)} }\phi(x,w_m)^\top \right]^\top.
\end{align*}
\end{definition}
Now we are ready to present our result.
\begin{theorem}[Kernel approximation with leverage score sampling, generalization of Lemma 8 in \cite{akmmvz17}]\label{thm:leverage_score_intro}
Given parameter $\lambda \in (0,\|K\|)$. Let $q_\lambda:\R^{d_1}\to\R_{\geq 0}$ be the leverage score defined in Definition~\ref{def:leverage_score_intro}. Let $\tilde{q}_{\lambda}:\R^{d_1}\rightarrow \R$ be any measurable function such that $\tilde{q}_{\lambda}(w) \geq q_\lambda(w)$ holds for all $w\in \mathbb{R}^{d_1}$. Assume $ s_{\tilde{q}_\lambda} = \int_{\mathbb{R}^{d_1}} \tilde{q}_{\lambda}(w)\d w$ is finite. Let $\bar{q}_{\lambda}(w)=\tilde{q}_{\lambda}(w)/s_{\tilde{q}_\lambda}$. Given any accuracy parameter $\epsilon \in (0,1/2)$ and failure probability $\delta \in ( 0 , 1)$. Let $w_1,\cdots,w_m \in \R^d$ denote $m$ samples draw independently from the distribution associated with the density $\bar{q}_{\lambda}(\cdot)$, and construct the modified random features $\bar{\Psi} \in \R^{n \times md_2}$ as in Definition~\ref{def:modify_random_feature_lev} with $q=\bar{q}_{\lambda}$. Let $s_\lambda(K)$ be the statistical dimension defined in~\eqref{eq:stat_dim_lev}.
If $m \geq 3 \epsilon^{-2} s_{\tilde{q}_\lambda} \ln(16s_{\tilde{q}_\lambda}\cdot s_\lambda(K) / \delta)$, then we have
{\small
\begin{align}\label{eq:leverage_score_thm_intro}
 (1-\epsilon) \cdot (K + \lambda I_n) \preceq \bar{\Psi} \bar{\Psi}^\top + \lambda I_n \preceq (1+\epsilon) \cdot ( K + \lambda I_n )
\end{align}}
 holds with probability at least $1-\delta$. 
\end{theorem}
\begin{remark}
Above results can be generalized to the complex domain $\C$. Note for the random Fourier feature case discussed in \cite{akmmvz17}, we have $d_1 = d$, $d_2 = 1$, $\phi(x,w) = e^{ -2 \pi \i w^\top x } \in \C$ and $p(\cdot)$ denotes the Fourier transform density distribution, which is a special case in our setting.
\end{remark}

\subsection{Application in training regularized neural network}\label{subsec:app}
In this section, we consider the application of leverage score sampling in training $\ell_2$ regularized neural networks.

Past literature such as \cite{dzps19},\cite{adhlsw19} have already witnessed the equivalence between training a neural network and solving a kernel regression problem in a broad class of network models. In this work, we first generalize this result to the regularization case, where we connect regularized neural network with kernel ridge regression. Then we apply the above discussed the leverage score sampling theory for KRR to the task of training neural nets.

\subsubsection{Equivalence I, training with random Gaussian initialization}
\label{subsubsec:equivalence1}
% In this work, we establish the equivalence between training neural network with $\ell_2$ regularization and kernel ridge regression with respect to neural tangent kernel (NTK). 
To illustrate the idea, we consider a simple model two layer neural network with ReLU activation function as in \cite{dzps19,sy19}\footnote{Our results directly extends to multi-layer deep neural networks with all layers trained together}. 
\begin{align*}
f_{\nn} (W, a, x) = \frac{1}{\sqrt{m}} \sum_{r=1}^m a_r \sigma (w_r^\top x) \in \R,
\end{align*}
where $x \in \R^d$ is the input, $w_r \in \R^d,~r\in[m]$ is the weight vector of the first layer, $W = [w_1, \cdots, w_m]\in\R^{d \times m}$, $a_r \in \R,~r\in[m]$ is the output weight, $a = [a_1, \cdots, a_m]^\top$ and $\sigma(\cdot)$ is the ReLU activation function: $\sigma(z) = \max\{0,z\}$. 

% And we discuss how to extend our results to a broader class of neural network models in Section~\ref{sec:extend}.

% We first define our problem and then present the equivalence result in Theorem~\ref{thm:main_train_equivalence_intro} and Theorem~\ref{thm:main_test_equivalence_intro}. For simplicity, we consider a two layer network with ReLU activation function as following.

% \begin{definition}[Neural network function]\label{def:f_nn_intro}
% We define a two layer neural networks with rectified linear unit (ReLU) activation as the following form
% \begin{align*}
% f_{\nn} (W, a, x) = \frac{1}{\sqrt{m}} \sum_{r=1}^m a_r \sigma (w_r^\top x) \in \R,
% \end{align*}
% where $x \in \R^d$ is the input, $w_r \in \R^d,~r\in[m]$ is the weight vector of the first layer, $W = [w_1, \cdots, w_m]\in\R^{d \times m}$, $a_r \in \R,~r\in[m]$ is the output weight, $a = [a_1, \cdots, a_m]^\top$ and $\sigma(\cdot)$ is the ReLU activation function: $\sigma(z) = \max\{0,z\}$.
% \end{definition}
Here we consider only training the first layer $W$ with fixed $a$, so we also write $f_{\nn}(W,x) = f_{\nn}(W, a, x)$. Again, given training data matrix $X=[x_1,\cdots,x_n]^\top\in\R^{n\times d}$ and labels $Y=[y_1,\cdots,y_n]^\top\in\R^n$, we denote $f_{\nn}(W, X) = [f_{\nn}(W, x_1),\cdots, f_{\nn}(W, x_n)]^\top\in\R^n$. We formally define training neural network with $\ell_2$ regularization as follows:

\begin{definition}[Training neural network with regularization]\label{def:nn_intro}
Let $\kappa\in(0,1]$ be a small multiplier\footnote{To establish the training equivalence result, we assign $\kappa = 1$ back to the normal case. For the training equivalence result, we pick $\kappa>0$ to be a small multiplier only to shrink the initial output of the neural network. The is the same as what is used in \cite{akmmvz17}.}. Let $\lambda\in(0,1)$ be the regularization parameter. We initialize the network as $a_r\overset{i.i.d.}{\sim} \unif[\{-1,1\}]$ and $w_r(0)\overset{i.i.d.}{\sim} \N(0,I_d)$. Then we consider solving the following optimization problem using gradient descent:
\begin{align}\label{eq:nn_intro}
\min_{W} \frac{1}{2}\| Y - \kappa f_{\nn}(W,X) \|_2 + \frac{1}{2}\lambda\|W\|_F^2.
\end{align}
Let $w_r(t),r\in[m]$ be the network weight at iteration $t$. We denote the training data predictor at iteration $t$ as
% \begin{align}\label{eq:nn_predict_train_intro} 
$u_{\nn}(t) = \kappa f_{\nn}(W(t),X) \in\R^n$.
% \end{align}
Further, given any test data $x_{\test}\in\R^d$, we denote 
% \begin{align}\label{eq:nn_predict_test_intro} 
$u_{\nn,\test}(t) = \kappa f_{\nn}(W(t),x_{\test}) \in \R$
% \end{align}
as the test data predictor at iteration $t$.
\end{definition}
% \begin{remark}
% To establish the training equivalence result, we assign $\kappa = 1$ back to the normal case. For the training equivalence result, we pick $\kappa>0$ to be a small multiplier only to shrink the initial output of the neural network. The is the same as what is used in \cite{akmmvz17}. 
% \end{remark}
On the other hand, we consider the following neural tangent kernel ridge regression problem:
\begin{align}\label{eq:krr_intro}
\min_{\beta} \frac{1}{2}\| Y - \kappa f_{\ntk}(\beta,X) \|_2^2 + \frac{1}{2}\lambda\|\beta\|_2^2,
\end{align}
where $\kappa,\lambda$ are the same parameters as in Eq.~\eqref{eq:nn_intro}, $f_{\ntk}(\beta,x) = \Phi(x)^\top \beta \in \R$ and $f_{\ntk}(\beta,X) = [f_{\ntk}(\beta,x_1),\cdots,f_{\ntk}(\beta,x_n)]^\top\in\R^{n}$ are the test data predictors. Here, $\Phi$ is the feature map corresponding to the neural tangent kernel (NTK):
\begin{align}\label{eq:ntk_lev}
	\k_{\ntk}(x, z) = \E \left[\left\langle \frac{\partial f_{\nn}(W,x)}{\partial W},\frac{\partial f_{\nn}(W,z)}{\partial W} \right\rangle \right]
\end{align}
where $x,z \in \R^d$ are any input data, and the expectation is taken over $w_r\overset{i.i.d.}{\sim} \N(0,I),~r=1, \cdots, m$. 

Under the standard assumption $\k_{\ntk}$ being positive definite, the problem Eq.~\eqref{eq:krr_intro} is a strongly convex optimization problem with the optimal predictor $u^*=\kappa^2 H^{\cts}(\kappa^2 H^{\cts}+\lambda I)^{-1}Y$ for training data, and the corresponding predictor $u_{\test}^* = \kappa^2 \k_{\ntk}(x_{\test}, X)^\top(\kappa^2 H^{\cts}+\lambda I)^{-1}Y$ for the test data $x_{\test}$, where $H^{\cts}\in\R^{n\times n}$ is the kernel matrix with $[H^{\cts}]_{i,j} = \k_{\ntk}(x_i, x_j)$.

We connect the problem Eq.~\eqref{eq:nn_intro} and Eq.~\eqref{eq:krr_intro} by building the following equivalence between their training and test predictors with polynomial widths:

\begin{theorem}[Equivalence between training neural net with regularization and kernel ridge regression for training data prediction]\label{thm:main_train_equivalence_intro}
% Given training data matrix $X \in \R^{n \times d}$ and corresponding label vector $Y \in \R^n$. 
Given any accuracy $\epsilon \in ( 0 , 1/10 )$ and failure probability $\delta \in (0,1/10)$. Let multiplier $\kappa = 1$, number of iterations $T=\wt{O}(\frac{1}{\Lambda_0})$, network width $m \geq \wt{O}(\frac{n^4d}{\Lambda_0^4\epsilon})$ and the regularization parameter $\lambda \leq \wt{O}(\frac{1}{\sqrt{m}})$. Then with probability at least $1-\delta$ over the Gaussian random initialization, we have
\begin{align*}
	\|u_{\nn}(T) - u^*\|_2 \leq \epsilon.
\end{align*}
Here $\wt{O}(\cdot)$ hides $\poly\log(n/(\epsilon \delta \Lambda_0 ))$.
\end{theorem}
We can further show the equivalence between the test data predictors with the help of the multiplier $\kappa$.
\begin{theorem}[Equivalence between training neural net with regularization and kernel ridge regression for test data prediction]\label{thm:main_test_equivalence_intro}
% Given training data matrix $X \in \R^{n \times d}$ and corresponding label vector $Y \in \R^n$. Given arbitrary test data $x_{\test} \in \R^d$. 
Given any accuracy $\epsilon \in (0,1/10)$ and failure probability $\delta \in (0,1/10)$. Let multiplier $\kappa = \wt{O}(\frac{\epsilon\Lambda_0}{n})$, number of iterations $T=\wt{O}(\frac{1}{\kappa^2\Lambda_0})$, network width $m \geq \wt{O}(\frac{n^{10}d}{\epsilon^6\Lambda_0^{10}})$ and regularization parameter $\lambda \leq \wt{O}(\frac{1}{\sqrt{m}})$. Then with probability at least $1-\delta$ over the Gaussian random initialization, we have
\begin{align*}
\| u_{\nn,\test}(T) - u_{\test}^* \|_2 \leq \epsilon.
\end{align*}
Here $\wt{O}(\cdot)$ hides $\poly\log(n/(\epsilon \delta \Lambda_0 ))$.
\end{theorem}

\subsubsection{Equivalence II, training with leverage scores}
\label{subsubsec:equivalence2}
To apply the leverage score theory discussed in Section~\ref{subsubsec:leverage_sampling}, Note the definition of the neural tangent kernel is exactly of the form:
\begin{align*}
    \k_{\ntk}(x, z)  = \E \left[\left\langle \frac{\partial f_{\nn}(W,x)}{\partial W},\frac{\partial f_{\nn}(W,z)}{\partial W} \right\rangle \right] 
     = \E_{w \sim p}[\phi(x,w)^\top\phi(z,w)]
\end{align*}
where $\phi(x,w) = x\sigma'(w^\top x)\in\R^{d}$ and $p(\cdot)$ denotes the probability density function of standard Gaussian distribution $\N(0,I_d)$. Therefore, we try to connect the theory of training regularized neural network with leverage score sampling. Note the width of the network corresponds to the size of the feature vector in approximating the kernel. Thus, the smaller feature size given by the leverage score sampling theory helps us build a smaller upper-bound on the width of the neural nets.

Specifically, given regularization parameter $\lambda>0$, we can define the ridge leverage function with respect to neural tangent kernel $H^{\cts}$ defined in Definition~\ref{def:leverage_score_intro} as
\begin{align*}
	q_\lambda(w) = p(w) \Tr[\Phi(w)^\top ( H^{\cts} + \lambda I_n )^{-1} \Phi(w)]
\end{align*}
and corresponding probability density function 
\begin{align}\label{eq:lev_dis_intro}
q(w)=\frac{q_{\lambda}(w)}{s_{\lambda}(H^{\cts})}
\end{align}
where $\Phi(w) = [\phi(x_1,w)^\top, \cdots, \phi(x_n,w)^\top]^\top\in\R^{n\times d_2}$.

We consider training the following reweighed neural network using leverage score initialization:
\begin{definition}[Training reweighed neural network with regularization]\label{def:f_nn_lev_intro}
Let $\kappa\in(0,1]$ be a small multiplier. Let $\lambda\in(0,1)$ be the regularization parameter. Let $q(\cdot):\R^d\to\R_{> 0}$ defined in~\eqref{eq:lev_dis_intro}. Let $p(\cdot)$ denotes the probability density function of Gaussian distribution $\N(0,I_d)$. We initialize the network as $a_r\overset{i.i.d.}{\sim} \unif[\{-1,1\}]$ and $w_r(0)\overset{i.i.d.}{\sim} q$. Then we consider solving the following optimization problem using gradient descent:
\begin{align}\label{problem:nn_lev_intro}
\min_{W} \frac{1}{2}\| Y - \kappa \bar{f}_{\nn}(W,X) \|_2 + \frac{1}{2}\lambda\|W\|_F^2.
\end{align}
where 
\begin{align*}
\bar{f}_{\nn}(W,x) = \frac{1}{\sqrt{m}} \sum_{r=1}^m a_r \sigma (w_r^\top X) \sqrt{\frac{p(w_r(0))}{q(w_r(0))}} \text{~and~}\bar{f}_{\nn}(W,X) = [\bar{f}_{\nn}(W,x_1),\cdots,\bar{f}_{\nn}(W,x_n)]^\top .
\end{align*} 

We denote $w_r(t),r\in[m]$ as the estimate weight at iteration $t$. We denote $\bar{u}_{\nn}(t) = \kappa \bar{f}_{\nn}(W(t),X)$ as the training data predictor at iteration $t$. Given any test data $x_{\test}\in\R^d$, we denote $\bar{u}_{\nn,\test}(t) = \kappa \bar{f}_{\nn}(W(t),x_{\test})$ as the test data predictor at iteration $t$.
% \begin{align}\label{eq:nn_predict_train_lev_intro}
% 	\bar{u}_{\nn}(t) = \kappa \bar{f}_{\nn}(W(t),X)
% % \end{align}
% as the training data predictor at iteration $t$. Given any test data $x_{\test}\in\R^d$, we denote
% \begin{align}\label{eq:nn_predict_test_lev_intro}
% 	\bar{u}_{\nn,\test}(t) = \kappa \bar{f}_{\nn}(W(t),x_{\test})
% \end{align}
% as the test data predictor at iteration $t$.
\end{definition}
We show that training this reweighed neural net with leverage score initialization is still equivalence to the neural tangent kernel ridge regression problem~\eqref{eq:krr_intro} as in following theorem:
\begin{theorem}[Equivalence between training reweighed neural net with regularization and kernel ridge regression for training data prediction]\label{thm:equivalence_train_lev_intro}
% Given training data matrix $X \in \R^{n \times d}$ and corresponding label vector $Y \in \R^n$. Let $q(\cdot):\R^d\to\R_{> 0}$ defined in~\eqref{eq:lev_dis_intro}. Let $p(\cdot)$ denotes the probability density function of Gaussian distribution $\N(0,I_d)$. Let $\bar{u}_{\nn}(t) \in \R^n$ and $u^* \in \R^n$ be the training data predictors defined in~\eqref{eq:nn_predict_train_lev_intro} and~\eqref{eq:def_u_*_intro} respectively.
Given any accuracy $\epsilon\in(0,1)$ and failure probability $\delta\in(0,1/10)$. Let multiplier $\kappa=1$, number of iterations $T=O(\frac{1}{\Lambda_0}\log(\frac{1}{\epsilon}))$, network width $m = \poly(\frac{1}{\Lambda_0},n,d,\frac{1}{\epsilon},\log(\frac{1}{\delta}))$ and regularization parameter $\lambda = \wt{O}(\frac{1}{\sqrt{m}})$. Then with probability at least $1-\delta$ over the random leverage score initialization, we have
\begin{align*}
	\|\bar{u}_{\nn}(T) - u^*\|_2 \leq \epsilon.
\end{align*}
Here $\wt{O}(\cdot)$ hides $\poly\log(n/(\epsilon \delta \Lambda_0 ))$.
\end{theorem}
% \subsection{Extension to other neural network models}\label{sec:extend}
% In previous sections
%%% Section 3. Main Results
\section{Overview of techniques}
\paragraph{Generalization of leverage score theory}

To prove Theorem~\ref{thm:leverage_score_intro}, we follow the similar proof framework as Lemma 8 in \cite{akmmvz17}. 

Let $K+\lambda I_n = V^\top \Sigma^2 V$ be an eigenvalue decomposition of $K+\lambda I_n$. Then conclusion~\eqref{eq:leverage_score_thm_intro} is equivalent to
\begin{align}\label{eq:57_2_intro}
    \| \Sigma^{-1}V\bar{\Psi}\bar{\Psi}^\top V^\top \Sigma^{-1} - \Sigma^{-1}VKV^\top \Sigma^{-1}\| \leq \epsilon
\end{align} 
Let random matrix $Y_r\in\R^{n\times n}$ defined as
\begin{align*}
    Y_r := \frac{p(w_r)}{\bar{q}_{\lambda}(w_r)} \Sigma^{-1}V{\Phi}(w_r){\Phi}(w_r)^\top V^\top \Sigma^{-1}.
\end{align*}
where $\Phi(w)=[\phi(x_1,w),\cdots,\phi(x_n,w)]^\top\in\R^{n\times d_2}$.
% Consider kernel $\K(x,z) = \E_{w \sim p}[\phi(x,w)^\top\phi(z,w)]$ as defined in~\eqref{eq:lev_kernel_intro}. Denote $\Phi(w)=[\phi(x_1,w),\cdots,\phi(x_n,w)]^\top\in\R^{n\times d_2}$.
Then we have
\begin{align*}
    \E_{\bar{q}_{\lambda}}[Y_l] = \E_{\bar{q}_{\lambda}} \left[ \frac{p(w_r)}{\bar{q}_{\lambda}(w_r)}\Sigma^{-1}V{\Phi}(w_r)\bar{\Phi}(w_r)^\top V^\top \Sigma^{-1} \right]
    =  \Sigma^{-1}V K  V^\top \Sigma^{-1},
\end{align*}
and
\begin{align*}
    \frac{1}{m}\sum_{r=1}^m Y_r =  \frac{1}{m}\sum_{r=1}^m \frac{p(w_r)}{\bar{q}_{\lambda}(w_r)}\Sigma^{-1}V{\Phi}(w_r)\bar{\Phi}(w_r)^\top V^\top \Sigma^{-1}
    =  \Sigma^{-1}V \bar{\Psi}\bar{\Psi}^\top  V^\top \Sigma^{-1}.
\end{align*}
Thus, it suffices to show that 
\begin{align}\label{eq:57_3}
    \left\| \frac{1}{m}\sum_{r=1}^m Y_r -\E_{\bar{q}_{\lambda}}[Y_l] \right\| \leq \epsilon
\end{align}
holds with probability at least $1-\delta$, which can be shown by applying matrix concentration results. Note
\begin{align*}
    \|Y_l\| \leq  s_{\bar{q}_{\lambda}}~~\text{and}~
    \E_{\bar{q}_{\lambda}}[Y_r^2] \preceq s_{\tilde{q}_{\lambda}} \cdot \diag\{\lambda_1/(\lambda_1 +\lambda),\cdots, \lambda_n/(\lambda_n +\lambda)\} .
\end{align*}
Applying matrix concentration Lemma 7 in \cite{akmmvz17}, we complete the proof.

\paragraph{Equivalence between regularized neural network and kernel ridge regression}

To establish the equivalence between training neural network with regularization and neural tangent kernel ridge regression, the key observation is that the dynamic kernel during the training is always close to the neural tangent kernel.
% we refer to the key observation in characterizing the convergence property of neural network in the standard case without regularization: random initialization and over-parametrization (\cite{dzps19},\cite{adhlsw19}, etc.).

Specifically, given training data $x_1,\cdots,x_n\in\R^d$, we define the dynamic kernel matrix $H(t)\in\R^{n\times n}$ along training process as
\begin{align*}
    [H(t)]_{i,j}
    = \left\langle \frac{\d f_{\nn}(W(t),x_i)}{\d W(t)},\frac{\d f_{\nn}(W(t),x_j)}{\d W(t)} \right\rangle
\end{align*}
Then we can show the gradient flow of training regularized neural net satisfies
% Recall the gradient flow of training the neural network in the standard case without regularization satisfies:
% \begin{align*}
% 	\frac{\d}{\d t} \|Y-u_{\nn}(t)\|_2^2 = -2(Y-u_{\nn}(t))^\top H(t)(Y-u_{\nn}(t))
% \end{align*}
% which shows the linear convergence of the training error $ \|Y-u_{\nn}(t)\|_2$ to 0 as long as $\Lambda_{\min}(H(t)) \geq b >0,~\forall t$. It has been proved $\Lambda_{\min}(H(t)) \geq \Lambda_{\min}(H^{\cts})/2 = \Lambda_0/2 > 0$ by combining the following two observations:
% \begin{itemize}
% 	\item Random initialization: $H(0)$ is sufficiently close to $H^{\cts}$ if the network is sufficiently wide;
% 	\item Over-parametrization: if the net is wide enough, the weight estimate $W(t)$ at time $t$ will be sufficiently close to its initialization $W(0)$, which implies the dynamic kernel $H(t)$ being sufficiently close to $H(0)$.
% \end{itemize}
% Combining above results and observations, we can see
% \begin{align*}
% 	\|Y-u_{\nn}(t)\|_2 \leq \exp{(\Lambda_0 t/2)}\|Y-u_{\nn}(0)\|_2 \overset{t\to\infty}{\longrightarrow} 0.
% \end{align*}
% Back to the regularization case, we follow the similar idea as the standard case. Due to the $\ell_2$ regularization term during the training process, instead of training data predictor $u_{\nn}(t)$ converging to $Y$, we show that $u_{\nn}(t)$ will approach $u^*$ defined in~\eqref{eq:def_u_*_intro}. Specifically, the gradient flow in our case satisfies 
\begin{align}
	 \frac{\d \|u^* - u_{\nn}(t)\|_2^2}{\d t}
	= & -2(u^* - u_{\nn}(t))^\top (  H(t)+\lambda I) (u^* - u_{\nn}(t))\label{eq:linear_1_lev}\\
	& + 2  (u_{\nn}(t)-u^*)^\top (H(t) - H^{\cts}) (Y-u^*)\label{eq:linear_2_lev}
\end{align}
where term~\eqref{eq:linear_1_lev} is the primary term characterizing the linear convergence of $u_{\nn}(t)$ to $t^*$, and term~\eqref{eq:linear_2_lev} is the additive term that can be well controlled if $H(t)$ is sufficiently close to $H^{\cts}$. We argue the closeness of $H(t)\approx H^{\cts}$ as the consequence of the following two observations:
\begin{itemize}
	\item{Initialization phase:} At the beginning of the training, $H(0)$ can be viewed as approximating the neural tangent kernel $H^{\cts}$ using finite dimensional random features. Note the size of these random features corresponds to the width of the neural network (scale by the data dimension $d$). Therefore, when the neural network is sufficiently wide, it is equivalent to approximate the neural tangent kernel using sufficient high dimensional feature vectors, which ensures $H(0)$ is sufficiently close to $H^{\cts}$.

    In the case of leverage score initialization, we further take the regularization into consideration. We use the tool of leverage score to modify the initialization distribution and corresponding network parameter, to give a smaller upper-bound of the width of the nets needed. 
    % Note $\E[H(0)]=H^{\cts}$ and $H(0)$
    % identical to the standard case since both initialize in the same way, i.e., $H(0)$ is sufficiently close to $H^{\cts}$ if the network is sufficiently wide;
	\item{Training phase:} If the net is sufficiently wide, we can observe the over-parametrization phenomenon such that the weight estimate $W(t)$ at time $t$ will be sufficiently close to its initialization $W(0)$, which implies the dynamic kernel $H(t)$ being sufficiently close to $H(0)$. Due to the fact $H(0)\approx H^{\cts}$ argued in initialization phase, we have $H(t)\approx H^{\cts}$ throughout the algorithm.
\end{itemize}

Combining both observations, we are able to iteratively show the (nearly) linear convergence property of training the regularized neural net as in following lemma:
\begin{lemma}[Bounding kernel perturbation, informal]\label{lem:induction_lev}

For any accuracy $\Delta \in (0,1/10)$. If the network width $m = \poly(1/\Delta, 1/T, 1/\epsilon_{\train}, n, d, 1/\kappa, 1/\Lambda_0, \log(1/\delta))$ and $\lambda = O(\frac{1}{\sqrt{m}})$, with probability $1-\delta$, there exist $\epsilon_W,~\epsilon_H',~\epsilon_K'\in(0,\Delta)$ that are independent of $t$, such that the following hold for all $0 \leq t \le T$:
\begin{enumerate}
    \item  $\| w_r(0) - w_r(t) \|_2 \leq \epsilon_W $, $\forall r \in [m]$
    \item  $\| H(0) - H(t) \|_2 \leq \epsilon_H'$ 
    \item  $\| u_{\nn}(t) - u^* \|_2^2 \leq \max\{\epsilon_{\train}^2, e^{-(\kappa^2\Lambda_0 + \lambda) t/2}\| u_{\nn}(0) - u^* \|_2^2\}$
    % \item 4. $\| \k_0( x_{\test} , X )- \k_{t} (x_{\test}, X) \|_2 \leq \epsilon_K'$
\end{enumerate}
\end{lemma}
Given arbitrary accuracy $\epsilon\in(0,1)$, if we choose $\epsilon_{\train} = \epsilon$, $T=\wt{O}(\frac{1}{\kappa^2\Lambda_0})$ and $m$ sufficiently large in Lemma~~\ref{lem:induction_lev}, then we have $\| u_{\nn}(t) - u^* \|_2 \leq \epsilon$, indicating the equivalence between training neural network with regularization and neural tangent kernel ridge regression for the training data predictions.

To further argue the equivalence for any given test data $x_{\test}$, we observe the similarity between the gradient flows of neural tangent kernel ridge regression $u_{\ntk, \test}(t)$ and regularized neural networks $u_{\nn,\test}(t)$ as following:
\begin{align}
    \frac{\d u_{\ntk, \test}(t)}{\d t} = & ~ \kappa^2 \k_{\ntk}(x_{\test}, X)^\top  ( Y - u_{\ntk}(t) ) - \lambda \cdot u_{\ntk, \test}(t).\label{eq:gradient_nn_intro}\\
    \frac{\d u_{\nn,\test}(t)}{\d t} = & ~ \kappa^2 \k_{t}(x_{\test}, X)^\top ( Y - u_{\nn}(t) ) - \lambda \cdot u_{\nn,\test}(t).\label{eq:gradient_ntk_intro}
\end{align}
By choosing the multiplier $\kappa>0$ small enough, we can bound the initial difference between these two predictors. Combining with above similarity between gradient flows, we are able to show $|u_{\nn,\test}(T)-u_{\ntk,\test}(T)|\geq \epsilon/2$ for appropriate $T>0$. Finally, note the linear convergence property of the gradient of the kernel ridge regression, we can prove $|u_{\nn,\test}(T)-u^*_{\ntk,\test}|\geq \epsilon$.
% By the same observation we can inductively argue the similarity between the dynamic kernel and NTK in this 

Using the similar idea, we can also show the equivalence for test data predictors and the case of leverage score initialization. We refer to the Appendix for a detailed proof sketch and rigorous proof.

\begin{remark}
Our results can be naturally extended to multi-layer ReLU deep neural networks with all parameters training together. Note the core of the connection between regularized NNs and KRR is to show the similarity between their gradient flows, as in Eq.~\eqref{eq:gradient_nn_intro},~\eqref{eq:gradient_ntk_intro}. The gradient flows consist of two terms: the first term is from normal NN training without regularizer, whose similarity has been shown in broader settings, e.g. \cite{dzps19,sy19,adhlsw19,als19a,als19b}; the second term is from the $\ell_2$ regularizer, whose similarity is true for multi-layer ReLU DNNs if the regularization parameter is divided by the number of layers of parameters trained, due to the piecewise linearity of the output with respect to the training parameters.

\end{remark}

%%% Section 4. Overview of Techniques
\section{Conclusion}
In this paper, we generalize the leverage score sampling theory for kernel approximation. We discuss the interesting application of connecting leverage score sampling and training regularized neural networks. We present two theoretical results: 1) the equivalence between the regularized neural nets and kernel ridge regression problems under the classical random Gaussian initialization for both training and test predictors; 2) the new equivalence under the leverage score initialization.  We believe this work can be the starting point of future study on the use of leverage score sampling in neural network training. 

\paragraph{Roadmap} 
In the appendix, we present our complete results and rigorous proofs. Section~\ref{sec:pre} presents some well-known mathematically results that will be used in our proof. Section~\ref{sec:equiv} discusses our first equivalence result between training regularized neural network and kernel ridge regression. Section~\ref{sec:ger_lev} discusses our generalization result of the leverage score sampling theory. Section~\ref{sec:eq_lev} discusses our second equivalence result under leverage score initialization and potential benefits compared to the Gaussian initialization. Section~\ref{sec:gen_dnn} discusses how to extend our results to a broader class of neural network models.

% In this paper, we show the equivalence between training regularized neural network and kernel ridge regression. Moreover, we build the connection between neural network and leverage sampling by further showing the equivalence between leveraged score initialized neural network and kernel ridge regression. We believe leverage score sampling, which is a power tool in many other problem, can also be used in neural network training. For example, it is possible to use leverage score to utilize the information contained in input data distribution when initializing neural network. We believe this work can be the starting point of future study on the use of leverage score sampling in neural network training.  %%% Section 5. Conclusion
\newpage

\appendix

\section*{Appendix}
\section{Preliminaries}\label{sec:pre}

\subsection{Probability tools}

In this section we introduce the probability tools we use in the proof.

We state Chernoff, Hoeffding and Bernstein inequalities.
\begin{lemma}[Chernoff bound \cite{c52}]\label{lem:chernoff}
Let $X = \sum_{i=1}^n X_i$, where $X_i=1$ with probability $p_i$ and $X_i = 0$ with probability $1-p_i$, and all $X_i$ are independent. Let $\mu = \E[X] = \sum_{i=1}^n p_i$. Then \\
1. $ \Pr[ X \geq (1+\delta) \mu ] \leq \exp ( - \delta^2 \mu / 3 ) $, $\forall \delta > 0$ ; \\
2. $ \Pr[ X \leq (1-\delta) \mu ] \leq \exp ( - \delta^2 \mu / 2 ) $, $\forall 0 < \delta < 1$. 
\end{lemma}

\begin{lemma}[Hoeffding bound \cite{h63}]\label{lem:hoeffding}
Let $X_1, \cdots, X_n$ denote $n$ independent bounded variables in $[a_i,b_i]$. Let $X= \sum_{i=1}^n X_i$, then we have
\begin{align*}
\Pr[ | X - \E[X] | \geq t ] \leq 2\exp \left( - \frac{2t^2}{ \sum_{i=1}^n (b_i - a_i)^2 } \right).
\end{align*}
\end{lemma}

\begin{lemma}[Bernstein inequality \cite{b24}]\label{lem:bernstein}
Let $X_1, \cdots, X_n$ be independent zero-mean random variables. Suppose that $|X_i| \leq M$ almost surely, for all $i$. Then, for all positive $t$,
\begin{align*}
\Pr \left[ \sum_{i=1}^n X_i > t \right] \leq \exp \left( - \frac{ t^2/2 }{ \sum_{j=1}^n \E[X_j^2]  + M t /3 } \right).
\end{align*}
\end{lemma}

We state three inequalities for Gaussian random variables.
\begin{lemma}[Anti-concentration of Gaussian distribution]\label{lem:anti_gaussian}
Let $X\sim N(0,\sigma^2)$,
that is,
the probability density function of $X$ is given by $\phi(x)=\frac 1 {\sqrt{2\pi\sigma^2}}e^{-\frac {x^2} {2\sigma^2} }$.
Then
\begin{align*}
    \Pr[|X|\leq t]\in \left( \frac 2 3\frac t \sigma, \frac 4 5\frac t \sigma \right).
\end{align*}
\end{lemma}

\begin{lemma}[Gaussian tail bounds]\label{lem:gaussian_tail}
Let $X\sim\N(\mu,\sigma^2)$ be a Gaussian random variable with mean $\mu$ and variance $\sigma^2$. Then for all $t \geq 0$, we have
\begin{align*}
  \Pr[|X-\mu|\geq t]\leq 2e^{-\frac{-t^2}{2\sigma^2}}.
\end{align*}
\end{lemma}

\begin{lemma}[Lemma 1 on page 1325 of Laurent and Massart \cite{lm00}]\label{lem:chi_square_tail}
Let $X \sim {\cal X}_k^2$ be a chi-squared distributed random variable with $k$ degrees of freedom. Each one has zero mean and $\sigma^2$ variance. Then
\begin{align*}
\Pr[ X - k \sigma^2 \geq ( 2 \sqrt{kt} + 2t ) \sigma^2 ] \leq \exp (-t), \\
\Pr[ k \sigma^2 - X \geq 2 \sqrt{k t} \sigma^2 ] \leq \exp(-t).
\end{align*}
\end{lemma}

We state two inequalities for random matrices.
\begin{lemma}[Matrix Bernstein, Theorem 6.1.1 in \cite{t15}]\label{lem:matrix_bernstein}
Consider a finite sequence $\{ X_1, \cdots, X_m \} \subset \R^{n_1 \times n_2}$ of independent, random matrices with common dimension $n_1 \times n_2$. Assume that
\begin{align*}
\E[ X_i ] = 0, \forall i \in [m] ~~~ \mathrm{and}~~~ \| X_i \| \leq M, \forall i \in [m] .
\end{align*}
Let $Z = \sum_{i=1}^m X_i$. Let $\mathrm{Var}[Z]$ be the matrix variance statistic of sum:
\begin{align*}
\mathrm{Var} [Z] = \max \left\{ \Big\| \sum_{i=1}^m \E[ X_i X_i^\top ] \Big\| , \Big\| \sum_{i=1}^m \E [ X_i^\top X_i ] \Big\| \right\}.
\end{align*}
Then 
{\small
\begin{align*}
\E[ \| Z \| ] \leq ( 2 \mathrm{Var} [Z] \cdot \log (n_1 + n_2) )^{1/2} +  M \cdot \log (n_1 + n_2) / 3.
\end{align*}
}
Furthermore, for all $t \geq 0$,
\begin{align*}
\Pr[ \| Z \| \geq t ] \leq (n_1 + n_2) \cdot \exp \left( - \frac{t^2/2}{ \mathrm{Var} [Z] + M t /3 }  \right)  .
\end{align*}
\end{lemma}

\begin{lemma}[Matrix Bernstein, Lemma 7 in \cite{akmmvz17}]\label{lem:con_lev}
Let $B\in\R^{d_1\times d_2}$ be a fixed matrix. Construct a random matrix $R\in\R^{d_1\times d_2}$ satisfies
\begin{align*}
  \E[R] = B~\text{and}~\|R\| \leq L.
\end{align*}
Let $M_1$ and $M_2$ be semidefinite upper bounds for the expected squares:
\begin{align*}
  \E[RR^\top] \preceq M_1~\text{and}~\E[R^\top R] \preceq M_2.
\end{align*}
Define the quantities
\begin{align*}
  m = \max\{\|M_1\|, \|M_2\|\}~\text{and}~d=(\Tr[M_1]+\Tr[M_2])/m.
\end{align*}
Form the matrix sampling estimator
\begin{align*}
  \bar{R}_n = \frac{1}{n} \sum_{k=1}^n R_k
\end{align*}
where each $R_k$ is an independent copy of $R$. Then, for all $t\geq \sqrt{m/n}+2L/3n$,
\begin{align*}
  \Pr[\|\bar{R}_n - B\|_2 \geq t] \leq 4d\exp{\Big(\frac{-nt^2}{m+2Lt/3}\Big)}
\end{align*}
\end{lemma}

\subsection{Neural tangent kernel and its properties}

\begin{lemma}[Lemma 4.1 in \cite{sy19}]\label{lem:lemma_4.1_in_sy19}
We define $H^{\cts}$, $H^{\dis} \in \R^{n \times n}$ as follows
\begin{align*}
    H^{\cts}_{i,j} = & ~ \E_{w \sim {\cal N} (0,I) } [ x_i^\top x_j {\bf 1}_{ w^\top x_i \geq 0 , w^\top x_j \geq 0 } ] , \\
    H^{\dis}_{i,j} = & ~ \frac{1}{m} \sum_{r=1}^m [ x_i^\top x_j {\bf 1}_{w_r^\top x_i \geq 0, w_r^\top x_j \geq 0 } ] .
\end{align*}
Let $\lambda = \lambda_{\min} (H^{\cts})$. If $m = \Omega( \lambda^{-2} n^2 \log ( n / \delta ) )$, we have
\begin{align*}
    \| H^{\dis} - H^{\cts} \|_F \leq \frac{ \lambda }{ 4 }, \text{~and~} \lambda_{\min} ( H^{\dis} ) \geq \frac{3}{4} \lambda
\end{align*}
hold with probability at least $1-\delta$.
\end{lemma}

\begin{lemma}[Lemma 4.2 in \cite{sy19}]\label{lem:lemma_4.2_in_sy19}
Let $R \in (0,1)$. If $\wt{w}_1, \cdots, \wt{w}_m$ are i.i.d. generated from ${\cal N}(0,I)$. For any set of weight vectors $w_1, \cdots, w_m \in \R^d$ that satisfy for any $r \in [m]$, $\| \wt{w}_r - w_r \|_2 \leq R$, then the $H : \R^{ m \times d} \rightarrow \R^{n \times n}$ defined
\begin{align*}
    H(W) = \frac{1}{m} x_i^\top x_j \sum_{r=1}^m {\bf 1}_{ w_r^\top x_i \geq 0, w_r^\top x_j \geq 0 } .
\end{align*}
Then we have
\begin{align*}
    \| H(w) - H(\wt{w}) \|_F < 2 n R
\end{align*}
holds with probability at least $1-n^2 \cdot \exp(-m R /10)$.
\end{lemma}

%%% Section A. Preliminaries
%\newpage

\section{Equivalence between sufficiently wide neural net and kernel ridge regression}\label{sec:equiv}
In this section, we extend the equivalence result in \cite{jgh18, adhlsw19} to the case with regularization term, where they showed the equivalence between a fully-trained infinitely wide/sufficiently wide neural net and the kernel regression solution using the neural tangent kernel (NTK). Specifically, we prove Theorem~\ref{thm:main_train_equivalence_intro} and Theorem~\ref{thm:main_test_equivalence_intro} in this section.

Section~\ref{sec:equiv_preli} introduces key notations and standard data assumptions. Section~\ref{sec:equiv_definition} restates and supplements the definitions introduced in the paper. Section~\ref{sec:equiv_gradient_flow} presents several key lemmas about the gradient flow and linear convergence of neural network and kernel ridge regression predictors, which are crucial to the final proof. Section~\ref{sec:equiv_proof_sketch} provides a brief proof sketch. Section~\ref{sec:test} restates the main equivalence Theorem~\ref{thm:main_test_equivalence_intro} and provides a complete proof following the proof sketch. Section~\ref{sec:train} restates and proves Theorem~\ref{thm:main_train_equivalence_intro} by showing it as a by-product of previous proof.

% Here, we list the locations where definitions and theorems in the paper are restated. Definition~\ref{def:f_nn_intro},~\ref{def:nn_intro},~\ref{def:ntk_phi_intro} are restated and supplemented in Definition~\ref{def:f_nn},~\ref{def:nn},~\ref{def:ntk_phi} respectively. Definition~\ref{def:krr_ntk_intro} and Lemma~\ref{lem:opt_predict_intro} are restated in Definition~\ref{def:krr_ntk}. Theorem~\ref{thm:main_train_equivalence_intro} is restated in Theorem~\ref{thm:main_train_equivalence}, and Theorem~\ref{thm:main_test_equivalence_intro} is restated in Theorem~\ref{thm:main_test_equivalence}.

% Section~\ref{sec:equiv_bound_ntk_test_T_and_test_*} shows how to bound $| u_{\ntk,\test}(T) - u_{\test}^* |$. Section~\ref{sec:equiv_intialization} bounds the quantities at the initialization point, i.e., $|f_{\nn}(W(0),x)|$, $|u_{\nn}(0)|$ and $|u_{\nn,\test}(0)|$. Section~\ref{sec:equiv_splitting_nn_test_T_and_ntk_test_T_into_three} splits $| u_{\nn,\test}(T) - u_{\ntk,\test}(T) |$ into three terms and Section~\ref{sec:equiv_bound_nn_test_T_and_ntk_test_T} upper bounds $| u_{\nn,\test}(T) - u_{\ntk,\test}(T) |$ finally. Section~\ref{sec:equiv_kernel_perturbation} bounds the kernel perturbation via induction. Section~\ref{sec:equiv_main_test_equivalence} presents our main result for test equivalence, and Section~\ref{sec:equiv_main_train_equivalence} presents our result for train equivalence.
 
\subsection{Preliminaries}\label{sec:equiv_preli}
Let's define the following notations:
\begin{itemize}
    \item $X\in\R^{n\times d}$ be the training data
    \item $x_{\test}\in\R^d$ be the test data
    \item $u_{\ntk}(t)=\kappa f_{\ntk}(\beta(t),X) = \kappa \Phi(X)\beta(t) \in \R^n$ be the prediction of the kernel ridge regression for the training data at time $t$. (See Definition~\ref{def:krr_ntk})
    \item $u^* = \lim_{t \rightarrow \infty} u_{\ntk}(t)$ (See Eq.~\eqref{eq:def_u_*})
    \item $u_{\ntk,\test}(t) = \kappa f_{\ntk}(\beta(t),x_{\test}) = \kappa \Phi(x_{\test}) \beta(t) \in \R$ be the prediction of the kernel ridge regression for the test data at time $t$. (See Definition~\ref{def:krr_ntk})
    \item $u_{\test}^* = \lim_{t\rightarrow \infty} u_{\ntk,\test}( t ) $ (See Eq.~\eqref{eq:def_u_test_*})
    \item $\k_{\ntk}(x, y) = \E [\left\langle \frac{\partial f_{\nn}(W,x)}{\partial W},\frac{\partial f_{\nn}(W,y)}{\partial W} \right\rangle]$ (See Definition~\ref{def:krr_ntk})
    \item $\k_t(x_{\test},X) \in \R^n$ be the induced kernel between the training data and test data at time $t$, where
    \begin{align*}
    [\k_t(x_{\test},X)]_i 
    = \k_t(x_{\test},x_i)
    = \left\langle \frac{\partial f(W(t),x_{\test})}{\partial W(t)},\frac{\partial f(W(t),x_i)}{\partial W(t)} \right\rangle
    \end{align*}
    (see Definition~\ref{def:dynamic_kernel})
    \item $u_{\nn}(t) = \kappa f_{\nn}(W(t),X) \in \R^n$ be the prediction of the neural network for the training data at time $t$. (See Definition~\ref{def:nn})
    \item $u_{\nn,\test}(t) = \kappa f_{\nn}(W(t),x_{\test}) \in \R$ be the prediction of the neural network for the test data at time $t$ (See Definition~\ref{def:nn})
\end{itemize}

\begin{assumption}[data assumption]\label{ass:data_assumption}
We made the following assumptions: \\
1. For each $i \in [n]$, we assume $|y_i| = O(1)$.\\
2. $H^{\cts}$ is positive definite, i.e., $\Lambda_0:=\lambda_{\min}(H^{\cts})>0$.\\
3. All the training data and test data have Euclidean norm equal to 1.
\end{assumption}

\subsection{Definitions}\label{sec:equiv_definition}
To establish the equivalence between neural network and kernel ridge regression, we prove the similarity of their gradient flow and initial predictor. Note kernel ridge regression starts at 0 as initialization, so we hope the initialization of neural network also close to zero. Therefore, using the same technique in~\cite{adhlsw19}, we apply a small multiplier $\kappa>0$ to both predictors to bound the different of initialization.

\begin{definition}[Neural network function]\label{def:f_nn}
We define a two layer neural networks with rectified linear unit (ReLU) activation as the following form
\begin{align*}
f_{\nn} (W, a, x) = \frac{1}{\sqrt{m}} \sum_{r=1}^m a_r \sigma (w_r^\top x) \in \R,
\end{align*}
where $x \in \R^d$ is the input, $w_r \in \R^d,~r\in[m]$ is the weight vector of the first layer, $W = [w_1, \cdots, w_m]\in\R^{d \times m}$, $a_r \in \R,~r\in[m]$ is the output weight, $a = [a_1, \cdots, a_m]^\top$ and $\sigma(\cdot)$ is the ReLU activation function: $\sigma(z) = \max\{0,z\}$. In this paper, we consider only training the first layer $W$ while fix $a$. So we also write $f_{\nn}(W,x) = f_{\nn}(W, a, x)$. We denote $f_{\nn}(W, X) = [f_{\nn}(W, x_1),\cdots, f_{\nn}(W, x_n)]^\top\in\R^n$.
\end{definition}

\begin{definition}[Training neural network with regularization, restatement of Definition~\ref{def:nn_intro}]\label{def:nn}
Given training data matrix $X\in\R^{n\times d}$ and corresponding label vector $Y\in\R^n$. Let $f_{\nn}$ be defined as in Definition~\ref{def:f_nn}. Let $\kappa\in(0,1)$ be a small multiplier. Let $\lambda\in(0,1)$ be the regularization parameter. We initialize the network as $a_r\overset{i.i.d.}{\sim} \unif[\{-1,1\}]$ and $w_r(0)\overset{i.i.d.}{\sim} \N(0,I_d)$. Then we consider solving the following optimization problem using gradient descent:
\begin{align}\label{eq:nn}
\min_{W} \frac{1}{2}\| Y - \kappa f_{\nn}(W,X) \|_2 + \frac{1}{2}\lambda\|W\|_F^2.
\end{align}\label{eq:nn_predict_train} 
We denote $w_r(t),r\in[m]$ as the variable at iteration $t$. We denote
\begin{align}\label{eq:nn_predict_train} 
u_{\nn}(t) = \kappa f_{\nn}(W(t),X) = \frac{\kappa}{\sqrt{m}} \sum_{r=1}^m a_r \sigma (w_r(t)^\top X) \in \R^n
\end{align}
as the training data predictor at iteration $t$. Given any test data $x_{\test}\in\R^d$, we denote 
\begin{align}\label{eq:nn_predict_test} 
u_{\nn,\test}(t) = \kappa f_{\nn}(W(t),x_{\test}) = \frac{\kappa}{\sqrt{m}} \sum_{r=1}^m a_r \sigma (w_r(t)^\top x_{\test}) \in \R
\end{align}
as the test data predictor at iteration $t$.
\end{definition}

\begin{definition}[Neural tangent kernel and feature function]\label{def:ntk_phi}
We define the neural tangent kernel(NTK) and the feature function corresponding to the neural networks $f_{\nn}$ defined in Definition~\ref{def:f_nn} as following 
\begin{align*}
	\k_{\ntk}(x, z) = \E \left[\left\langle \frac{\partial f_{\nn}(W,x)}{\partial W},\frac{\partial f_{\nn}(W,z)}{\partial W} \right\rangle \right]
\end{align*}
where $x,z \in \R^d$ are any input data, and the expectation is taking over $w_r\overset{i.i.d.}{\sim} \N(0,I),~r=1, \cdots, m$. Given training data matrix $X=[x_1,\cdots,x_n]^\top\in\R^{n\times d}$, we define $H^{\cts}\in\R^{n\times n}$ as the kernel matrix between training data as
\begin{align*}
	[H^{\cts}]_{i,j} = \k_{\ntk}(x_i, x_j) \in \R.
\end{align*}
We denote the smallest eigenvalue of $H^{\cts}$ as $\Lambda_0 > 0$, where we assume $H^{\cts}$ is positive definite. Further, given any data $z \in \R^d$, we write the kernel between test and training data $\k_{\ntk}(z, X) \in \R^n$ as
\begin{align*}
	\k_{\ntk}(z, X) = [\k_{\ntk}(z, x_1),\cdots,\k_{\ntk}(z,x_n)]^\top \in \R^n.
\end{align*}
We denote the feature function corresponding to the kernel $\k_{\ntk}$ as we defined above as $\Phi:\R^{d} \rightarrow \mathcal{F}$, which satisfies
\begin{align*}
	\langle\Phi(x),\Phi(z)\rangle_\mathcal{F} = \k_{\ntk}(x,z),
\end{align*}
for any data $x$, $z\in\R^d$. And we write $\Phi(X)=[\Phi(x_1),\cdots,\Phi(x_n)]^\top$.
\end{definition}

\begin{definition}[Neural tangent kernel ridge regression]\label{def:krr_ntk}
Given training data matrix $X=[x_1,\cdots,x_n]^\top$ $\in \R^{n\times d}$ and corresponding label vector $Y\in\R^n$. Let $\k_{\ntk}$, $H^{\cts}\in\R^{n\times n}$ and $\Phi$ be the neural tangent kernel and corresponding feature functions defined as in Definition~\ref{def:ntk_phi}. Let $\kappa\in(0,1)$ be a small multiplier. Let $\lambda\in(0,1)$ be the regularization parameter. Then we consider the following neural tangent kernel ridge regression problem:
\begin{align}\label{eq:krr}
\min_{\beta} \frac{1}{2}\| Y - \kappa f_{\ntk}(\beta,X) \|_2^2 + \frac{1}{2}\lambda\|\beta\|_2^2.
\end{align}
where $f_{\ntk}(\beta,x) = \Phi(x)^\top \beta \in \R$ denotes the prediction function is corresponding RKHS and $f_{\ntk}(\beta,X) = [f_{\ntk}(\beta,x_1),\cdots,f_{\ntk}(\beta,x_n)]^\top\in\R^{n}$. Consider the gradient flow of solving problem~\eqref{eq:krr} with initialization $\beta(0) = 0$.
We denote $\beta(t)$ as the variable at iteration $t$. We denote
\begin{align}\label{eq:ntk_predict_train}
	u_{\ntk}(t) = \kappa\Phi(X)\beta(t) \in \R^n
\end{align} 
as the training data predictor at iteration $t$. Given any test data $x_{\test}\in\R^d$, we denote
\begin{align}\label{eq:ntk_predict_test}
	u_{\ntk,\test}(t) = \kappa\Phi(x_{\test})^\top\beta(t) \in \R
\end{align} 
as the test data predictor at iteration $t$. Note the gradient flow converge the to optimal solution of problem~\eqref{eq:krr} due to the strongly convexity of the problem. We denote
\begin{align}\label{eq:def_beta_*}
	\beta^* = \lim_{t\to\infty} \beta(t) = \kappa (\kappa^2 \Phi(X)^\top \Phi(X) + \lambda I)^{-1} \Phi(X)^\top Y
\end{align}
and the optimal training data predictor
\begin{align}\label{eq:def_u_*}
	u^* = \lim_{t\to\infty} u_{\ntk}(t) = \kappa \Phi(X)\beta^* = \kappa^2 H^{\cts}(\kappa^2 H^{\cts}+\lambda I)^{-1}Y \in \R^n
\end{align}
and the optimal test data predictor
\begin{align}\label{eq:def_u_test_*}
	u_{\test}^* = \lim_{t\to\infty} u_{\ntk,\test}(t) = \kappa \Phi(x_{\test})^\top \beta^* = \kappa^2 \k_{\ntk}(x_{\test}, X)^\top(\kappa^2 H^{\cts}+\lambda I)^{-1}Y \in \R.
\end{align}
\end{definition}

\begin{definition}[Dynamic kernel]\label{def:dynamic_kernel}
Given $W(t) \in \R^{d \times m}$ as the parameters of the neural network at training time $t$ as defined in Definition~\ref{def:nn}. For any data $x,z\in\R^d$, we define $\k_t(x,z)\in\R$ as
\begin{align*}
    \k_t(x,z)
    = \left\langle \frac{\d f_{\nn}(W(t),x)}{\d W(t)},\frac{\d f_{\nn}(W(t),z)}{\d W(t)} \right\rangle
\end{align*}
Given training data matrix $X=[x_1,\cdots,x_n]^\top\in\R^{n\times d}$, we define $H^{(t)}\in\R^{n\times n}$ as
\begin{align*}
	[H(t)]_{i,j} = \k_{t}(x_i, x_j)\in\R.
\end{align*}
Further, given a test data $x_{\test}\in\R^d$, we define $\k_t(x_{\test},X)\in\R^n$ as
\begin{align*}
    \k_t(x_{\test},X) = [\k_t(x_{\test},x_1), \cdots, \k_t(x_{\test},x_n)]^\top\in\R^n.
\end{align*}
\end{definition}

\subsection{Gradient, gradient flow, and linear convergence}\label{sec:equiv_gradient_flow}
\begin{lemma}[Gradient flow of kernel ridge regression]\label{lem:gradient_flow_of_krr}
Given training data matrix $X\in\R^{n\times d}$ and corresponding label vector $Y\in\R^n$. Let $f_{\ntk}$ be defined as in Definition~\ref{def:krr_ntk}. Let $\beta(t)$, $\kappa\in(0,1)$ and $u_{\ntk}(t)\in\R^n$ be defined as in Definition~\ref{def:krr_ntk}. Let $\k_{\ntk}: \R^d \times \R^{n\times d} \to\R^n$ be defined as in Definition~\ref{def:ntk_phi}. Then for any data $z\in\R^d$, we have
\begin{align*}
	\frac{\d f_{\ntk}(\beta(t), z)}{\d t} = \kappa \cdot \k_{\ntk}(z, X)^\top ( Y - u_{\ntk}(t) ) - \lambda \cdot f_{\ntk}(\beta(t), z).
\end{align*}
\end{lemma}
\begin{proof}
Denote $L(t)= \frac{1}{2}\|Y-u_{\ntk}(t)\|_2^2+\frac{1}{2}\lambda\|\beta(t)\|_2^2$. By the rule of gradient descent, we have
\begin{align*}
	\frac{\d \beta(t)}{\d t}=-\frac{\d L}{\d \beta}=\kappa \Phi(X)^\top(Y-u_{\ntk}(t))-\lambda\beta(t),
\end{align*}
where $\Phi$ is defined in Definition~\ref{def:ntk_phi}.
Thus we have
\begin{align*}
	\frac{\d f_{\ntk}(\beta(t), z)}{\d t}
	= & ~ \frac{\d f_{\ntk}(\beta(t), z)}{\d \beta(t)}\frac{\d \beta(t)}{\d t} \\
	= & ~ \Phi(z)^\top (\kappa\Phi(X)^\top(Y-u_{\ntk}(t))-\lambda\beta(t)) \\
	= & ~ \kappa\k_{\ntk}(z, X)^\top (Y-u_{\ntk}(t))-\lambda\Phi(z)^\top\beta(t) \\
	= & ~ \kappa\k_{\ntk}(z, X)^\top (Y-u_{\ntk}(t))-\lambda f_{\ntk}(\beta(t), z),
\end{align*}
where the first step is due to chain rule, the second step follows from the fact $ \d f_{\ntk}(\beta, z)/ \d \beta=\Phi(z)$, the third step is due to the definition of the kernel $\k_{\ntk}(z, X)=\Phi(X)\Phi(z) \in \R^{n}$, and the last step is due to the definition of $f_{\ntk}(\beta(t), z)\in\R$.
\end{proof}

\begin{corollary}[Gradient of prediction of kernel ridge regression]\label{cor:ntk_gradient}
Given training data matrix $X=[x_1,\cdots,x_n]^\top \in \R^{n\times d}$ and corresponding label vector $Y \in \R^n$. Given a test data $x_{\test}\in\R^d$. Let $f_{\ntk}$ be defined as in Definition~\ref{def:krr_ntk}. Let $\beta(t)$, $\kappa\in(0,1)$ and $u_{\ntk}(t)\in\R^n$ be defined as in Definition~\ref{def:krr_ntk}. Let $\k_{\ntk}: \R^d \times \R^{n\times d} \rightarrow \R^n,~H^{\cts}\in\R^{n\times n}$ be defined as in Definition~\ref{def:ntk_phi}. Then we have 
\begin{align*}
	\frac{\d u_{\ntk}(t)}{\d t} & = \kappa^2 H^{\cts} ( Y - u_{\ntk}(t) ) - \lambda \cdot u_{\ntk}(t)\\
	\frac{\d u_{\ntk, \test}(t)}{\d t} & = \kappa^2 \k_{\ntk}(x_{\test}, X)^\top  ( Y - u_{\ntk}(t) ) - \lambda \cdot u_{\ntk, \test}(t).
\end{align*}
\end{corollary}
\begin{proof}
Plugging in $z = x_i\in\R^d$ in Lemma~\ref{lem:gradient_flow_of_krr}, we have
\begin{align*}
	\frac{\d f_{\ntk}(\beta(t), x_i)}{\d t} = \kappa \k_{\ntk}(x_i, X)^\top ( Y - u_{\ntk}(t) ) - \lambda \cdot f_{\ntk}(\beta(t), x_i).
\end{align*}
Note $[u_{\ntk}(t)]_i = \kappa f_{\ntk}(\beta(t), x_i)$ and $[H^{\cts}]_{:,i} = \k_{\ntk}(x_i, X)$, so writing all the data in a compact form, we have
\begin{align*}
	\frac{\d u_{\ntk}(t)}{\d t} = \kappa^2 H^{\cts} ( Y - u_{\ntk}(t) ) - \lambda \cdot u_{\ntk}(t).
\end{align*}
Plugging in data $z = x_{\test}\in\R^d$ in Lemma~\ref{lem:gradient_flow_of_krr}, we have
\begin{align*}
	\frac{\d f_{\ntk}(\beta(t), x_{\test})}{\d t} = \kappa \k_{\ntk}(x_{\test}, X)^\top ( Y - u_{\ntk}(t) ) - \lambda \cdot f_{\ntk}(\beta(t), x_{\test}).
\end{align*}
Note by definition, $u_{\ntk,\test}(t) = \kappa f_{\ntk}(\beta(t), x_{\test}) \in \R$, so we have
\begin{align*}
	\frac{\d u_{\ntk, \test}(t)}{\d t} = \kappa^2 \k_{\ntk}(x_{\test}, X)^\top ( Y - u_{\ntk}(t) ) - \lambda \cdot u_{\ntk, \test}(t).
\end{align*}
\end{proof}

\begin{lemma}[Linear convergence of kernel ridge regression]\label{lem:linear_converge_krr}
Given training data matrix $X=[x_1,\cdots,x_n]^\top \in \R^{n\times d}$ and corresponding label vector $Y\in\R^n$. Let $\kappa\in(0,1)$ and $u_{\ntk}(t) \in \R^n$ be defined as in Definition~\ref{def:krr_ntk}. Let $u^* \in \R^n$ be defined in Definition~\ref{def:krr_ntk}. Let $\Lambda_0 > 0$ be defined as in Definition~\ref{def:ntk_phi}. Let $\lambda > 0$ be the regularization parameter. Then we have
\begin{align*}
\frac{\d \|u_{\ntk}(t)-u^*\|_2^2}{\d t} \le - 2(\kappa^2 \Lambda_0+\lambda) \|u_{\ntk}(t)-u^*\|_2^2.
\end{align*}
Further, we have
\begin{align*}
	\|u_{\ntk}(t)-u^*\|_2 \leq e^{-(\kappa^2 \Lambda_0+\lambda)t} \|u_{\ntk}(0)-u^*\|_2.
\end{align*}

\end{lemma}
\begin{proof}
Let $H^{\cts} \in \R^{n \times n}$ be defined as in Definition~\ref{def:ntk_phi}. Then
\begin{align}\label{eq:322_1}
	\kappa^2 H^{\cts}(Y-u^*) 
	= & ~ \kappa^2 H^{\cts}(Y-\kappa^2 H^{\cts}(\kappa^2 H^{\cts}+\lambda I_n)^{-1}Y) \notag \\
	= & ~ \kappa^2 H^{\cts}(I_n-\kappa^2 H^{\cts}(\kappa^2 H^{\cts}+\lambda I)^{-1})Y \notag \\
	= & ~ \kappa^2 H^{\cts}(\kappa^2 H^{\cts}+\lambda I_n - \kappa^2 H^{\cts})(\kappa^2 H^{\cts}+\lambda I_n)^{-1} Y \notag \\
	= & ~ \kappa^2 \lambda H^{\cts}(\kappa^2 H^{\cts}+\lambda I_n)^{-1} Y \notag \\
	= & ~ \lambda u^*,
\end{align}
where the first step follows the definition of $u^* \in \R^n$, the second to fourth step simplify the formula, and the last step use the definition of $u^* \in \R^n$ again.
So we have
\begin{align}\label{eq:322_2}
	\frac{\d \|u_{\ntk}(t)-u^*\|_2^2}{\d t} 
	= & ~ 2(u_{\ntk}(t)-u^*)^\top \frac{\d u_{\ntk}(t)}{\d t} \notag\\
	= & ~ -2\kappa^2 (u_{\ntk}(t)-u^*)^\top H^{\cts} (u_{\ntk}(t) - Y) -2\lambda(u_{\ntk}(t)-u^*)^\top u_{\ntk}(t) \notag\\
	= & ~ -2\kappa^2 (u_{\ntk}(t)-u^*)^\top H^{\cts} (u_{\ntk}(t) - u^*) + 2\kappa^2 (u_{\ntk}(t)-u^*)^\top H^{\cts} (Y-u^*) \notag\\
	& ~ -2\lambda(u_{\ntk}(t)-u^*)^\top u_{\ntk}(t)\notag\\
	= & ~ -2\kappa^2 (u_{\ntk}(t)-u^*)^\top H^{\cts} (u_{\ntk}(t) - u^*) + 2\lambda(u_{\ntk}(t)-u^*)^\top u^* \notag\\
	& ~ -2\lambda(u_{\ntk}(t)-u^*)^\top u_{\ntk}(t) \notag\\
	= & ~ -2(u_{\ntk}(t)-u^*)^\top (\kappa^2 H^{\cts}+\lambda I) (u_{\ntk}(t) - u^*) \notag\\
	\leq & ~ -2(\kappa^2 \Lambda_0 + \lambda)\|u_{\ntk}(t)-u^*\|_2^2,
\end{align}
where the first step follows the chain rule, the second step follows Corollary~\ref{cor:ntk_gradient}, the third step uses basic linear algebra, the fourth step follows Eq.~\eqref{eq:322_1}, the fifth step simplifies the expression, and the last step follows the definition of $\Lambda_0$.
Further, since 
\begin{align*}
	& ~ \frac{\d (e^{2(\kappa^2 \Lambda_0+\lambda)t}\|u_{\ntk}(t)-u^*\|_2^2)}{\d t} \\
	= & ~ 2(\kappa^2 \Lambda_0+\lambda)e^{2(\kappa^2 \Lambda_0+\lambda)t}\|u_{\ntk}(t)-u^*\|_2^2 + e^{2(\kappa^2 \Lambda_0+\lambda)t}\cdot\frac{\d \|u_{\ntk}(t)-u^*\|_2^2}{\d t} \\
	\leq & ~ 0,
	% \leq & ~ 
\end{align*}
where the first step calculates the gradient, and the second step follows from Eq.~\eqref{eq:322_2}. Thus, $e^{2(\kappa^2 \Lambda_0+\lambda)t}\|u_{\ntk}(t)-u^*\|_2^2$ is non-increasing, which implies
\begin{align*}
	\|u_{\ntk}(t)-u^*\|_2 \leq e^{-(\kappa^2 \Lambda_0+\lambda)t} \|u_{\ntk}(0)-u^*\|_2.
\end{align*}
\end{proof}

\begin{lemma}[Gradient flow of neural network training]\label{lem:gradient_flow_of_nn}
Given training data matrix $X \in \R^{n\times d}$ and corresponding label vector $Y\in\R^n$. Let $f_{\nn}: \R^{d\times m} \times \R^{d} \rightarrow \R$ be defined as in Definition~\ref{def:f_nn}. Let $W(t) \in \R^{d \times m}$, $\kappa\in(0,1)$ and $u_{\nn}(t)\in\R^n$ be defined as in Definition~\ref{def:nn}. Let $\k_{t}: \R^d \times \R^{n\times d} \rightarrow \R^n$ be defined as in Definition~\ref{def:dynamic_kernel}. Then for any data $z \in \R^d$, we have
\begin{align*}
\frac{\d f_{\nn}(W(t),z)}{\d t} = \kappa \k_{t}(z,X)^\top ( Y - u_{\nn}(t) )- \lambda \cdot f_{\nn}(W(t),z).
\end{align*}
%where $u(t)=f(W(t),a,X)=\frac{1}{\sqrt{m}}\sum_{r=1}^m a_r\sigma(w_r(t)^\top X)$.
\end{lemma}
\begin{proof}
Denote $L(t)=\frac{1}{2}\|Y-u_{\nn}(t)\|_2^2+\frac{1}{2}\lambda\|W(t)\|_F^2$. By the rule of gradient descent, we have
\begin{align}\label{eq:323_1}
	\frac{\d w_r}{\d t} = -\frac{\partial L}{\partial w_r}=(\frac{\partial u_{\nn}}{\partial w_r})^\top(Y-u_{\nn})-\lambda w_r.
\end{align} 
Also note for ReLU activation $\sigma$, we have
\begin{align}\label{eq:323_2}
	\Big\langle \frac{\d f_{\nn}(W(t),z)}{\d W(t)},\lambda W(t) \Big\rangle = & ~ \sum_{r=1}^m \Big(\frac{1}{\sqrt{m}}a_r z \sigma'(w_r(t)^\top z)\Big)^\top (\lambda w_r(t)) \notag \\
	= & ~ \frac{\lambda}{\sqrt{m}}\sum_{r=1}^m a_r w_r(t)^\top z \sigma'(w_r(t)^\top z) \notag \\
	= & ~ \frac{\lambda}{\sqrt{m}}\sum_{r=1}^m a_r \sigma(w_t(t)^\top z) \notag \\
	= & ~ \lambda f_{\nn}(W(t),z),
\end{align}
where the first step calculates the derivatives, the second step follows basic linear algebra, the third step follows the property of ReLU activation: $\sigma(l) = l\sigma'(l)$, and the last step follows from the definition of $f_{\nn}$.
Thus, we have
\begin{align*}
 & ~ \frac{\d f_{\nn}(W(t),z)}{\d t} \\
= & ~ \Big\langle \frac{\d f_{\nn}(W(t),z)}{\d W(t)}, \frac{\d W(t)}{\d t} \Big\rangle \notag \\
= & ~ \sum_{j=1}^{n}(y_j - \kappa f_{\nn}(W(t),x_j)) \Big\langle \frac{\d f_{\nn}(W(t),z)}{\d W(t)},\frac{\d \kappa f_{\nn}(W(t),x_j)}{\d W(t)} \Big\rangle - \Big\langle \frac{\d f_{\nn}(W(t),z)}{\d W(t)},\lambda W(t) \Big\rangle \notag\\
= & ~ \kappa \sum_{j=1}^{n}(y_j- \kappa f_{\nn}(W(t),x_j)) \k_{t}(z,x_j)-\lambda \cdot f_{\nn}(W(t),z)\notag\\
= & ~ \kappa \k_{t}(z,X)^\top ( Y - u_{\nn}(t) )- \lambda \cdot f_{\nn}(W(t),z),
\end{align*}
where the first step follows from chain rule, the second step follows from Eq.~\eqref{eq:323_1}, the third step follows from the definition of $\k_{t}$ and Eq.~\eqref{eq:323_2}, and the last step rewrites the formula in a compact form.
% \begin{align}\label{eq:3}
% \sum_{r=1}^m\langle\frac{1}{\sqrt{m}}a_r x_i \sigma'(w_r(t)^\top x_i),\lambda w_r(t)\rangle = \lambda \frac{1}{\sqrt{m}}\sum_{r=1}^m a_r w_r^\top x_i {\bf 1}[w_r^\top x_i\ge 0] = \lambda u_i,
% \end{align}
% where first equality we plug in the derivative of the ReLU activation function, and the second equality is due to the definition of $u$.}
% Combining~\eqref{eq:2} and~\eqref{eq:3}, we have
% \begin{align*}
% \frac{\d u_i}{\d t}=\sum_{j=1}^n (y_j-u_j(t))H(t)_{i,j}-\lambda u_i(t)
% \end{align*}
% Write~\eqref{eq:2} in a compact form, we have 
% \begin{align*}
% \frac{\d u(t)}{\d t} = H(t) \cdot (Y-u(t)) -\lambda u(t),
% \end{align*}
% which completes the proof.
\end{proof}

\begin{corollary}[Gradient of prediction of neural network]\label{cor:nn_gradient}
Given training data matrix $X = [x_1,\cdots,x_n]^\top\in\R^{n\times d}$ and corresponding label vector $Y \in \R^n$. Given a test data $x_{\test} \in \R^d$. Let $f_{\nn}: \R^{d\times m} \times \R^d \rightarrow \R$ be defined as in Definition~\ref{def:f_nn}. Let $W(t) \in \R^{d\times m}$, $\kappa\in(0,1)$ and $u_{\nn}(t) \in \R^n$ be defined as in Definition~\ref{def:nn}. Let $\k_{t} : \R^d \times \R^{n \times d} \rightarrow \R^n,~H(t) \in \R^{n \times n}$ be defined as in Definition~\ref{def:dynamic_kernel}. Then we have 
\begin{align*}
	\frac{\d u_{\nn}(t)}{\d t} = & ~ \kappa^2 H(t) ( Y - u_{\nn}(t) ) - \lambda \cdot u_{\nn}(t)\\
	\frac{\d u_{\nn,\test}(t)}{\d t} = & ~ \kappa^2 \k_{t}(x_{\test}, X)^\top ( Y - u_{\nn}(t) ) - \lambda \cdot u_{\nn,\test}(t).
\end{align*}
\end{corollary}
\begin{proof}
Plugging in $z = x_i\in\R^d$ in Lemma~\ref{lem:gradient_flow_of_nn}, we have
\begin{align*}
	\frac{\d f_{\nn}(W(t), x_i)}{\d t} = \kappa \k_{t}(x_i, X)^\top ( Y - u_{\nn}(t) ) - \lambda \cdot f_{\nn}(W(t), x_i).
\end{align*}
Note $[u_{\nn}(t)]_i = \kappa f_{\nn}(W(t), x_i)$ and $[H(t))]_{:,i} = \k_{t}(x_i, X)$, so writing all the data in a compact form, we have
\begin{align*}
	\frac{\d u_{\nn}(t)}{\d t} = \kappa^2 H(t) ( Y - u_{\nn}(t) ) - \lambda \cdot u_{\nn}(t).
\end{align*}
Plugging in data $z = x_{\test}\in\R^d$ in Lemma~\ref{lem:gradient_flow_of_nn}, we have
\begin{align*}
	\frac{\d f_{\nn}(W(t), x_{\test})}{\d t} = \kappa \k_{t}(x_{\test}, X)^\top ( Y - u_{\nn}(t) ) - \lambda \cdot f_{\nn}(W(t), x_{\test}).
\end{align*}
Note by definition, $u_{\nn,\test}(t) = \kappa f_{\nn}(W(t), x_{\test}) $, so we have
\begin{align*}
	\frac{\d u_{\nn, \test}(t)}{\d t} = \kappa^2 \k_{t}(x_{\test}, X)^\top ( Y - u_{\nn}(t) ) - \lambda \cdot u_{\nn, \test}(t).
\end{align*}
\end{proof}

\begin{lemma}[Linear convergence of neural network training]\label{lem:linear_converge_nn}
Given training data matrix $X=[x_1,\cdots,x_n]^\top\in\R^{n\times d}$ and corresponding label vector $Y \in \R^n$. Fix the total number of iterations $T>0$. Let , $\kappa\in(0,1)$ and $u_{\nn}(t) \in \R^{n \times n}$ be defined as in Definition~\ref{def:nn}. Let $u^* \in \R^n$ be defined in Eq.~\eqref{eq:def_u_*}. Let $H^{\cts} \in \R^{n \times n}$ and $\Lambda_0 > 0$ be defined as in Definition~\ref{def:ntk_phi}. Let $H(t) \in \R^{n \times n}$ be defined as in Definition~\ref{def:dynamic_kernel}. Let $\lambda > 0$ be the regularization parameter. Assume $\|H(t) - H^{\cts}\| \leq \Lambda_0/2$ holds for all $t\in[0,T]$. Then we have
\begin{align*}
	\frac{\d \|u_{\nn}(t)-u^*\|_2^2}{\d t} \le - ( \kappa^2 \Lambda_0+\lambda) \|u_{\nn}(t)-u^*\|_2^2+ 2 \kappa^2 \| H(t) - H^{\cts} \|  \cdot \| u_{\nn}(t) - u^* \|_2 \cdot \| Y - u^* \|_2.
\end{align*}
% holds for all $t\in[0,T]$. Further,
% \begin{align*}
	
% \end{align*}
\end{lemma}

\begin{proof}
Note same as in Lemma~\ref{lem:linear_converge_krr}, we have
\begin{align}\label{eq:325_1}
	\kappa^2 H^{\cts}(Y-u^*) 
	= & ~ \kappa^2 H^{\cts}(Y-\kappa^2 H^{\cts}(\kappa^2 H^{\cts}+\lambda I_n)^{-1}Y) \notag \\
	= & ~ \kappa^2 H^{\cts}(I_n-\kappa^2 H^{\cts}(\kappa^2 H^{\cts}+\lambda I)^{-1})Y \notag \\
	= & ~ \kappa^2 H^{\cts}(\kappa^2 H^{\cts}+\lambda I_n - \kappa^2 H^{\cts})(\kappa^2 H^{\cts}+\lambda I_n)^{-1} Y \notag \\
	= & ~ \kappa^2 \lambda H^{\cts}(\kappa^2 H^{\cts}+\lambda I_n)^{-1} Y \notag \\
	= & ~ \lambda u^*,
\end{align}
where the first step follows the definition of $u^* \in \R^n$, the second to fourth step simplify the formula, and the last step use the definition of $u^* \in \R^n$ again.
Thus, we have
\begin{align*}
 & ~ \frac{\d \|u_{\nn}(t)-u^*\|_2^2}{\d t} \\
= & ~  2(u_{\nn}(t)-u^*)^\top \frac{\d u_{\nn}(t)}{\d t}\\
= & ~ -2 \kappa^2 (u_{\nn}(t)-u^*)^\top H(t) (u_{\nn}(t) - Y) -2\lambda(u_{\nn}(t)-u^*)^\top u_{\nn}(t)\\
% = & ~ -2(u_{\nn}(t)-u^*)^\top H(t) (u_{\nn}(t) - u^*) + 2(u_{\nn}(t)-u^*)^\top H(t) (Y-u^*) -2\lambda(u_{\nn}(t)-u^*)^\top u_{\nn}(t)\\
= & ~ -2 \kappa^2 (u_{\nn}(t)-u^*)^\top H(t) (u_{\nn}(t) - u^*) + 2 \kappa^2 (u_{\nn}(t)-u^*)^\top H^{\cts} (Y-u^*)\\
& ~  + 2 \kappa^2 (u_{\nn}(t)-u^*)^\top (H(t) - H^{\cts}) (Y-u^*) -2\lambda(u_{\nn}(t)-u^*)^\top u_{\nn}(t)\\
= & ~ -2 \kappa^2 (u_{\nn}(t)-u^*)^\top H(t) (u_{\nn}(t) - u^*) + 2\lambda(u_{\nn}(t)-u^*)^\top u^*\\
& ~ +2 \kappa^2 (u_{\nn}(t)-u^*)^\top (H(t) - H^{\cts}) (Y-u^*) -2\lambda(u_{\nn}(t)-u^*)^\top u_{\nn}(t)\\
= & ~ -2(u_{\nn}(t)-u^*)^\top ( \kappa^2 H(t)+\lambda I) (u_{\nn}(t) - u^*) +2 \kappa^2 (u_{\nn}(t)-u^*)^\top (H(t) - H^{\cts}) (Y-u^*)\\
%\le & ~ -2(u(t)-u^*)^\top (H(t)+\lambda I) (u(t) - u^*) + (u(t)-u^*)^\top (H(t)-H^{\cts})(u(t)-u^*)\\
%& ~ + (Y-u^*)^\top (H(t)-H^{\cts})(Y-u^*)\\
\leq & ~ - ( \kappa^2 \Lambda_0 + \lambda) \| u_{\nn}(t) - u^* \|_2^2 + 2 \kappa^2 \| H(t) - H^{\cts} \| \| u_{\nn}(t) - u^* \|_2 \| Y - u^* \|_2 % + \Lambda_{\max} ( H(t) - H^{\cts} ) \| u(t) - u^* \|_2^2 + \Lambda_{\max} ( H(t) - H^{\cts} ) \| Y - u^* \|_2^2
\end{align*}
where the first step follows the chain rule, the second step follows Corollary~\ref{cor:nn_gradient}, the third step uses basic linear algebra, the fourth step follows Eq.~\eqref{eq:325_1}, the fifth step simplifies the expression, and the last step follows the assumption $\|H(t) - H^{\cts}\| \leq \Lambda_0/2$.
\end{proof}

\subsection{Proof sketch}\label{sec:equiv_proof_sketch}

Our goal is to show with appropriate width of the neural network and appropriate training iterations, the neural network predictor will be sufficiently close to the neural tangent kernel ridge regression predictor for any test data. We follow similar proof framework of Theorem 3.2 in \cite{adhlsw19}. Given any accuracy $\epsilon\in(0,1)$, we divide this proof into following steps:
\begin{enumerate}
    \item Firstly, according to the linear convergence property of kernel ridge regression shown in Lemma~\ref{lem:linear_converge_krr}, we can choose sufficiently large training iterations $T>0$, so that $|u_{\test}^* - u_{\ntk,\test}(T) | \leq \epsilon/2$, as shown in Lemma~\ref{lem:u_ntk_test_T_minus_u_test_*}.
    \item Once fix training iteration $T$ as in step 1, we bound $|u_{\nn,\test}(T) - u_{\ntk,\test}(T)| \leq \epsilon/2$ by showing the following:
    \begin{enumerate}
    	\item Due to the similarity of the the gradient flow of neural network training and neural tangent kernel ridge regression, we can reduce the task of bounding the prediction perturbation at time $T$, i.e., $|u_{\nn,\test}(T) -  u_{\ntk,\test}(T)| $, back to bounding 
    	\begin{enumerate}
    		\item the initialization perturbation $|u_{\nn,\test}(0) -  u_{\ntk,\test}(0)| $ and
    		\item kernel perturbation $ \|H(t)-H^{\cts}\|$,~$\|\k_{\ntk}(x_{\test}, X) - \k_{t}(x_{\test}, X)\|_2$, as shown in Lemma~\ref{lem:more_concreate_bound}.
    	\end{enumerate}
    	\item According to concentration results, we can bound the initialization perturbation $ |u_{\nn,\test}(0) -  u_{\ntk,\test}(0) | $ small enough by choosing sufficiently small $\kappa\in(0,1)$, as shown in Lemma~\ref{lem:epsilon_init}.
    	\item We characterize the over-parametrization property of the neural network by inductively show that we can bound kernel perturbation  $ \|H(t)-H^{\cts}\|$,~$\|\k_{\ntk}(x_{\test}, X) - \k_{t}(x_{\test}, X)\|_2$ small enough by choosing network width $m>0$ large enough, as shown in Lemma~\ref{lem:induction}.
    	% we fix training steps $T$, and try to show by picking appropriate width of the neural network, the two predictors at step $T$ will be close enough due to the similarity of their{} gradient flows, as shown in Lemma~\ref{lem:more_concreate_bound}, Lemma~\ref{lem:induction}, Lemma~\ref{lem:equivalence_at_T}.
    	% \item Secondly, we pick appropriate training length $T$, so that the kernel ridge regression predictor at time $T$ will be close enough to its optimal solution due to linear convergence, as shown in 
	\end{enumerate}
	\item Lastly, we combine the results of step 1 and 2 using triangle inequality, to show the equivalence between training neural network with regularization and neural tangent kernel ridge regression, i.e., $| u_{\nn,\test}(T) - u_{\test}^* | \leq \epsilon$, as shown in Theorem~\ref{thm:main_test_equivalence}.
\end{enumerate}
% Combine above two steps, we can show the "equivalence" between training neural network with regularization and neural tangent kernel ridge regression.

\subsection{Equivalence between training net with regularization and kernel ridge regression for test data prediction}\label{sec:test}
In this section, we prove Theorem~\ref{thm:main_test_equivalence_intro} following the proof sketch in Section~\ref{sec:equiv_proof_sketch}.

\subsubsection{Upper bounding $| u_{\ntk,\test}(T) - u_{\test}^* |$}\label{sec:equiv_bound_ntk_test_T_and_test_*}
In this section, we give an upper bound for $| u_{\ntk,\test}(T) - u_{\test}^* |$. 
%where $u_{\test}^*$ is defined
%\begin{definition}\label{def:u_test_*}
%We define $u_{\test}^* \in \R^n$ as follows
%\begin{align*}
%	u_{\test}^* = \kappa^2 \k_{\ntk}(x_{\test}, X)( \kappa^2 H^{\cts} + \lambda I )^{-1} Y = \kappa \Phi(x_{\test}) \beta^*.
%\end{align*}
%\end{definition}

\begin{lemma}\label{lem:u_ntk_test_T_minus_u_test_*}
Let $u_{\ntk,\test}(T) \in \R$ and $u_{\test}^* \in \R$ be defined as Definition~\ref{def:krr_ntk}. % Let $u_{\test}^*$ be defined as Definition~\ref{def:u_test_*}. 
Given any accuracy $\epsilon>0$, if $\kappa\in(0,1)$, then by picking $T = \wt{O}(\frac{1}{\kappa^2 \Lambda_0})$, we have 
\begin{align*}
| u_{\ntk,\test}(T) - u_{\test}^* | \leq \epsilon/2.
\end{align*}
\end{lemma}
where $\wt{O}(\cdot)$ here hides $\poly\log( n/(\epsilon \Lambda_0) )$.

\begin{proof}
Due to the linear convergence of kernel ridge regression, i.e.,
\begin{align*}
	\frac{ \d \| \beta(t) - \beta^* \|_2^2 }{ \d t } \leq - 2(\kappa^2 \Lambda_0 +\lambda) \| \beta(t) - \beta^* \|_2^2
\end{align*}
Thus,
\begin{align*}
	| u_{\ntk,\test}(T) - u_{\test}^* |
	= & ~ | \kappa \Phi(x_{\test})^\top \beta(T) - \kappa \Phi(x_{\test})\beta^* | \\
	\leq & ~ \kappa \| \Phi(x_{\test}) \|_2 \| \beta(T) - \beta^* \|_2\\
	\leq & ~ \kappa e^{-(\kappa^2 \Lambda_0+\lambda)T}\| \beta(0) - \beta^* \|_2 \\
	\leq & ~ e^{-(\kappa^2 \Lambda_0+\lambda)T} \cdot \poly(\kappa,n,1/\Lambda_0)
\end{align*}
where the last step follows from $\beta(0) = 0$ and $\|\beta^*\|_2 = \poly(\kappa,n, 1/\Lambda_0)$. 

Note $\kappa\in(0,1)$. Thus, by picking $ T = \wt{O}(\frac{1}{\kappa^2\Lambda_0})$, we have 
\begin{align*}
	\| u_{\ntk,\test}(T) - u_{\test}^* \|_2 \leq \epsilon/2,
\end{align*}
where $\wt{O}(\cdot)$ here hides $\poly\log( n/ (\epsilon \Lambda_0) )$.% the poly-log factors in $\frac{1}{\epsilon}$, $n$, $\Lambda_0$ and $\lambda$.
\end{proof}

\subsubsection{Upper bounding $| u_{\nn,\test}(T) - u_{\ntk,\test}(T) |$ by bounding initialization and kernel perturbation}\label{sec:equiv_splitting_nn_test_T_and_ntk_test_T_into_three}
The goal of this section is to prove Lemma~\ref{lem:more_concreate_bound}, which reduces the problem of bounding prediction perturbation to the problem of bounding initialization perturbation and kernel perturbation.

\begin{lemma}[Prediction perturbation implies kernel perturbation]\label{lem:more_concreate_bound}
Given training data matrix $X \in \R^{n \times d}$ and corresponding label vector $Y \in \R^n$. Fix the total number of iterations $T > 0$. Given arbitrary test data $x_{\test} \in \R^d$. Let $u_{\nn,\test}(t) \in \R^n$ and $u_{\ntk,\test}(t) \in \R^n$ be the test data predictors defined in Definition~\ref{def:nn} and Definition~\ref{def:krr_ntk} respectively. Let $\kappa\in(0,1)$ be the corresponding multiplier. Let $\k_{\ntk}(x_{\test},X) \in \R^n,~\k_{t}(x_{\test},X) \in \R^n,~H^{\cts} \in \R^{n \times n},~H(t) \in \R^{n \times n},~\Lambda_0 > 0$ be defined in Definition~\ref{def:ntk_phi} and Definition~\ref{def:dynamic_kernel}. Let $u^* \in \R^n$ be defined as in Eq.~\eqref{eq:def_u_*}. Let $\lambda > 0$ be the regularization parameter.
Let $\epsilon_{K} \in (0,1)$, $\epsilon_{\init} \in (0,1)$ and $\epsilon_H \in (0,1)$ denote parameters that are independent of $t$, and the following conditions hold for all $t\in[0,T]$,
\begin{itemize}
	\item $\|u_{\nn}(0)\|_2 \leq \sqrt{n}\epsilon_{\init}$ and $|u_{\nn,\test}(0)| \leq \epsilon_{\init}$
    \item $\|\k_{\ntk}(x_{\test},X)-\k_{t}(x_{\test},X)\|_2\le \epsilon_{K}$
    \item $\| H(t) - H^{\cts} \| \le \epsilon_H$
    % \item $\Lambda_{\min} (H^{\cts}) = 2\Lambda_0$ 
\end{itemize}
then we have
\begin{align*}
    |u_{\nn,\test}(T)-u_{\ntk,\test}(T)| \leq & ~ (1+\kappa^2 nT)\epsilon_{\init} + \kappa^2 \epsilon_K \cdot \Big( \frac{ \| u^* \|_2 }{ \kappa^2 \Lambda_0 + \lambda } + \|u^*-Y\|_2 T\Big)\\
    & ~ + \sqrt{n}T^2\kappa^4 \epsilon_H ( \| u^* \|_2 + \| u^* - Y \|_2 )
\end{align*}
\end{lemma}
\begin{proof}
%\begin{proof}[Proof of Lemma~\ref{lem:more_concreate_bound}]
Combining results from Lemma~\ref{lem:very_rough_bound}, Claim~\ref{cla:A}.~\ref{cla:B},~\ref{cla:C}, we complete the proof.
%\end{proof}
We have
\begin{align*}
|u_{\nn,\test}(T)-u_{\ntk,\test}(T)| 
\leq & ~ |u_{\nn,\test}(0)-u_{\ntk,\test}(0)|\\
& ~ + \kappa^2 \Big| \int_{0}^T (\k_{\ntk}(x_{\test},X)-\k_{t}(x_{\test},X))^\top (u_{\ntk}(t)-Y) \d t \Big|\\
& ~ + \kappa^2 \Big| \int_{0}^T \k_{t}(x_{\test},X)^\top(u_{\ntk}(t)-u_{\nn}(t)) \d t \Big|\\
\leq & ~ \epsilon_{\init} + \kappa^2 \epsilon_K \cdot \Big( \frac{ \|u^*\| }{ \kappa^2 \Lambda_0 + \lambda } + \|u^*-Y\|_2 T\Big)\\
& ~ + \kappa^2 n\epsilon_{\init}T + \sqrt{n}T^2 \cdot  \kappa^4 \epsilon_H \cdot ( \| u^* \|_2 + \| u^* - Y \|_2 )\\
\leq & ~ (1+ \kappa^2 nT)\epsilon_{\init} + \kappa^2 \epsilon_K \cdot \Big( \frac{ \| u^* \|_2 }{ \kappa^2 \Lambda_0 + \lambda } + \|u^*-Y\|_2 T\Big)\\
    & ~ + \sqrt{n}T^2\kappa^4 \epsilon_H ( \| u^* \|_2 + \| u^* - Y \|_2 )
\end{align*}
where the first step follows from Lemma~\ref{lem:very_rough_bound}, the second step follows from Claim~\ref{cla:A},~\ref{cla:B} and \ref{cla:C}, and the last step simplifies the expression.
\end{proof}

To prove Lemma~\ref{lem:more_concreate_bound}, we first bound $| u_{\nn,\test}(T) - u_{\ntk,\test}(T) |$ by three terms in Lemma~\ref{lem:very_rough_bound}, then we bound each term individually in Claim~\ref{cla:A}, Claim~\ref{cla:B}, and Claim~\ref{cla:C}.

\begin{lemma}\label{lem:very_rough_bound}
Follow the same notation as~Lemma~\ref{lem:more_concreate_bound}, we have
\begin{align*}
| u_{\nn,\test}(T) - u_{\ntk,\test}(T) | \leq A + B + C,
\end{align*}
where
\begin{align*}
    A = & ~ |u_{\nn,\test}(0)-u_{\ntk,\test}(0)|\\
    B = & ~  \kappa^2 \Big| \int_{0}^T (\k_{\ntk}(x_{\test},X)-\k_{t}(x_{\test},X))^\top (u_{\ntk}(t)-Y) \d t \Big|\\
    C = & ~  \kappa^2 \Big| \int_{0}^T \k_{t}(x_{\test},X)^\top(u_{\ntk}(t)-u_{\nn}(t)) \d t \Big|
\end{align*}
\end{lemma}
\begin{proof}

\begin{align}\label{eq:320_3}
    & ~|u_{\nn,\test}(T)-u_{\ntk,\test}(T)|\notag\\
    = & ~ \Big| u_{\nn,\test}(0)-u_{\ntk,\test}(0)+\int_{0}^T(\frac{\d u_{\nn,\test}(t)}{\d t}-\frac{\d u_{\ntk,\test}(t)}{\d t})\d t \Big |\notag\\
    \leq & ~ | u_{\nn,\test}(0)-u_{\ntk,\test}(0) |+ \Big| \int_{0}^T(\frac{\d u_{\nn,\test}(t)}{\d t}-\frac{\d u_{\ntk,\test}(t)}{\d t}) \d t \Big|,
\end{align}
where the first step follows from the definition of integral, the second step follows from the triangle inequality. 
Note by Corollary~\ref{cor:ntk_gradient},~\ref{cor:nn_gradient}, their gradient flow are given by 
\begin{align}
    \frac{\d u_{\ntk,\test}(t)}{\d t}&= - \kappa^2 \k_{\ntk}(x_{\test},X)^\top(u_{\ntk}(t)-Y)-\lambda u_{\ntk,\test}(t)\label{eq:320_1}\\
    \frac{\d u_{\nn,\test}(t)}{\d t}&= - \kappa^2 \k_{t}(x_{\test},X)^\top(u_{\nn}(t)-Y)-\lambda u_{\nn,\test}(t)\label{eq:320_2}
\end{align}
where $u_{\ntk}(t) \in \R^n$ and $u_{\nn}(t) \in \R^n$ are the predictors for training data defined in Definition~\ref{def:krr_ntk} and Definition~\ref{def:nn}. Thus, we have 
\begin{align}\label{eq:320_4}
    & ~ \frac{\d u_{\nn,\test}(t)}{\d t}-\frac{\d u_{\ntk,\test}(t)}{\d t} \notag\\
    = & ~ - \kappa^2 \k_{t}(x_{\test},X)^\top(u_{\nn}(t)-Y)+ \kappa^2 \k_{\ntk}(x_{\test},X)^\top(u_{\ntk}(t)-Y) -\lambda (u_{\nn,\test}(t)-u_{\ntk,\test}(t)) \notag\\
    = & ~  \kappa^2 (\k_{\ntk}(x_{\test},X)- \k_{t}(x_{\test},X))^\top (u_{\ntk}(t)-Y)-  \kappa^2 \k_{t}(x_{\test},X)^\top(u_{\ntk}(t)-u_{\nn}(t)) \notag\\
    & ~ -\lambda (u_{\nn,\test}(t)-u_{\ntk,\test}(t)),
\end{align}
where the first step follows from Eq.~\eqref{eq:320_1} and Eq.~\eqref{eq:320_2}, the second step rewrites the formula. 
Note the term $-\lambda (u_{\nn,\test}(t)-u_{\ntk,\test}(t))$ will only make 
\begin{align*}
\Big| \int_{0}^T(\frac{\d u_{\nn,\test}(t)}{\d t}-\frac{\d u_{\ntk,\test}(t)}{\d t})\d t \Big|
\end{align*}
smaller, so we have
\begin{align}\label{eq:320_5}
    & ~ \Big|\int_{0}^T (\frac{\d u_{\nn,\test}(t)}{\d t}-\frac{\d u_{\ntk,\test}(t)}{\d t}) \d t\Big| \notag\\
    \leq & ~ \Big|\int_{0}^T  \kappa^2 ((\k_{\ntk}(x_{\test},X)-\k_{t}(x_{\test},X))^\top (u_{\ntk}(t)-Y)- \kappa^2  \k_{t}(x_{\test},X)^\top(u_{\ntk}(t)-u_{\nn}(t))) \d t\Big|
\end{align}
Thus,
\begin{align*}
    & ~ |u_{\nn,\test}(T)-u_{\ntk,\test}(T)|\\
    \leq & ~ |u_{\nn,\test}(0)-u_{\ntk,\test}(0)| + \Big| \int_{0}^T(\frac{\d u_{\nn,\test}(t)}{\d t}-\frac{\d u_{\ntk,\test}(t)}{\d t}) \d t \Big|\\
    \leq & ~ |u_{\nn,\test}(0)-u_{\ntk,\test}(0)| + \Big|\int_{0}^T  \kappa^2 ((\k_{\ntk}(x_{\test},X)-\k_{t}(x_{\test},X))^\top (u_{\ntk}(t)-Y) \\
    & ~ -  \kappa^2 \k_{t}(x_{\test},X)^\top(u_{\ntk}(t)-u_{\nn}(t))) \d t\Big|\\
    \leq & ~ |u_{\nn,\test}(0)-u_{\ntk,\test}(0)| + \Big| \int_{0}^T \kappa^2 (\k_{\ntk}(x_{\test},X)-\k_{t}(x_{\test},X))^\top (u_{\ntk}(t)-Y) \d t \Big|\\
    & ~ + \Big| \int_{0}^T  \kappa^2 \k_{t}(x_{\test},X)^\top(u_{\ntk}(t)-u_{\nn}(t)) \d t \Big|\\
    = & ~ A + B + C,
\end{align*}
where the first step follows from Eq.~\eqref{eq:320_3}, the second step follows from Eq.~\eqref{eq:320_5}, the third step follows from triangle inequality, and the last step follows from the definition of $A,~B,~C$.
\end{proof}

Now let us bound these three terms $A$, $B$ and $C$ one by one. We claim

% {\bf Bounding the term $A$.}
\begin{claim}[Bounding the term $A$]\label{cla:A}\
We have
\begin{align*}
	A \leq \epsilon_{\init}.
\end{align*}
\end{claim}
\begin{proof}
	Note $u_{\ntk,\test}(0) = 0$, so by assumption we have 
\begin{align*}
	A=|u_{\nn,\test}(0)| \leq \epsilon_{\init}.
\end{align*}

\end{proof}

\begin{claim}[Bounding the term $B$]\label{cla:B}
We have
\begin{align*}
B \leq  \kappa^2 \epsilon_K \cdot \Big( \frac{ \|u^*\| }{ \kappa^2 \Lambda_0 + \lambda } + \|u^*-Y\|_2 T\Big) .
\end{align*}
\end{claim}

\begin{proof}
Note
\begin{align*}
	B = & ~ \kappa^2 \Big| \int_{0}^T (\k_{\ntk}(x_{\test},X)-\k_{t}(x_{\test},X))^\top (u_{\ntk}(t)-Y) \d t \Big|\\
    \le & ~ \kappa^2 \max_{ t \in [ 0 , T ] }\|\k_{\ntk}(x_{\test},X)-\k_{t}(x_{\test},X)\|_2 \int_{0}^T \|u_{\ntk}(t)-Y\|_2 \d t ,
\end{align*}
where the first step follows from definition of $B$, and the second step follows from the Cauchy-Schwartz inequality.
Note by Lemma~\ref{lem:linear_converge_krr}, the kernel ridge regression predictor $u_{\ntk}(t) \in \R^n$ converges linearly to the optimal predictor $u^*=\kappa^2 H^{\cts}(\kappa^2 H^{\cts}+\lambda I)^{-1}Y \in \R^n$, i.e.,
\begin{align}\label{eq:320_6}
	\|u_{\ntk}(t) - u^*\|_2 \leq e^{-(\kappa^2 \Lambda_0+\lambda)t} \|u_{\ntk}(0) - u^*\|_2.
\end{align}
Thus, we have
\begin{align}\label{eq:upper_bound_int_0_T_u_ntk_t_minus_Y}
    \int_{0}^T \|u_{\ntk}(t)-Y\|_2 \d t 
    \leq & ~ \int_0^T \| u_{\ntk}(t) - u^* \|_2 \d t + \int_0^T \|u^*-Y\|_2 \d t \notag \\
    \le & ~ \int_{0}^T e^{-(\kappa^2 \Lambda_0+\lambda)} \|u_{\ntk}(0)-u^*\|_2 \d t + \int_{0}^T \|u^*-Y\|_2 \d t \notag \\
    % \le & ~ \int_{0}^\infty e^{-(\Lambda_0+\lambda)} \|u_{\ntk}(0)-u^*\|_2 \d t + \int_{0}^T \|u^*-Y\|_2 \d t \notag \\
    \leq & ~  \frac{\|u_{\ntk}(0)-u^*\|_2}{\kappa^2 \Lambda_0+\lambda} + \|u^*-Y\|_2 T  \notag \\
    = & ~   \frac{ \|u^*\| }{\kappa^2 \Lambda_0 + \lambda } + \|u^*-Y\|_2 T,
\end{align}
where the first step follows from the triangle inequality, the second step follows from Eq.~\eqref{eq:320_6}, the third step calculates the integration, and the last step follows from the fact $u_{\ntk}(0) = 0$.
Thus, we have
\begin{align*}
    B \le & ~ \kappa^2 \max_{t \in [0, T]}\|\k_{\ntk}(x_{\test},X)-\k_{t}(x_{\test},X)\|_2 \cdot \int_0^T \| u_{\ntk}(t) - Y \|_2 \d t \\
    \leq & ~ \kappa^2 \epsilon_K \cdot \Big( \frac{ \|u^*\| }{ \kappa^2 \Lambda_0 + \lambda } + \|u^*-Y\|_2 T \Big) .
\end{align*}
where the first step follows from Eq.~\eqref{eq:320_6}, the second step follows from Eq.~\eqref{eq:upper_bound_int_0_T_u_ntk_t_minus_Y} and definition of $\epsilon_K$.
\end{proof}

%{\bf Bounding the term $C$.} 

\begin{claim}[Bounding the term $C$]\label{cla:C}
We have
\begin{align*}
C \leq n\epsilon_{\init}T + \sqrt{n}T^2 \cdot  \kappa^2 \epsilon_H \cdot ( \| u^* \|_2 + \| u^* - Y \|_2 )
\end{align*}
\end{claim}
\begin{proof}
Note 
\begin{align}\label{eq:320_10}
	C = & ~ \kappa^2 \Big| \int_{0}^T \k_{t}(x_{\test},X)^\top(u_{\ntk}(t)-u_{\nn}(t)) \d t \Big| \notag\\
    \le & ~ \kappa^2 \max_{ t \in [0,T] }\|\k_t(x_{\test},X)\|_2 \max_{ t \in [0,T] } \|u_{\ntk}(t)-u_{\nn}(t)\|_2\cdot T
\end{align}
where the first step follows from the definition of $C$, and the second step follows the Cauchy-Schwartz inequality.

% \begin{align*}
%     C 
%     \le & ~ \max_{ t \in [0,T] }\|\k_t(x_{\test},X)\|_2 \cdot \int_{0}^T \|u_{\ntk}(t)-u_{\nn}(t)\|_2 \d t\\
%     \le & ~ \max_{ t \in [0,T] }\|\k_t(x_{\test},X)\|_2 \cdot \Big( \int_{0}^{T_0} \|u_{\ntk}(t)-u_{\nn}(t)\|_2 \d t + \int_{T_0}^T \|u_{\ntk}(t)-u_{\nn}(t)\|_2 \d t \Big) .
% \end{align*}
% We say the integration from $0$ to $T_0$ is the first phase, and the integration from $T_0$ to $T$ is the second phase. 

% For the second phase, note $u_{\ntk}(t) \in \R^n$ and $u_{\nn}(t) \in \R^n$ tend to linearly converge to $u^* \in \R^n$, so we have
% \begin{align*}
%     & ~ \int_{T_0}^T \|u_{\ntk}(t)-u_{\nn}(t)\|_2 \d t\\
%     \le & ~ \int_{T_0}^T \|u_{\ntk}(t)-u^*\|_2 \d t + \int_{T_0}^T \|u_{\nn}(t)-u^*\|_2 \d t\\
%     \le & ~ \int_{T_0}^T e^{-(\Lambda_0+\lambda)t}\|u_{\ntk}(0)-u^*\|_2 \d t + \int_{T_0}^T e^{-(\Lambda_0+\lambda)t/c}\|u_{\nn}(0)-u^*\|_2 \d t\\
%     \le & ~ \frac{\|u_{\ntk}(0)-u^*\|_2}{\Lambda_0+\lambda}e^{-(\Lambda_0+\lambda)T_0} + \frac{2\|u_{\nn}(0)-u^*\|_2}{\Lambda_0+\lambda}e^{-(\Lambda_0+\lambda)T_0/2}
% \end{align*}
% By picking $T_0 = O (\frac{1}{\Lambda_0+\lambda}\log{\frac{\|u^*\|_2}{\epsilon_H(\Lambda_0+\lambda)}})$, we have 
% \begin{align*}
%     \int_{T_0}^T \|u_{\ntk}(t)-u_{\nn}(t)\|_2 \d t \le O(\epsilon_{H}),
% \end{align*}
% which bounds the second phase (the integration from $T_0$ to $T$).

% For the first phase (the integration from $0$ to $T_0$), note
% \begin{align*}
%     \int_{0}^{T_0} \|u_{\ntk}(t)-u_{\nn}(t)\|_2 \d t \le T_0 \cdot \max_{ t \in [0, T_0] } \| u_{\ntk}(t)-u_{\nn}(t) \|_2.
% \end{align*}
To bound term $\max_{ t \in [0,T] } \|u_{\ntk}(t)-u_{\nn}(t)\|_2$, notice that for any $t\in[0,T]$, we have
\begin{align}\label{eq:320_7}
    \|u_{\ntk}(t)-u_{\nn}(t)\|_2 \leq & ~\|u_{\ntk}(0)-u_{\nn}(0)\|_2 + \Big\|\int_{0}^{t}  \frac{\d (u_{\ntk}(\tau)-u_{\nn}(\tau))}{\d \tau} \d \tau \Big\|_2 \notag\\
    = & ~ \sqrt{n} \epsilon_{\init} + \Big\|\int_{0}^{t}  \frac{\d (u_{\ntk}(\tau)-u_{\nn}(\tau))}{\d \tau} \d \tau \Big\|_2,
\end{align}
where the first step follows the triangle inequality, and the second step follows the assumption.
Further, 
\begin{align*}
    \frac{\d (u_{\ntk}(\tau)-u_{\nn}(\tau))}{\d \tau} &= - \kappa^2 H^{\cts}(u_{\ntk}(\tau)-Y)-\lambda u_{\ntk}(\tau) + \kappa^2 H(\tau)(u_{\nn}(\tau)-Y) +\lambda u_{\nn}(\tau)\\
    &=-( \kappa^2 H(\tau)+\lambda I)(u_{\ntk}(\tau)-u_{\nn}(\tau)) +  \kappa^2 (H(\tau)-H^{\cts})(u_{\ntk}(\tau)-Y),
\end{align*}
where the first step follows the Corollary~\ref{cor:ntk_gradient},~\ref{cor:nn_gradient}, the second step rewrites the formula.
Since the term $-( \kappa^2 H(\tau)+\lambda I)(u_{\ntk}(\tau)-u_{\nn}(\tau))$ makes $\|\int_{0}^t \frac{\d (u_{\ntk}(\tau)-u_{\nn}(\tau))}{\d \tau} \d \tau\|_2$ smaller. 

Taking the integral and apply the $\ell_2$ norm, we have
\begin{align}\label{eq:320_8}
	\Big\| \int_{0}^t \frac{\d (u_{\ntk}(\tau)-u_{\nn}(\tau))}{\d \tau} \d \tau \Big\|_2 \leq \Big\| \int_0^t  \kappa^2 (H(\tau)-H^{\cts})(u_{\ntk}(\tau)-Y) \d \tau \Big\|_2.
\end{align}
Thus,
\begin{align}\label{eq:320_9}
    \max_{t\in[0,T]} \|u_{\ntk}(t)-u_{\nn}(t)\|_2
    \leq & ~ \sqrt{n} \epsilon_{\init} + \max_{t\in[0,T]}  \Big\| \int_{0}^t \frac{\d (u_{\ntk}(\tau)-u_{\nn}(\tau))}{\d \tau} \d \tau \Big\|_2 \notag \\
    \leq & ~ \sqrt{n} \epsilon_{\init} + \max_{t\in[0,T]} \Big\| \int_0^t  \kappa^2 (H(\tau)-H^{\cts})(u_{\ntk}(\tau)-Y) \d \tau \Big\|_2 \notag \\
    \leq & ~ \sqrt{n} \epsilon_{\init} + \max_{t\in[0,T]} \int_0^t  \kappa^2 \| H(\tau) - H^{\cts} \| \cdot \| u_{\ntk}(\tau) - Y \|_2 \d \tau \notag \\
    \leq & ~ \sqrt{n} \epsilon_{\init} + \max_{t\in[0,T]}  \kappa^2 \epsilon_H \Big( \int_{0}^t  \| u_{\ntk}(\tau) - u^* \|_2 \d \tau +\int_{0}^t \| u^* - Y \|_2 \d \tau \Big) \notag\\
    \leq & ~ \sqrt{n} \epsilon_{\init} + \max_{t\in[0,T]}  \kappa^2 \epsilon_H \Big( \int_{0}^t  \| u_{\ntk}(0) - u^* \|_2 \d \tau +\int_{0}^t \| u^* - Y \|_2 \d \tau \Big) \notag\\
    \leq & ~ \sqrt{n} \epsilon_{\init} + \max_{t\in[0,T]} t \cdot  \kappa^2 \epsilon_H \cdot (\|u^*\|_2 + \|u^*-Y\|_2) \notag \\
    \leq & ~ \sqrt{n} \epsilon_{\init} + T \cdot  \kappa^2 \epsilon_H \cdot (\|u^*\|_2 + \|u^*-Y\|_2)
\end{align}
where the first step follows from Eq.~\eqref{eq:320_7}, the second step follows from Eq.~\eqref{eq:320_8}, the third step follows from triangle inequality, the fourth step follows from the condition $\| H(\tau) - H^{\cts} \| \leq \epsilon_H$ for all $\tau \leq T$ and the triangle inequality, the fifth step follows from the linear convergence of $\| u_{\ntk}(\tau) - u^* \|_2$ as in Lemma~\ref{lem:linear_converge_krr}, the sixth step follows the fact $u_{\ntk}(0) = 0$, and the last step calculates the maximum.
Therefore,
\begin{align*}
     C \leq & ~ \kappa^2 \max_{ t \in [0,T] }\|\k_t(x_{\test},X)\|_2 \max_{ t \in [0,T] } \|u_{\ntk}(t)-u_{\nn}(t)\|_2\cdot T \\
     \leq & ~ \kappa^2 \max_{ t \in [0,T] } \| \k_t (x_{\test} , X) \|_2 \cdot (\sqrt{n} \epsilon_{\init}T + T^2 \cdot  \kappa^2 \epsilon_H \cdot ( \| u^* \|_2 + \| u^* - Y \|_2 )) \\
     \leq & ~ \kappa^2 n\epsilon_{\init}T + \sqrt{n}T^2 \cdot  \kappa^4 \epsilon_H \cdot ( \| u^* \|_2 + \| u^* - Y \|_2 )
\end{align*}
where the first step follows from Eq.~\eqref{eq:320_10}, and the second step follows from Eq.~\eqref{eq:320_9}, and the last step follows from the fact that $\k_{t}(x,z) \leq 1$ holds for any $\|x\|_2,~\|z\|_2 \leq 1$.
\end{proof}

\begin{remark}
	Given final accuracy $\epsilon$, to ensure $|u_{\nn,\test}(T)-u_{\ntk,\test}(T)| \leq \epsilon$, we need to choose $\kappa>0$ small enough to make $\epsilon_{\init}=O(\epsilon)$ and choose width $m>0$ large enough to make $\epsilon_H$ and $\epsilon_{\test}$ both $O(\epsilon)$. And we discuss these two tasks one by one in the following sections.
\end{remark}

\subsubsection{Upper bounding initialization perturbation}\label{sec:equiv_intialization}
In this section, we bound $\epsilon_{\init}$ to our wanted accuracy $\epsilon$ by picking $\kappa$ large enough. We prove Lemma~\ref{lem:epsilon_init}.

\begin{lemma}[Bounding initialization perturbation]\label{lem:epsilon_init}
Let $f_{\nn}$ be as defined in Definition~\ref{def:f_nn}. Assume the initial weight of the network work $w_r(0)\in\R^d,~r=1,\cdots,m$ as defined in Definition~\ref{def:nn} are drawn independently from standard Gaussian distribution $\mathcal{N}(0,I_d)$. And $a_r\in\R,~r=1,\cdots,m$ as defined in Definition~\ref{def:f_nn} are drawn independently from $\unif[\{-1,+1\}]$. Let $\kappa\in(0,1)$, $u_{\nn}(t)$ and $u_{\nn,\test}(t)\in\R$ be defined as in Definition~\ref{def:nn}. Then for any data $x\in\R^d$ with $\|x\|_2\leq 1$, we have with probability $1-\delta$,
\begin{align*}
	|f_{\nn}(W(0),x)| \leq 2 \log(2m/\delta).
\end{align*}
Further, given any accuracy $\epsilon\in(0,1)$, if $\kappa = \wt{O}(\epsilon(\Lambda_0+\lambda)/n)$, let $\epsilon_{\init} = \epsilon(\Lambda_0+\lambda)/n$, we have
\begin{align*}
	|u_{\nn}(0)| \leq \sqrt{n}\epsilon_{\init}~\text{and}~|u_{\nn,\test}(0)| \leq \epsilon_{\init}
\end{align*}
hold with probability $1-\delta$, where $\wt{O}(\cdot)$ hides the $\poly\log( n / ( \epsilon  \delta  \Lambda_0 ) )$.
\end{lemma}
\begin{proof}
Note by definition,
\begin{align*}
	f_{\nn}(W(0), x) = \frac{1}{\sqrt{m}}\sum_{r=1}^m a_r\sigma(w_r(0)^\top x).
\end{align*}
Since $w_r(0)\sim\mathcal{N}(0,I_d)$, so $w_r(0)^\top x_{\test}\sim N(0,\|x_{\test}\|_2)$. Note $\|x_{\test}\|_2 \leq 1$, by Gaussian tail bounds~Lemma~\ref{lem:gaussian_tail}, we have with probability $1-\delta / (2m)$:
\begin{align}\label{eq:guassian_tail}
	|w_r(0)^\top x| \leq \sqrt{ 2 \log( 2 m / \delta ) }.
\end{align}
Condition on Eq.~\eqref{eq:guassian_tail} holds for all $r\in[m]$, denote $Z_r = a_r\sigma(w_r(0)^\top x)$, then we have $\E[Z_r] = 0$ and $|Z_r| \leq \sqrt{2\log(2m/\delta)}$. By Lemma~\ref{lem:hoeffding}, with probability $1-\delta/2$:
\begin{align}\label{eq:heoffding}
	\Big| \sum_{r=1}^m Z_r \Big| \leq 2 \sqrt{m}\log{ ( 2 m / \delta ) }.
\end{align}
Since $u_{\nn,\test}(0) = \frac{1}{\sqrt{n}}\sum Z_r $, by combining Eq.~\eqref{eq:guassian_tail},~\eqref{eq:heoffding} and union bound over all $r\in[m]$, we have with probability $1-\delta$:
\begin{align*}
	|u_{\nn,\test}(0)| \leq 2\log( 2 m / \delta ).
\end{align*}
Further, note $[u_{\nn}(0)]_i = \kappa f_{\nn}(W(0),x_i)$ and $u_{\nn,\test}(0) = \kappa f_{\nn}(W(0),x_{\test})$. Thus, by choosing $\kappa = \wt{O}(\epsilon(\Lambda_0+\lambda)/n)$, taking the union bound over all training and test data, we have
\begin{align*}
	|u_{\nn}(0)| \leq \sqrt{n}\epsilon_{\init}~\text{and}~|u_{\nn,\test}(0)| \leq \epsilon_{\init}
\end{align*}
hold with probability $1-\delta$, where $\wt{O}(\cdot)$ hides the $\poly\log( n / ( \epsilon  \delta  \Lambda_0 ) )$.
\end{proof}

\subsubsection{Upper bounding kernel perturbation}\label{sec:equiv_kernel_perturbation}
In this section, we try to bound kernel perturbation by induction. We want to prove Lemma~\ref{lem:induction}, which also helps to show the equivalence for training data prediction as shown in Section~\ref{sec:train}.

% We define our Hypothesis induction
\begin{lemma}[Bounding kernel perturbation]\label{lem:induction}
Given training data $X\in\R^{n\times d}$, $Y\in\R^n$ and a test data $x_{\test}\in\R^d$. Let $T > 0$ denotes the total number of iterations, $m >0 $ denotes the width of the network,  $\epsilon_{\train}$ denotes a fixed training error threshold, $\delta > 0$ denotes the failure probability. Let $u_{\nn}(t) \in \R^n$ and $u_{\ntk}(t) \in \R^n$ be the training data predictors defined in Definition~\ref{def:nn} and Definition~\ref{def:krr_ntk} respectively. Let $\kappa\in(0,1)$ be the corresponding multiplier. Let $\k_{\ntk}(x_{\test},X) \in \R^n,~\k_{t}(x_{\test},X) \in \R^n,~H(t) \in \R^{n \times n},~\Lambda_0 > 0$ be the kernel related quantities defined in Definition~\ref{def:ntk_phi} and Definition~\ref{def:dynamic_kernel}. Let $u^* \in \R^n$ be defined as in Eq.~\eqref{eq:def_u_*}. Let $\lambda > 0$ be the regularization parameter. Let $W(t) = [w_1(t),\cdots,w_m(t)]\in\R^{d\times m}$ be the parameters of the neural network defined in Definition~\ref{def:nn}.

For any accuracy $\epsilon\in(0,1/10)$. If $\kappa=\wt{O}(\frac{\epsilon\Lambda_0}{n})$, $T=\wt{O}(\frac{1}{\kappa^2(\Lambda_0+\lambda)})$, $\epsilon_{\train} = \wt{O}(\|u_{\nn}(0)-u^*\|_2)$, $m \geq\wt{O}(\frac{n^{10} d}{\epsilon^6 \Lambda_0^{10}})$ and $\lambda=\wt{O}(\frac{1}{\sqrt{m}})$, with probability $1-\delta$, there exist $\epsilon_W,~\epsilon_H',~\epsilon_K'>0$ that are independent of $t$, such that the following hold for all $0 \leq t \le T$:
\begin{itemize}
    \item 1. $\| w_r(0) - w_r(t) \|_2 \leq \epsilon_W $, $\forall r \in [m]$
    \item 2. $\| H(0) - H(t) \|_2 \leq \epsilon_H'$ 
    \item 3. $\| u_{\nn}(t) - u^* \|_2^2 \leq \max\{\exp(-(\kappa^2\Lambda_0 + \lambda) t/2) \cdot \| u_{\nn}(0) - u^* \|_2^2, ~ \epsilon_{\train}^2\}$
    \item 4. $\| \k_0( x_{\test} , X )- \k_{t} (x_{\test}, X) \|_2 \leq \epsilon_K'$
\end{itemize}
Further, $\epsilon_W \leq \wt{O}(\frac{\epsilon \lambda_0^2}{n^2})$,  $\epsilon_H' \leq \wt{O}(\frac{\epsilon \lambda_0^2}{n})$ and  $\epsilon_K' \leq \wt{O}(\frac{\epsilon \lambda_0^2}{n^{1.5}})$. Here $\wt{O}(\cdot)$ hides the  $\poly\log( n / ( \epsilon  \delta  \Lambda_0) )$.
\end{lemma}

% \subsubsection{Random initialization}
We first state some concentration results for the random initialization that can help us prove the lemma.

\begin{lemma}[Random initialization result]\label{lem:random_init}
Assume initial value $w_r(0) \in \R^d ,~r=1,\cdots,m$ are drawn independently from standard Gaussian distribution $\mathcal{N}(0,I_d)$, then with probability $1-3\delta$ we have
\begin{align}
	\|w_r(0)\|_2 \leq & ~ 2\sqrt{d} + 2\sqrt{\log{(m/\delta)}}~\text{for all}~r\in[m]\label{eq:3322_1}\\
	\|H(0)- H^{\cts} \| \leq & ~ 4n ( \log(n/\delta) / m )^{1/2}\label{eq:3322_2}\\
	\|\k_{0}( x_{\test} , X ) - \k_{\ntk} ( x_{\test} , X )\|_2 \leq & ~ ( 2n \log{(2n/\delta)} / m )^{1/2}\label{eq:3322_3}
\end{align}
\end{lemma}
\begin{proof}
By lemma~\ref{lem:chi_square_tail}, with probability at least $1-\delta$,
\begin{align*}
	\| w_r(0) \|_2 \leq \sqrt{d} + \sqrt{\log(m/\delta)}
\end{align*}
holds for all $r\in[m]$.\\

Using Lemma~\ref{lem:lemma_4.1_in_sy19} in \cite{sy19}, we have
%Using Lemma 4.1 in \cite{sy19}, we have
\begin{align*}
    \| H(0) - H^{\cts} \| \leq \epsilon_H'' = 4 n ( \log{(n/\delta)} / m )^{1/2}
\end{align*}
holds with probability at least $1-\delta$.\\
Note by definition,
\begin{align*}
    \E[\k_0 ( x_{\test}, x_i )] =  \k_{\ntk} ( x_{\test}, x_i )
\end{align*}
holds for any training data $x_i$. By Hoeffding inequality, we have for any $t>0$,
\begin{align*}
    \Pr[|\k_0 ( x_{\test}, x_i ) - \k_{\ntk} ( x_{\test}, x_i )|\ge t] \le 2\exp{(-mt^2/2)}.
\end{align*}
Setting $t=(\frac{2}{m}\log{(2n/\delta)})^{1/2}$, we can apply union bound on all training data $x_i$ to get with probability at least $1-\delta$, for all $i\in[n]$,
\begin{align*}
    |\k_0 ( x_{\test}, x_i ) - \k_{\ntk} ( x_{\test}, x_i )| \le (2 \log(2n/\delta) / m)^{1/2}.
\end{align*}
Thus, we have
\begin{align}
    \| \k_0 ( x_{\test}, X ) - \k_{\ntk} ( x_{\test}, X ) \|_2 \le ( 2n\log(2n/\delta) / m )^{1/2}
\end{align}
holds with probability at least $1-\delta$.\\
Using union bound over above three events, we finish the proof.
\end{proof}

Now conditioning on Eq.~\eqref{eq:3322_1},~\eqref{eq:3322_2},~\eqref{eq:3322_3} holds, We show all the four conclusions in Lemma~\ref{lem:induction} holds using induction.

We define the following quantity:
\begin{align}
	\epsilon_W :=  & ~ \frac{ \sqrt{n} }{ \sqrt{m} } \max\{4\| u_{\nn}(0) - u^* \|_2/(\kappa^2\Lambda_0+\lambda), \epsilon_{\train} \cdot T\} \notag \\
	& ~ + \Big( \frac{ \sqrt{n} }{ \sqrt{m} } \| Y-u^* \|_2 + \lambda (2\sqrt{d} + 2\sqrt{\log(m/\delta)}) \Big) \cdot T\label{eq:def_epsilon_W}\\
    \epsilon_H' := & ~ 2n\epsilon_W\notag \\
    \epsilon_K := & ~ 2\sqrt{n}\epsilon_W\notag
\end{align}
which are independent of $t$. 

Note the base case when $t=0$ trivially holds. Now assuming Lemma~\ref{lem:induction} holds before time $t\in[0,T]$, we argue that it also holds at time $t$. To do so, Lemmas~\ref{lem:hypothesis_1},~\ref{lem:hypothesis_2},~\ref{lem:hypothesis_3} argue these conclusions one by one.

% \subsubsection{Induction hypothesis 1}

\begin{lemma}[Conclusion 1]\label{lem:hypothesis_1}
% Let $c$ denote a fixed constant. 
If for any $\tau < t$, we have
% $m \geq  ??? \cdot 1/\epsilon_W$ and 
\begin{align*}
    \| u_{\nn}(\tau) - u^* \|_2^2 \leq ~ \max\{\exp(-(\kappa^2 \Lambda_0 + \lambda) \tau/2) \cdot \| u_{\nn}(0) - u^* \|_2^2,~\epsilon_{\train}^2\}
\end{align*}
and
% \begin{align*}
%     \| w_r(0) - w_r(\tau) \|_2 \leq 1
% \end{align*}
\begin{align*}
     \| w_r(0) - w_r(\tau) \|_2 \leq \epsilon_W \leq 1
\end{align*}
and
\begin{align*}
	\| w_r(0) \|_2 \leq ~ \sqrt{d} + \sqrt{\log(m/\delta)} ~ \text{for all}~r\in[m]% \text{ for some constant } c
\end{align*}
hold,
then
\begin{align*}
    \| w_r(0) - w_r(t) \|_2 \leq \epsilon_W 
    % := & ~ \frac{ \sqrt{n} }{ \sqrt{m} } \max\{4\| u_{\nn}(0) - u^* \|_2/(\Lambda_0+\lambda), \epsilon_{\train} \cdot T\} + (\frac{ \sqrt{n} }{ \sqrt{m} } \| Y-u^* \|_2\\
    % & ~ + \lambda (\sqrt{d} + \sqrt{\log(m/\delta)} + 1) )\cdot T.
\end{align*}

\end{lemma}
% \begin{remark}
% By lemma~\ref{lem:chi_square_tail}, with probability at least $1-\delta$,
% \begin{align*}
% 	\| w_r(0) \|_2 \leq \sqrt{d} + \sqrt{\log(m/\delta)}
% \end{align*}
% holds for all $r\in[m]$.
% \end{remark}
\begin{proof}
Recall the gradient flow as Eq.~\eqref{eq:323_1}
\begin{align}\label{eq:332_2}
    \frac{ \d w_r( \tau ) }{ \d \tau } = ~ \sum_{i=1}^n \frac{1}{\sqrt{m}} a_r ( y_i - u_{\nn}(\tau)_i ) x_i \sigma'( w_r(\tau)^\top x_i ) - \lambda w_r(\tau)
\end{align}
So we have
\begin{align}\label{eq:332_1}
	\Big\| \frac{ \d w_r( \tau ) }{ \d \tau } \Big\|_2 
	= & ~ \left\| \sum_{i=1}^n \frac{1}{\sqrt{m}} a_r ( y_i - u_{\nn}(\tau)_i ) x_i \sigma'( w_r(\tau)^\top x_i ) - \lambda w_r(\tau) \right\|_2 \notag \\
	\leq & ~ \frac{1}{\sqrt{m}} \sum_{i=1}^n |y_i-u_{\nn}(\tau)_i| + \lambda \| w_r(\tau) \|_2 \notag\\
	\leq & ~ \frac{ \sqrt{n} }{ \sqrt{m} } \| Y-u_{\nn}(\tau) \|_2 + \lambda \| w_r (\tau) \|_2 \notag \\
	\leq & ~ \frac{ \sqrt{n} }{ \sqrt{m} } \| Y-u_{\nn}(\tau) \|_2 + \lambda (\| w_r(0) \|_2 + \| w_r(\tau) - w_r(0) \|_2) \notag\\ 
	\leq & ~ \frac{ \sqrt{n} }{ \sqrt{m} } \| Y-u_{\nn}(\tau) \|_2 + \lambda (\sqrt{d} + \sqrt{\log(m/\delta)} + 1) \notag\\ 
	\leq & ~ \frac{ \sqrt{n} }{ \sqrt{m} } \| Y-u_{\nn}(\tau) \|_2 + \lambda (2 \sqrt{d} +2 \sqrt{\log(m/\delta)} ) \notag\\
	\leq & ~ \frac{ \sqrt{n} }{ \sqrt{m} } (\| Y-u^*\|_2 + \| u_{\nn}(\tau) - u^*\|_2) + \lambda (2\sqrt{d} + 2\sqrt{\log(m/\delta)} ) \notag\\
	= & ~ \frac{ \sqrt{n} }{ \sqrt{m} } \| u_{\nn}(\tau) - u^*\|_2 \notag \\
	& ~ + \frac{ \sqrt{n} }{ \sqrt{m} } \| Y-u^* \|_2 + \lambda ( 2 \sqrt{d} + 2 \sqrt{\log(m/\delta)} ) \notag \\
	\leq & ~ \frac{ \sqrt{n} }{ \sqrt{m} } \max\{e^{-(\kappa^2 \Lambda_0+\lambda)\tau/4} \| u_{\nn}(0) - u^* \|_2, \epsilon_{\train}\} \notag\\
	& ~ + \frac{ \sqrt{n} }{ \sqrt{m} } \| Y-u^* \|_2 + \lambda ( 2 \sqrt{d} + 2 \sqrt{\log(m/\delta)} ),
\end{align}
where the first step follows from Eq.~\eqref{eq:332_2}, the second step follows from triangle inequality, the third step follows from Cauchy-schwarz inequality, the forth step follows from triangle inequality, the fifth step follows from condition $\| w_r(0) - w_r(\tau) \|_2 \leq 1,~\| w_r(0) \|_2 \leq \sqrt{d} + \sqrt{\log(m/\delta)}$, the seventh step follows from triangle inequality, the last step follows from $\| u_{\nn}(\tau) - u^* \|_2^2 \leq \max\{\exp(-(\kappa^2\Lambda_0 + \lambda) \tau/2) \cdot \| u_{\nn}(0) - u^* \|_2^2,~\epsilon_{\train}^2\}$.

Thus, for any $t \le T$,
\begin{align*}
	\| w_r(0) - w_r(t) \|_2 \leq & ~ \int_0^t \Big\| \frac{ \d w_r( \tau ) }{ \d \tau } \Big\|_2 d\tau \\
	\leq & ~ \frac{ \sqrt{n} }{ \sqrt{m} } \max\{4\| u_{\nn}(0) - u^* \|_2/(\kappa^2\Lambda_0+\lambda), \epsilon_{\train}\cdot T \} \\
	& ~ + \Big( \frac{ \sqrt{n} }{ \sqrt{m} } \| Y-u^* \|_2 + \lambda ( 2 \sqrt{d} + 2 \sqrt{\log(m/\delta)} ) \Big) \cdot T\\
    = & ~ \epsilon_W
\end{align*}
where the first step follows triangle inequality, the second step follows Eq.~\eqref{eq:332_1}, and the last step follows the definition of $\epsilon_W$ as Eq.~\eqref{eq:def_epsilon_W}.
\end{proof}

\begin{lemma}[Conclusion 2]\label{lem:hypothesis_2}
% Fix $\epsilon_W\in(0,1)$ independent of $t$. 
If $\forall r \in [m]$,
\begin{align*}
    \| w_r(0) - w_r(t) \|_2 \leq \epsilon_W < 1,
\end{align*}
then
\begin{align*}
    \| H(0) - H(t) \|_F \leq 2n \epsilon_W
\end{align*}
holds with probability $1-n^2 \cdot \exp{(-m\epsilon_W/10)}$.
\end{lemma}
\begin{proof}
Directly applying Lemma~\ref{lem:lemma_4.2_in_sy19}, we finish the proof.
%Lemma 4.2 in \cite{sy19}, we finish the proof.
\end{proof}

% \subsubsection{Induction hypothesis 3}

\begin{lemma}[Conclusion 3]\label{lem:hypothesis_3}
Fix $\epsilon_H'>0$ independent of $t$. If for all $\tau < t$
\begin{align*}
    \| H(0) - H^{\cts} \| \leq 4n ( \log{(n/\delta)} / m )^{1/2} \leq \Lambda_0/4
\end{align*}
and
\begin{align*}
    \| H(0) - H(\tau) \| \leq \epsilon_H' \leq \Lambda_0/4
\end{align*}
and
\begin{align}\label{eq:335_2}
	4 n ( \log ( n / \delta ) / m )^{1/2} \leq \frac{\epsilon_{\train}}{8\kappa^2\|Y-u^*\|_2}(\kappa^2\Lambda_0+\lambda)
\end{align}
and 
\begin{align}\label{eq:335_3}
    \epsilon_H' \leq \frac{\epsilon_{\train}}{8\kappa^2\|Y-u^*\|_2}(\kappa^2\Lambda_0+\lambda)
\end{align}
% and 
% \begin{align*}
%     \| u_{\nn}(t) - u^* \|_2^2 \leq \max\{\exp(-(\Lambda_0 + \lambda) (t) ) \cdot \| u_{\nn}(0) - u^* \|_2^2, ~ \epsilon^2\}~\text{holds for all}~\tau < t,
% \end{align*}
% and 
% \begin{align*}
%     \| u_{\nn}(t) - u^* \|_2^2 \leq \exp(-(\Lambda_0 + \lambda) (t) ) \cdot \| u_{\nn}(0) - u^* \|_2^2~\text{holds for all}~\tau < t,
% \end{align*}
then we have
\begin{align*}
    \| u_{\nn}(t) - u^* \|_2^2 \leq \max\{\exp(-(\kappa^2\Lambda_0 + \lambda)t/2 ) \cdot \| u_{\nn}(0) - u^* \|_2^2, ~ \epsilon_{\train}^2\}.
\end{align*}
\end{lemma}

\begin{proof}
By triangle inequality we have
\begin{align}\label{eq:335_1}
    \| H(\tau) - H^{\cts} \| 
    \leq & ~ \| H(0) - H(\tau) \| + \| H(0) - H^{\cts} \| \notag \\
    \leq & ~ \epsilon_H' + 4n( \log{(n/\delta)} / m)^{1/2} \notag\\
    \leq & ~ \Lambda_0/2
\end{align}
holds for all $\tau < t$. Denote $ \epsilon_H = \epsilon_H' + 4n( \log{(n/\delta)} / m )^{1/2}$, we have $\| H(\tau) - H^{\cts} \| \leq\epsilon_H \leq \Lambda_0/2$, which satisfies the condition of Lemma~\ref{lem:linear_converge_nn}. Thus, for any $\tau < t$, we have
\begin{align}\label{eq:induction_linear_convergence}
    \frac{ \d \| u_{\nn}(\tau) - u^* \|_2^2 }{ \d \tau } 
    \leq & ~ - ( \kappa^2 \Lambda_0 + \lambda ) \cdot \| u_{\nn} (\tau) - u^* \|_2^2 + 2 \kappa^2 \| H(\tau) - H^{\cts} \| \cdot \| u_{\nn} (\tau) - u^* \|_2 \cdot \| Y - u^* \|_2 \notag\\
    \leq & ~ - ( \kappa^2 \Lambda_0 + \lambda ) \cdot \| u_{\nn} (\tau) - u^* \|_2^2 + 2 \kappa^2 \epsilon_H \cdot \| u_{\nn} (\tau) - u^* \|_2 \cdot \|Y-u^*\|_2 
    % \\
    % \leq & ~ - ( \Lambda_0 + \lambda ) \cdot \| u_{\nn} (t) - u^* \|_2^2 ,
\end{align}
where the first step follows from Lemma~\ref{lem:linear_converge_nn}, the second step follows from Eq.~\eqref{eq:335_1}. Now let us discuss two cases:

{\bf Case 1.} If for all $\tau < t$, $\|u_{\nn}(\tau) - u^*\|_2 \geq \epsilon_{\train}$ always holds, we want to argue that 
	\begin{align*}
		\| u_{\nn}(t) - u^* \|_2^2 \leq \exp(-(\kappa^2 \Lambda_0 + \lambda) t/2) \cdot \| u_{\nn}(0) - u^* \|_2.
	\end{align*}
	Note by assumption~\eqref{eq:335_2} and~\eqref{eq:335_3}, we have 
	\begin{align*}
		\epsilon_H \leq \frac{\epsilon_{\train}}{4 \kappa^2 \|Y-u^*\|_2}(\kappa^2 \Lambda_0+\lambda)
	\end{align*}
	implies
	\begin{align*}
	2 \kappa^2 \epsilon_H \cdot \|Y-u^*\|_2 \leq (\kappa^2 \Lambda_0 + \lambda)/2\cdot \|u_{\nn}(\tau)-u^*\|_2
	\end{align*}
	holds for any $\tau < t$. Thus, plugging into~\eqref{eq:induction_linear_convergence},
	\begin{align*}
	    \frac{ \d \| u_{\nn}(\tau) - u^* \|_2^2 }{ \d \tau } 
	    \leq  ~ - ( \kappa^2 \Lambda_0 + \lambda )/2 \cdot \| u_{\nn} (\tau) - u^* \|_2^2,
	    % \\
	    % \leq & ~ - ( \Lambda_0 + \lambda ) \cdot \| u_{\nn} (t) - u^* \|_2^2 ,
	\end{align*}
	holds for all $\tau < t$, which implies
	\begin{align*}
		\| u_{\nn}(t) - u^* \|_2^2 \leq \exp{(-(\kappa^2 \Lambda_0 + \lambda)t/2)} \cdot \| u_{\nn}(0) - u^* \|_2^2.
	\end{align*}

{\bf Case 2.} If there exist $\bar{\tau} < t$, such that $\|u_{\nn}(\bar{\tau}) - u^*\|_2 < \epsilon_{\train}$, we want to argue that $\|u_{\nn}(t) - u^*\|_2 < \epsilon_{\train}$.
	Note by assumption~\eqref{eq:335_2} and~\eqref{eq:335_3}, we have 
	\begin{align*}
		\epsilon_H \leq \frac{\epsilon_{\train}}{4\kappa^2\|Y-u^*\|_2}(\kappa^2\Lambda_0+\lambda)
	\end{align*}
	implies
	\begin{align*}
	2\kappa^2 \epsilon_H \cdot \|u_{\nn}(\bar{\tau}) - u^*\|_2 \cdot \|Y-u^*\|_2 \leq (\kappa^2 \Lambda_0 + \lambda) \cdot \epsilon_{\train}^2.
	\end{align*}
	Thus, plugging into~\eqref{eq:induction_linear_convergence},
	\begin{align*}
	    \frac{ \d (\| u_{\nn}(\tau) - u^* \|_2^2 -\epsilon_{\train}^2)}{ \d \tau } 
	    \leq  ~ - ( \kappa^2 \Lambda_0 + \lambda ) \cdot (\| u_{\nn} (\tau) - u^* \|_2^2 - \epsilon_{\train}^2)
	    % \\
	    % \leq & ~ - ( \Lambda_0 + \lambda ) \cdot \| u_{\nn} (t) - u^* \|_2^2 ,
	\end{align*}
	holds for $\tau = \bar{\tau}$, which implies $ e^{( \kappa^2\Lambda_0 + \lambda )\tau} (\| u_{\nn}(\tau) - u^* \|_2^2 -\epsilon_{\train}^2) $ is non-increasing at $\tau = \bar{\tau}$. Since $\| u_{\nn}(\bar{\tau}) - u^* \|_2^2 - \epsilon_{\train}^2 < 0$, by induction, $ e^{( \kappa^2\Lambda_0 + \lambda )\tau} (\| u_{\nn}(\tau) - u^* \|_2^2 -\epsilon_{\train}^2) $ being non-increasing and $\| u_{\nn}(\tau) - u^* \|_2^2 - \epsilon_{\train}^2 < 0$ holds for all $\bar{\tau} \leq \tau \leq t$, which implies
	\begin{align*}
	\|u_{\nn}(t) - u^*\|_2 < \epsilon_{\train}.
	\end{align*}

Combine above two cases, we conclude
\begin{align*}
\| u_{\nn}(t) - u^* \|_2^2 \leq \max\{\exp(-(\kappa^2\Lambda_0 + \lambda)t/2 ) \cdot \| u_{\nn}(0) - u^* \|_2^2, ~ \epsilon_{\train}^2\}.
\end{align*}
\end{proof}

% \subsubsection{Induction hypothesis 4}

\begin{lemma}[Conclusion 4]\label{lem:hypothesis_4}
Fix $\epsilon_W\in(0,1)$ independent of $t$. If $\forall r \in [m]$, we have
\begin{align*}
    \| w_r(t) - w_r(0) \|_2 \leq \epsilon_W
\end{align*}
then
\begin{align*}
    \| \k_t( x_{\test} , X ) - \k_{0} ( x_{\test} , X ) \|_2 \leq \epsilon_K' = 2\sqrt{n}\epsilon_W
\end{align*}
holds with probability at least $1- n\cdot\exp{(-m\epsilon_W/10)}$.
\end{lemma}
\begin{proof}
% The proof can be splitted into two parts: Fristly we bound
% \begin{align*}
%     \| \k_0 ( x_{\test}, X ) - \k_{\ntk} ( x_{\test}, X ) \|_2.
% \end{align*}
% Note by definition,
% \begin{align*}
%     \E[\k_0 ( x_{\test}, x_i )] =  \k_{\ntk} ( x_{\test}, x_i )
% \end{align*}
% holds for any training data $x_i$. By Hoeffding inequality, we have for any $t>0$,
% \begin{align*}
%     \Pr[|\k_0 ( x_{\test}, x_i ) - \k_{\ntk} ( x_{\test}, x_i )|\ge t] \le 2\exp{(-mt^2/2)}.
% \end{align*}
% Setting $t=(\frac{2}{m}\log{(2n/\delta)})^{1/2}$, we can apply union bound on all training data $x_i$ to get with probability at least $1-\delta$, for all $i\in[n]$,
% \begin{align*}
%     |\k_0 ( x_{\test}, x_i ) - \k_{\ntk} ( x_{\test}, x_i )| \le (\frac{2}{m}\log{(2n/\delta)})^{1/2}.
% \end{align*}
% Thus, we have
% \begin{align}
%     \| \k_0 ( x_{\test}, X ) - \k_{\ntk} ( x_{\test}, X ) \|_2 \le (\frac{2n}{m}\log{(2n/\delta)})^{1/2}
% \end{align}
% holds with probability at least $1-\delta$.\\
% Secondly we bound
% \begin{align*}
%     \| \k_0 ( x_{\test}, X ) - \k_t ( x_{\test}, X ) \|_2.
% \end{align*}
Recall the definition of $\k_0$ and $\k_t$
\begin{align*}
    \k_0 (x_{\test},x_i) = & ~ \frac{1}{m} \sum_{r=1}^m x_{\test}^\top x_i \sigma'( x_{\test}^\top w_r(0) ) \sigma'( x_i^\top w_r(0) ) \\
    \k_t (x_{\test},x_i) = & ~ \frac{1}{m} \sum_{r=1}^m x_{\test}^\top x_i \sigma'( x_{\test}^\top w_r(t) ) \sigma'( x_i^\top w_r(t) )
\end{align*}
By direct calculation we have 
\begin{align*}
    \| \k_0 ( x_{\test}, X ) - \k_t ( x_{\test}, X ) \|_2^2\le \sum_{i=1}^n \Big( \frac{1}{m}\sum_{r=1}^m s_{r,i} \Big)^2,
\end{align*}
where 
\begin{align*}
s_{r,i} = {\bf 1}[w_r(0)^\top x_{\test} \ge 0, w_r(0)^\top x_i \ge 0] - {\bf 1}[w_r(t)^\top x_{\test} \ge 0, w_r(t)^\top x_i \ge 0], \forall r \in [m], i \in [n].
\end{align*}
Fix $i\in [n]$, by Bernstein inequality (Lemma~\ref{lem:bernstein}), we have for any $t>0$, 
\begin{align*}
    \Pr \Big[ \frac{1}{m}\sum_{r=1}^m s_{r,i} \ge 2\epsilon_W \Big] \le \exp ( - m \epsilon_W / 10 ) .
\end{align*}
Thus, applying union bound over all training data $x_i,~i\in[n]$, we conclude
\begin{align*}
    \Pr[\| \k_0 ( x_{\test}, X ) - \k_t ( x_{\test}, X ) \|_2 \le 2\sqrt{n}\epsilon_W] \le 1-n\cdot\exp{(-m\epsilon_W/10)}.
\end{align*}
Note by definition $\epsilon_K' = 2\sqrt{n} \epsilon_W$, so we finish the proof.
% Combining two parts together, we have
% \begin{align*}
%     \| \k_t( x_{\test} , X ) - \k_{\ntk} ( x_{\test} , X ) \|_2 \le \Big( \frac{2n}{m}\log{(2n/\delta)} \Big)^{1/2} +  2\sqrt{n}\epsilon_W
% \end{align*}
% holds with probability at least $1-\delta - n\cdot\exp{(-m\epsilon_W/10)} $.
\end{proof}

Now we summarize all the conditions need to be satisfied so that the induction works as in Table~\ref{tab:condition_only}.

\begin{table}[htb]\caption{Summary of conditions for induction}\label{tab:condition_only}

\centering
% {\footnotesize
{
  \begin{tabular}{| l | l | l |}
    \hline
    {\bf No.} & {\bf Condition} & {\bf Place} \\ \hline
    1 & $\epsilon_W \leq 1$ & Lem.~\ref{lem:hypothesis_1} \\ \hline
    2 & $ \epsilon_H' \leq \Lambda_0/4$ & Lem.~\ref{lem:hypothesis_3} \\ \hline 
    3 & $4n(\frac{\log{(n/\delta)}}{m})^{1/2} \leq \Lambda_0/4$ & Lem.~\ref{lem:hypothesis_3} \\ \hline
    4 & $4n(\frac{\log{(n/\delta)}}{m})^{1/2} \leq \frac{\epsilon_{\train}}{8\kappa^2\|Y-u^*\|_2}(\kappa^2\Lambda_0+\lambda)$ & Lem.~\ref{lem:hypothesis_3} \\ \hline
    5 & $\epsilon_H' \leq \frac{\epsilon_{\train}}{8\kappa^2\|Y-u^*\|_2}(\kappa^2\Lambda_0+\lambda)$ & Lem.~\ref{lem:hypothesis_3} \\
    \hline
  \end{tabular}
}
\end{table}

Note by choosing If $\kappa=\wt{O}(\frac{\epsilon\Lambda_0}{n}) ,~T=\wt{O}(\frac{1}{\kappa^2(\Lambda_0+\lambda)})$, $\epsilon_{\train} = \wt{O}(\|u_{\nn}(0)-u^*\|_2)$, $m \geq\wt{O}(\frac{n^{10}d}{\epsilon^6 \Lambda_0^{10}})$ and $\lambda=\wt{O}(\frac{1}{\sqrt{m}})$, we have
\begin{align}\label{eq:y-u*}
	\|Y-u^*\|_2 = & ~ \|(\kappa^2 \lambda(H^{\cts}+\lambda I_n)^{-1}Y\|_2 \notag\\
	\leq & ~ \lambda\|(\kappa^2 H^{\cts}+\lambda I_n)^{-1}\|\|Y\|_2\notag\\
	\leq & ~ \wt{O}(\frac{\lambda\sqrt{n}}{\kappa^2\Lambda_0}) \notag\\
	\leq & ~ \wt{O}(\frac{1}{n^{2.5}})
\end{align}
where the first step follows from the definition of $u^*$, the second step follows from Cauchy-Schwartz inequality, the third step follows from $Y=O(\sqrt{n})$, and the last step follows from the choice of the parameters.

Further, with probability $1-\delta$, we have
\begin{align*}
	\|u_{\nn}(0) - u^*\|_2 \leq & ~ \|u_{\nn}(0)\|_2 + \|Y-u^*\|_2 + \|Y\|_2 \\
	\leq & ~ \wt{O}(\frac{\epsilon\Lambda_0}{\sqrt{n}}) + \wt{O}(\frac{1}{n^{2.5}}) + \wt{O}(\sqrt{n}) \\
	\leq & ~ \wt{O}(\sqrt{n})
\end{align*}
where the first step follows from triangle inequality, the second step follows from Lemma~\ref{lem:epsilon_init}, Eq.~\eqref{eq:y-u*} and $Y=O(\sqrt{n})$, and the last step follows $\epsilon,~\Lambda_0<1$. With same reason,
\begin{align*}
	\|u_{\nn}(0) - u^*\|_2 \geq & ~ -\|u_{\nn}(0)\|_2 - \|Y-u^*\|_2 + \|Y\|_2 \\
	\geq & ~ -\wt{O}(\frac{\epsilon\Lambda_0}{\sqrt{n}}) - \wt{O}(\frac{1}{n^{2.5}}) + \wt{O}(\sqrt{n}) \\
	\geq & ~ \wt{O}(\sqrt{n}).
\end{align*}
Thus, we have $\|u_{\nn}(0) - u^*\|_2=\wt{O}(\sqrt{n})$. 

Now, by direct calculation, we have all the induction conditions satisfied with high probability. Note the failure probability only comes from Lemma~\ref{lem:epsilon_init},~\ref{lem:random_init},~\ref{lem:hypothesis_2},~\ref{lem:hypothesis_4}, which only depend on the initialization. By union bound over these failure events, we have all four conclusions in Lemma~\ref{lem:induction} holds with high probability, which completes the proof. 
% And we summarize all the crucial quantities in Table~\ref{tab:quantities} when we pick 
%  \begin{align*}
%  m = \wt{O} \Big( \max \Big\{\frac{n^4}{\kappa^6(\Lambda_0+\lambda)^4}, \frac{n^2 d}{\kappa^6(\Lambda_0+\lambda)^4} \Big\} \Big).
%  \end{align*}

%\newpage
% \begin{remark}
% Next step we hope to use induction to bound $\epsilon_{\test}$ and $\epsilon_H$ appropriately, i.e., show the over-parametrization effect in the case with regularization.

% We list all the bounds and constraints of the induction: For all $t \le T$\\
% 1. $\| w_r(0) - w_r(t) \|_2 \leq \epsilon_W = \frac{ 2\sqrt{n} }{ (\Lambda_0+\lambda)\sqrt{m} } \| u_{\nn}(0) - u^* \|_2 + (\frac{ \sqrt{n} }{ \sqrt{m} } \| Y-u^* \|_2 + \lambda )\cdot T <c$\\
% 2. $ \| w_r(t) \|_2 \leq 1 $\\
% 3. $\epsilon_H = 2n\cdot \epsilon_W $\\
% 4. $\epsilon_H \leq \Lambda_0$\\
% 5. $\Big( \epsilon_H + 4n \Big( \frac{\log{(n/\delta)}}{m} \Big)^{1/2} \Big) \le \frac{\|u_{\nn}(t)-u^*\|_2}{\|Y-u^*\|_2}(\Lambda_0+\lambda)$\\
% % 5. $\Big( \epsilon_H + 4n \Big( \frac{\log{(n/\delta)}}{m} \Big)^{1/2} \Big) \cdot e^{(\Lambda_0+\lambda)T/2}\le \frac{\|u_{\nn}(0)-u^*\|_2}{\|Y-u^*\|_2}(\Lambda_0+\lambda)$\\

% % Thus, assuming if $\lambda\sim \Lambda_{\max}$, by picking $T=O(\frac{1}{\Lambda_0+\lambda})$ and $m\ge\frac{n^5}{(\lambda+\Lambda_0)^2}$, we have the all the constraints satisfied.
% \end{remark}

% \newpage
\subsubsection{Final result for upper bounding $| u_{\nn,\test}(T) - u_{\ntk,\test}(T) |$}\label{sec:equiv_bound_nn_test_T_and_ntk_test_T}
In this section, we prove Lemma~\ref{lem:equivalence_at_T}.
\begin{lemma}[Upper bounding test error]\label{lem:equivalence_at_T}
Given training data matrix $X \in \R^{n \times d}$ and corresponding label vector $Y \in \R^n$. Fix the total number of iterations $T > 0$. Given arbitrary test data $x_{\test} \in \R^d$. Let $u_{\nn,\test}(t) \in \R^n$ and $u_{\ntk,\test}(t) \in \R^n$ be the test data predictors defined in Definition~\ref{def:nn} and Definition~\ref{def:krr_ntk} respectively. Let $\kappa\in(0,1)$ be the corresponding multiplier. Given accuracy $\epsilon>0$, if $\kappa = \wt{O}(\frac{\epsilon\Lambda_0}{n})$, $T=\wt{O}(\frac{1}{\kappa^2\Lambda_0})$, $m \geq \wt{O}(\frac{n^{10}d}{\epsilon^6\Lambda_0^{10}})$ and $\lambda \leq \wt{O}(\frac{1}{\sqrt{m}})$. Then for any $x_{\test}\in\R^d$, with probability at least $1-\delta$ over the random initialization, we have
\begin{align*}
\| u_{\nn,\test}(T) - u_{\ntk,\test}(T) \|_2 \leq \epsilon/2,
\end{align*}
where $\wt{O}(\cdot)$ hides $\poly\log( n / (\epsilon\delta\Lambda_0) )$.% the poly-log factors in $\frac{1}{\epsilon},~\frac{1}{\delta}$, $n$, $\Lambda_0$ and $\lambda$.
\end{lemma}
\begin{proof}
By Lemma~\ref{lem:more_concreate_bound}, we have
\begin{align}\label{eq:b44_1}
	|u_{\nn,\test}(T)-u_{\ntk,\test}(T)| \leq & ~ (1+\kappa^2 nT)\epsilon_{\init} + \kappa^2 \epsilon_K \cdot \Big( \frac{ \| u^* \|_2 }{ \kappa^2 \Lambda_0 + \lambda } + \|u^*-Y\|_2 T\Big) \notag \\
    & ~ + \sqrt{n}T^2\kappa^4 \epsilon_H ( \| u^* \|_2 + \| u^* - Y \|_2 )
\end{align}
By Lemma~\ref{lem:epsilon_init}, we can choose $\epsilon_{\init} = \epsilon(\Lambda_0)/n$. 

Further, note 
\begin{align*}
	\|\k_{\ntk}(x_{\test},X)-\k_{t}(x_{\test},X)\|_2 \leq & ~ \|\k_{\ntk}(x_{\test},X)-\k_0(x_{\test},X)\|_2 + \|\k_0(x_{\test},X)-\k_{t}(x_{\test},X)\|_2 \\
	\leq & ~  ( 2n \log{(2n/\delta)} / m )^{1/2} + \|\k_0(x_{\test},X)-\k_{t}(x_{\test},X)\|_2 \\
	\leq & ~  \wt{O}(\frac{\epsilon \Lambda_0^2}{n^{1.5}})
\end{align*}
where the first step follows from triangle inequality, the second step follows from Lemma~\ref{lem:random_init}, and the last step follows from Lemma~\ref{lem:induction}. Thus, we can choose $\epsilon_K = \frac{\epsilon \Lambda_0^2}{n^{1.5}}$.

Also,
\begin{align*}
	\|H^{\cts}-H(t)\| \leq & ~ \|H^{\cts}-H(0)\| + \|H(0)-H(t)\|_2 \\
	\leq & ~  4n ( \log(n/\delta) / m )^{1/2} + \|H(0)-H(t)\|_2 \\
	\leq & ~  \wt{O}(\frac{\epsilon \Lambda_0^2}{n})
\end{align*}
where the first step follows from triangle inequality, the second step follows from Lemma~\ref{lem:random_init}, and the last step follows from Lemma~\ref{lem:induction}. Thus, we can choose $\epsilon_H = \frac{\epsilon \Lambda_0^2}{n}$.

Note $\|u^*\|_2 \leq \sqrt{n}$ and $\|u^*-Y\|\leq\sqrt{n}$, plugging the value of $\epsilon_{\init},~\epsilon_K,~\epsilon_H$ into Eq.~\eqref{eq:b44_1}, we have 
\begin{align*}%\label{eq:b44_1}
	|u_{\nn,\test}(T)-u_{\ntk,\test}(T)| \leq \epsilon/2.
\end{align*}
\end{proof}

% For the term $ \|u_{\ntk}(T) - u^* \|_2$, note by the linear convergence, i.e.
% \begin{align*}
% \frac{ \d \|u_{\ntk}(T) - u^* \|_2^2 }{ \d t } \leq -(\Lambda_0 + \lambda) \|u_{\ntk}(T) - u^* \|_2^2,
% \end{align*}
% we have 
% \begin{align*}
% \|u_{\ntk}(T) - u^* \|_2 \leq e^{-(\Lambda_0 + \lambda) T/2 } \| u_{\ntk}(0) - u^* \|_2.
% \end{align*}
% Thus, denoting our objective accuracy as $\epsilon$, we have
% \begin{align*}
% \|u_{\ntk}(T) - u^* \|_2 \leq \epsilon
% \end{align*}
% if $T \geq \frac{2}{\Lambda_0 + \lambda} \log ( \|u^*\|_2 / \epsilon )$.

\subsubsection{Main result for test data prediction equivalence}\label{sec:equiv_main_test_equivalence}
In this section, we restate and prove Theorem~\ref{thm:main_test_equivalence_intro}.

\begin{theorem}[Equivalence between training net with regularization and kernel ridge regression for test data prediction, restatement of Theorem~\ref{thm:main_test_equivalence_intro}]\label{thm:main_test_equivalence}
Given training data matrix $X \in \R^{n \times d}$ and corresponding label vector $Y \in \R^n$. Let $T > 0$ be the total number of iterations. Given arbitrary test data $x_{\test} \in \R^d$. Let $u_{\nn,\test}(t) \in \R^n$ and $u_{\test}^* \in \R^n$ be the test data predictors defined in Definition~\ref{def:nn} and Definition~\ref{def:krr_ntk} respectively. %Let $\kappa\in(0,1)$ be the corresponding multiplier. 

For any accuracy $\epsilon \in (0,1/10)$ and failure probability $\delta \in (0,1/10)$, if $\kappa = \wt{O}(\frac{\epsilon\Lambda_0}{n})$, $T=\wt{O}(\frac{1}{\kappa^2\Lambda_0})$, $m \geq \wt{O}(\frac{n^{10}d}{\epsilon^6\Lambda_0^{10}})$ and $\lambda \leq \wt{O}(\frac{1}{\sqrt{m}})$. Then for any $x_{\test}\in\R^d$, with probability at least $1-\delta$ over the random initialization, we have
% Suppose $\sigma(z)=\max\{0,z\}$, initializing the neural network predictor
% symmetrically, and picking network width $m \geq \wt{O}(\frac{n^5 d}{(\Lambda_0+\lambda)^6 \epsilon^2})$ and training time $T=\wt{O}(\frac{1}{\Lambda_0+\lambda})$. Then for any $x_{\test}\in\R^d$ with $\|x_{\test}\|_2 = 1$, with probabiliy at least $1-\delta$ over the random initialization, we have
% Initialize the weight of two-layer neural network with Relu activation function $w_r(0) \sim N(0,1)$ i.i.d. for all $r\in[m]$ with the symmetric trick. Assume $\lambda_{\min}(H^{\cts}) = 2\Lambda_0 >0$. Then by picking $m = \wt{O}\Big(\frac{n^5}{(\Lambda_0+\lambda)^6 \epsilon_{\obj}^2}\Big)$ and $T = \wt{O} \Big(\frac{1}{\Lambda_0+\lambda}\Big)$, we have
\begin{align*}
\| u_{\nn,\test}(T) - u_{\test}^* \|_2 \leq \epsilon.
\end{align*}
Here $\wt{O}(\cdot)$ hides $\poly\log(n/(\epsilon \delta \Lambda_0 ))$.% the poly-log factors in $\frac{1}{\epsilon}$, $\frac{1}{\delta}$, $n$, $\Lambda_0$ and $\lambda$.
\end{theorem}
\begin{proof}
It follows from combining results of bounding $\| u_{\nn,\test}(T) - u_{\ntk,\test}(T) \|_2 \leq \epsilon/2$ as shown in Lemma~\ref{lem:equivalence_at_T} and $\| u_{\ntk,\test}(T) - u_{\test}^* \|_2 \leq \epsilon/2$ as shown in Lemma~\ref{lem:u_ntk_test_T_minus_u_test_*} using triangle inequality.
\end{proof}

\subsection{Equivalence between training net with regularization and kernel ridge regression for training data prediction}\label{sec:train}
In this section, we restate and proof Theorem~\ref{thm:main_train_equivalence_intro}.

Note the proof of equivalence results for the test data in previous sections automatically gives us an equivalence results of the prediction for training data. Specifically, the third conclusion in Lemma~\ref{lem:induction} characterizes the training prediction $u_{\nn}(t)$ throughout the training process. Thus, we have the following theorem characterize the equivalence between training net with regularization and kernel ridge regression for the training data.
\begin{theorem}[Equivalence between training net with regularization and kernel ridge regression for training data prediction, restatement of Theorem~\ref{thm:main_train_equivalence_intro}]\label{thm:main_train_equivalence}
Given training data matrix $X \in \R^{n \times d}$ and corresponding label vector $Y \in \R^n$. Let $T > 0$ be the total number of iterations. Let $u_{\nn}(t) \in \R^n$ and $u^* \in \R^n$ be the training data predictors defined in Definition~\ref{def:nn} and Definition~\ref{def:krr_ntk} respectively. Let $\kappa=1$ be the corresponding multiplier. 

Given any accuracy $\epsilon \in ( 0 , 1/10 )$ and failure probability $\delta \in (0,1/10)$, if $\kappa = 1$, $T=\wt{O}(\frac{1}{\Lambda_0})$, network width $m \geq \wt{O}(\frac{n^4d}{\lambda_0^4\epsilon})$ and regularization parameter $\lambda \leq \wt{O}(\frac{1}{\sqrt{m}})$, then with probability at least $1-\delta$ over the random initialization, we have
\begin{align*}
	\|u_{\nn}(T) - u^*\|_2 \leq \epsilon.
\end{align*}
Here $\wt{O}(\cdot)$ hides $\poly\log(n/(\epsilon \delta \Lambda_0 ))$.
\end{theorem}
\begin{proof}
Let $\kappa=1$, $T=\wt{O}(\frac{1}{\Lambda_0})$, $\lambda \leq \wt{O}(\frac{1}{\sqrt{m}})$, $m \geq \wt{O}(\frac{n^4d}{\lambda_0^4\epsilon})$ and $\epsilon_{\train} = \epsilon$ in Lemma~\ref{lem:induction}. We can see all the conditions in Table~\ref{tab:condition_only} hold. Thus, the third conclusion in Lemma~\ref{lem:induction} holds. So with probability $1-\delta$, we have
\begin{align*}
	\|u_{\nn}(T) - u^*\|_2^2 \leq & ~ \max\{\exp(-(\kappa^2\Lambda_0 + \lambda) t/2) \cdot \| u_{\nn}(0) - u^* \|_2^2, ~ \epsilon_{\train}^2\}\\
	\leq & ~ \max\{\exp(-(\Lambda_0 + \lambda) T/2) \cdot \| u_{\nn}(0) - u^* \|_2^2, ~ \epsilon^2\}\\
	\leq & ~ \epsilon^2
\end{align*}
where the first step follows from Lemma~\ref{lem:induction}, the second step follows from $\kappa =1$ and $\epsilon_{\train} = \epsilon$, the last step follows from $T=\wt{O}(\frac{1}{\Lambda_0})$ and $\| u_{\nn}(0) - u^* \|_2^2 \leq n$ with high probability.
\end{proof}

\begin{table}[!t]\caption{Summary of parameters of main results in Section~\ref{sec:equiv}}\label{tab:xxx}

\centering
{
  \begin{tabular}{ | l | l | l | l | l | l |}
    \hline
    {\bf Statement} & $\kappa$ & $T$ & $m$ & $\lambda$ & {\bf Comment} \\ \hline
    Theorem~\ref{thm:main_test_equivalence} & $\epsilon \Lambda_0 / n$ & $1/(\kappa^2 \Lambda_0)$ & $\Lambda_0^{-10} \epsilon^{-6} n^{10} d$ & $1/\sqrt{m}$ & test \\ \hline
    Theorem~\ref{thm:main_train_equivalence} & $1$ & $1/\Lambda_0$ & $\Lambda_0^{-4} \epsilon^{-1} n^4 d $ & $1/\sqrt{m}$ & train \\ \hline
  \end{tabular}
}
\end{table}
\section{Generalization result of leverage score sampling for approximating kernels}\label{sec:ger_lev}
In this section, we generalize the result of Lemma 8 in \cite{akmmvz17} for a more broad class of kernels and feature vectors. Specifically, we prove Theorem~\ref{thm:leverage_score_intro}.

Section~\ref{sec:pre_lev} introduces the related kernel and random features, we also restate Definition~\ref{def:leverage_score_intro} and~\ref{def:modify_random_feature_lev} for leverage score sampling and random features in this section. Section~\ref{sec:result_lev} restates and proves our main result Theorem~\ref{thm:leverage_score_intro}.

\subsection{Preliminaries}\label{sec:pre_lev}

\begin{definition}[Kernel]\label{def:kernel}
Consider kernel function $k: \R^d \times \R^d \to \R$ which can be written as 
\begin{align*}
  \k(x,z) = \E_{w \sim p}[\phi(x,w)^\top\phi(z,w)],
\end{align*}
for any data $x,z\in\R^d$, where $\phi:\R^{d} \times \R^{d_1}\to \R^{d_2}$ denotes a finite dimensional vector and $p:\R^{d_1}\to\R_{\geq 0}$ denotes probability density function. Given data $x_1,\cdots,x_n\in\R^d$, we define the corresponding kernel matrix $K \in \R^{n\times n}$ as
\begin{align*}
  K_{i,j} = \k(x_i,x_j) = \E_{w\sim p}[\phi(x_i,w)^\top\phi(x_j,w)]
\end{align*}
\end{definition}

\begin{definition}[Random features]\label{def:random_feature}
Given $m$ weight vectors $w_1, \cdots, w_m \in \R^{d_1}$. Let $\varphi : \R^d \rightarrow \R^{md_2}$ be define as 
\begin{align*}
  \varphi(x) =  \Big[ \frac{1}{ \sqrt{m} }\phi(x,w_1)^\top , \cdots, \frac{1}{ \sqrt{m} }\phi(x,w_m)^\top \Big]^\top
\end{align*}
If $w_1, \cdots, w_m$ are drawn according to $p(\cdot)$, then
\begin{align*}
  \k(x,z) = \E_{p} [ \varphi(x)^\top \varphi(z) ]
\end{align*}
Given data matrix $X=[x_1,\cdots,x_n]^\top \in \R^{n \times d}$, define $\Phi : \R^{d_1} \to \R^{n\times d_2}$ as
\begin{align*}
  \Phi(w) = [\phi(x_1,w)^\top, \cdots, \phi(x_n,w)^\top]^\top
\end{align*}
If $w$ are drawn according to $p(\cdot)$, then
\begin{align}\label{eq:ls_1}
K = \E_{p} [ \Phi(w)\Phi(w)^\top ]
\end{align}
Further, define $\Psi\in\R^{n\times md_2}$ as 
\begin{align*}
  \Psi =  [\varphi(x_1),\cdots,\varphi(x_n)]^\top.
\end{align*}
Then we have
\begin{align*}
  \Psi\Psi^\top =  \frac{1}{m}\sum_{r=1}^m \Phi(w_r)\Phi(w_r)^\top
\end{align*}
If $w_1, \cdots, w_m$ are drawn according to $p(\cdot)$, then
\begin{align*}
K = \E_{p} [ \Psi\Psi^\top ]
\end{align*}
\end{definition}

\begin{definition}[Modified random features, restatement of Definition~\ref{def:modify_random_feature_lev}]\label{def:modify_random_feature}
Given any probability density function $q(\cdot)$ whose support includes that of $p(\cdot)$. Given $m$ weight vectors $w_1, \cdots, w_m \in \R^{d_1}$. Let $\bar{\varphi} : \R^d \rightarrow \R^{m d_2}$ be defined as 
\begin{align*}
  \bar{\varphi}(x) =  \frac{1}{\sqrt{m}} \Big[ \frac{\sqrt{p(w_1)}}{ \sqrt{q(w_1)} }\phi(x,w_1)^\top , \cdots, \frac{\sqrt{p(w_m)}}{ \sqrt{q(w_m)} }\phi(x,w_m)^\top \Big]^\top
\end{align*}
If $w_1, \cdots, w_m$ are drawn according to $q(\cdot)$, then
\begin{align*}
  \k(x,z) = \E_{q} [ \bar{\varphi}(x)^\top \bar{\varphi}(z) ]
\end{align*}
Given data matrix $X=[x_1,\cdots,x_n]^\top \in \R^{n \times d}$, define $\bar{\Phi}:\R^{d_1}\to\R^{n\times d_2}$ as
\begin{align*}
  \bar{\Phi}(w) = \frac{\sqrt{p(w)}}{ \sqrt{q(w)} }[\phi(x_1,w)^\top, \cdots, \phi(x_n,w)^\top]^\top =  \frac{\sqrt{p(w)}}{ \sqrt{q(w)} }\Phi(w)
\end{align*}
If $w$ are drawn according to $q(\cdot)$, then
\begin{align*}
K = \E_{q} [ \bar{\Phi}(w)\bar{\Phi}(w)^\top ]
\end{align*}
Further, define $\bar{\Psi}\in\R^{n\times md_2}$ as 
\begin{align*}
  \bar{\Psi} =  [\bar{\varphi}(x_1),\cdots,\bar{\varphi}(x_n)]^\top.
\end{align*}
then
\begin{align}\label{eq:ls_2}
  \bar{\Psi}\bar{\Psi}^\top =  \frac{1}{m}\sum_{r=1}^m \bar{\Phi}(w_r)\bar{\Phi}(w_r)^\top = \frac{1}{m}\sum_{r=1}^m \frac{p(w_r)}{q(w_r)}\Phi(w_r)\Phi(w_r)^\top
\end{align}
If $w_1, \cdots, w_m$ are drawn according to $q(\cdot)$, then
\begin{align*}
K = \E_{q} [ \bar{\Psi}\bar{\Psi}^\top ]
\end{align*}
\end{definition}

\begin{definition}[Leverage score, restatement of Definition~\ref{def:leverage_score_intro}]\label{def:leverage_score}
Let $p : \R^{d_1} \rightarrow \R_{\geq 0}$ denote the probability density function defined in Definition~\ref{def:kernel}. Let $\Phi: \R^{d_1} \rightarrow \R^{n\times d_2}$ be defined as Definition~\ref{def:random_feature}. For parameter $\lambda > 0$, we define the ridge leverage score as
\begin{align*}
  q_\lambda(w) = p(w) \Tr[\Phi(w)^\top ( K + \lambda I_n )^{-1} \Phi(w)].
\end{align*} 
\end{definition}

\begin{definition}[Statistical dimension]\label{def:stat_dim}
Given kernel matrix $K\in\R^{n\times n}$ and parameter $\lambda>0$, we define statistical dimension $s_{\lambda}(K)$ as:
\begin{align*}
  s_{\lambda}(K) = \Tr[(K+\lambda I_n)^{-1} K].
\end{align*}
\end{definition}

Note we have $\int_{\R^{d_2}} q_\lambda(w) \d w = s_{\lambda}(K)$. Thus we can define the leverage score sampling distribution as
\begin{definition}[Leverage score sampling distribution]\label{def:lev_distribution}
Let $q_\lambda(w)$ denote the leverage score defined in Definition~\ref{def:leverage_score}. Let $s_{\lambda}(K)$ denote the statistical dimension defined in Definition~\ref{def:stat_dim}. We define the leverage score sampling distribution as
\begin{align*}
  q(w) = \frac{q_{\lambda}(w)}{s_{\lambda}(K)}.
\end{align*}
\end{definition}

\subsection{Main result}\label{sec:result_lev}

\begin{theorem}[Restatement of Theorem~\ref{thm:leverage_score_intro}, generalization of Lemma 8 in \cite{akmmvz17}]\label{thm:leverage_score}
Given $n$ data points $x_1, x_2, \cdots, x_n \in \R^d$. Let $k:\R^d\times\R^d\to\R$, $K\in\R^{n\times n}$ be the kernel defined in Definition~\ref{def:kernel}, with corresponding vector $\phi:\R^{d} \times \R^{d_1}\to \R^{d_2}$ and probability density function $p:\R^{d_1}\to\R_{\geq 0}$. Given parameter $\lambda \in (0,\|K\|)$, let $q_\lambda:\R^{d_1}\to\R_{\geq 0}$ be the leverage score defined in Definition~\ref{def:leverage_score}. Let $\tilde{q}_{\lambda}:\R^{d_1}\rightarrow \R$ be any measurable function such that $\tilde{q}_{\lambda}(w) \geq q_\lambda(w)$ holds for all $w\in \mathbb{R}^{d_1}$. Assume $ s_{\tilde{q}_\lambda} = \int_{\mathbb{R}^{d_1}} \tilde{q}_{\lambda}(w)\d w$ is finite. Denote $\bar{q}_{\lambda}(w)=\tilde{q}_{\lambda}(w)/s_{\tilde{q}_\lambda}$. For any accuracy parameter $\epsilon \in (0,1/2)$ and failure probability $\delta \in ( 0 , 1)$. Let $w_1,\cdots,w_m \in \R^d$ denote $m$ samples draw independently from the distribution associated with the density $\bar{q}_{\lambda}(\cdot)$, and construct the matrix $\bar{\Psi} \in \R^{n \times md_2}$ according to Definition~\ref{def:modify_random_feature} with $q=\bar{q}_{\lambda}$. Let $s_\lambda(K)$ be defined as Definition~\ref{def:stat_dim}.
% \begin{align*}
% Z_{i,j} =\frac{1}{\sqrt{m}} z(w_j)_i \sqrt{p(w_j)/p_{\tilde{\tau}}(w_j)}, \forall i \in [n], j \in [m].
% \end{align*}
If $m \geq 3 \epsilon^{-2} s_{\tilde{q}_\lambda} \ln(16s_{\tilde{q}_\lambda}\cdot s_\lambda(K) / \delta)$, then we have 
\begin{align} \label{eq:leverage_score_thm}
 (1-\epsilon) \cdot (K + \lambda I_n) \preceq \bar{\Psi} \bar{\Psi}^\top + \lambda I_n \preceq (1+\epsilon) \cdot ( K + \lambda I_n )
\end{align}
 holds with probability at least $1-\delta$. 
\end{theorem}

To prove the theorem, we follow the same proof framework as Lemma 8 in \cite{akmmvz17}.

% Firstly, we restate the key concentration result used in the proof.

\begin{proof}
Let $K+\lambda I_n = V^\top \Sigma^2 V$ be an eigenvalue decomposition of $K+\lambda I_n$. Note that Eq.~\eqref{eq:leverage_score_thm} is equivalent to
\begin{align}\label{eq:57_1}
	K-\epsilon(K+\lambda I_n) \preceq \bar{\Psi}\bar{\Psi}^\top \preceq K+ \epsilon(K+\lambda I_n).
\end{align}
Multiplying $\Sigma^{-1}V$ on the left and $V^\top \Sigma^{-1}$ on the left for both sides of Eq.~\eqref{eq:57_1}, it suffices to show that 
\begin{align}\label{eq:57_2}
	\| \Sigma^{-1}V\bar{\Psi}\bar{\Psi}^\top V^\top \Sigma^{-1} - \Sigma^{-1}VKV^\top \Sigma^{-1}\| \leq \epsilon
\end{align}
holds with probability at least $1-\delta$. Let
\begin{align*}
 	Y_r = \frac{p(w_r)}{\bar{q}_{\lambda}(w_r)} \Sigma^{-1}V{\Phi}(w_r){\Phi}(w_r)^\top V^\top \Sigma^{-1}.
\end{align*}
We have 
\begin{align*}
	\E_{\bar{q}_{\lambda}}[Y_l] = & ~ \E_{\bar{q}_{\lambda}}[\frac{p(w_r)}{\bar{q}_{\lambda}(w_r)}\Sigma^{-1}V{\Phi}(w_r)\bar{\Phi}(w_r)^\top V^\top \Sigma^{-1}]\\
	= & ~ \Sigma^{-1}V\E_{\bar{q}_{\lambda}}[\frac{p(w_r)}{\bar{q}_{\lambda}(w_r)}{\Phi}(w_r){\Phi}(w_r)^\top] V^\top \Sigma^{-1}\\
	= & ~ \Sigma^{-1}V \E_{p}[{\Phi}(w_r){\Phi}(w_r)^\top]  V^\top \Sigma^{-1}\\
	= & ~ \Sigma^{-1}V K  V^\top \Sigma^{-1}
\end{align*}
where the first step follows from the definition of $Y_r$, the second step follows the linearity of expectation, the third step calculations the expectation, and the last step follows Eq.~\eqref{eq:ls_1}. 

Also we have
\begin{align*}
	\frac{1}{m}\sum_{r=1}^m Y_r = & ~ \frac{1}{m}\sum_{r=1}^m \frac{p(w_r)}{\bar{q}_{\lambda}(w_r)}\Sigma^{-1}V{\Phi}(w_r)\bar{\Phi}(w_r)^\top V^\top \Sigma^{-1}\\
	= & ~ \Sigma^{-1}V\Big(\frac{1}{m}\sum_{r=1}^m \frac{p(w_r)}{\bar{q}_{\lambda}(w_r)}{\Phi}(w_r){\Phi}(w_r)^\top\Big) V^\top \Sigma^{-1}\\
	= & ~ \Sigma^{-1}V \bar{\Psi}\bar{\Psi}^\top  V^\top \Sigma^{-1}
\end{align*}
where the first step follows from the definition of $Y_r$, the second step follows from basic linear algebra, and the last step follows from Eq.~\eqref{eq:ls_2}. 

Thus, it suffices to show that 
\begin{align}\label{eq:57_3}
	\| \frac{1}{m}\sum_{r=1}^m Y_r -\E_{\bar{q}_{\lambda}}[Y_l] \| \leq \epsilon
\end{align}
holds with probability at least $1-\delta$. 

We can apply matrix concentration Lemma~\ref{lem:con_lev} to prove Eq.~\eqref{eq:57_3}, which requires us to bound $\|Y_r\|$ and $\E[Y_l^2]$. Note
\begin{align*}
	\|Y_l\| \leq & ~ \frac{p(w_r)}{\bar{q}_{\lambda}(w_r)} \Tr[\Sigma^{-1}V{\Phi}(w_r){\Phi}(w_r)^\top V^\top \Sigma^{-1}]\\
	= & ~ \frac{p(w_r)}{\bar{q}_{\lambda}(w_r)} \Tr[{\Phi}(w_r)^\top V^\top \Sigma^{-1} \Sigma^{-1}V{\Phi}(w_r)]\\
	= & ~ \frac{p(w_r)}{\bar{q}_{\lambda}(w_r)} \Tr[{\Phi}(w_r)^\top (K+\lambda I_n)^{-1}{\Phi}(w_r)]\\
	= & ~ \frac{q_\lambda(w_r) s_{\bar{q}_{\lambda}}}{\tilde{q}_{\lambda}(w_r)}\\
	\leq & ~ s_{\bar{q}_{\lambda}}.
\end{align*}
where the first step follows from $\|A\| \leq \Tr[A]$ for any positive semidefinite matrix, the second step follows $\Tr[AB]=\Tr[BA]$, the third step follows from the definition of $V,\Sigma$, the fourth step follows from the definition of leverage score $q_\lambda(\cdot)$ as defined in Definition~\ref{def:leverage_score}, and the last step follows from the condition $\tilde{q}_{\lambda}(w) \geq q_\lambda(w)$.

Further, we have
\begin{align*}
	Y_r^2 = & ~ \frac{p(w_r)^2}{\bar{q}_{\lambda}(w_r)^2} \Sigma^{-1}V{\Phi}(w_r){\Phi}(w_r)^\top V^\top \Sigma^{-1} \Sigma^{-1}V{\Phi}(w_r){\Phi}(w_r)^\top V^\top \Sigma^{-1} \\
	= & ~ \frac{p(w_r)^2}{\bar{q}_{\lambda}(w_r)^2} \Sigma^{-1}V{\Phi}(w_r){\Phi}(w_r)^\top (K+\lambda I_n)^{-1} {\Phi}(w_r){\Phi}(w_r)^\top V^\top \Sigma^{-1} \\
	\preceq & ~ \frac{p(w_r)^2}{\bar{q}_{\lambda}(w_r)^2} \Tr[{\Phi}(w_r)^\top (K+\lambda I_n)^{-1} {\Phi}(w_r)] \Sigma^{-1}V{\Phi}(w_r){\Phi}(w_r)^\top V^\top \Sigma^{-1} \\
	= & ~ \frac{p(w_r) q_\lambda(w_r)}{\bar{q}_{\lambda}(w_r)^2} \Sigma^{-1}V{\Phi}(w_r) {\Phi}(w_r)^\top V^\top \Sigma^{-1} \\
	= & ~ \frac{q_\lambda(w_r)}{\bar{q}_{\lambda}(w_r)} Y_r \\
	= & ~ \frac{q_\lambda(w_r) s_{\tilde{q}_{\lambda}}}{\tilde{q}_{\lambda}(w_r)} Y_r \\
	\preceq & ~ s_{\tilde{q}_{\lambda}} Y_r.
\end{align*}
where the first step follows from the definition of $Y_r$, the second step follows from  the definition of $V,\Sigma$, the third step follows from $\|A\| \leq \Tr[A]$ for any positive semidefinite matrix, the fourth step follows from the definition of leverage score $q_\lambda(\cdot)$ as defined in Definition~\ref{def:leverage_score}, the fifth step follows from the definition of $Y_r$, the sixth step follows from the definition of $\bar{q}_{\lambda}(\cdot)$, and the last step follows from the condition $\tilde{q}_{\lambda}(w) \geq q_\lambda(w)$.

Thus, let $\lambda_1 \geq \lambda_2 \geq \cdots \geq \lambda_n$ be the eigenvalues of $K$, we have
\begin{align*}
	\E_{\bar{q}_{\lambda}}[Y_r^2] \preceq & ~ \E_{\bar{q}_{\lambda}}[s_{\tilde{q}_{\lambda}} Y_r] \\
	= & ~ s_{\tilde{q}_{\lambda}} \Sigma^{-1}VKV\Sigma^{-1} \\
	= & ~ s_{\tilde{q}_{\lambda}} (I_n - \lambda\Sigma^{-2}) \\
	= & ~ s_{\tilde{q}_{\lambda}} \cdot \diag\{\lambda_1/(\lambda_1 +\lambda),\cdots, \lambda_n/(\lambda_n +\lambda)\}:= D.
\end{align*}
So by applying Lemma~\ref{lem:con_lev}, we have
\begin{align*}
	\Pr\Big\| \frac{1}{m}\sum_{r=1}^m Y_r -\E[Y_r] \Big\| \geq \epsilon] \leq ~ & \frac{8\Tr[D]}{\|D\|}\exp\Big( \frac{-m\epsilon^2/2}{\|D\|+2s_{\tilde{q}_{\lambda}} \epsilon/3} \Big) \\
	\leq & ~ \frac{8 s_{\tilde{q}_{\lambda}}\cdot s_{\lambda}(K)}{\lambda_1/(\lambda_1 + \lambda)} \exp\Big( \frac{-m\epsilon^2 }{ 2s_{\tilde{q}_{\lambda}}(1 + 2\epsilon/3) } \Big) \\
	\leq & ~ 16s_{\tilde{q}_{\lambda}}\cdot s_{\lambda}(K) \exp\Big( \frac{-m\epsilon^2 }{ 2s_{\tilde{q}_{\lambda}}(1 + 2\epsilon/3)} \Big) \\
	\leq & ~ 16s_{\tilde{q}_{\lambda}}\cdot s_{\lambda}(K) \exp\Big( \frac{-3m\epsilon^2 }{ 8s_{\tilde{q}_{\lambda}}} \Big) \\
	\leq & ~ \delta
\end{align*}
where the first step follows from Lemma~\ref{lem:con_lev}, the second step follows from the definition of $D$ and $s_\lambda(K) = \Tr[(K+\lambda I_n)^{-1}K] = \sum_{i}^n\lambda_i/(\lambda_i+\lambda) = \Tr[D]$, the third step follows the condition $\lambda\in(0,\|K\|)$, the fourth step follows the condition $\epsilon\in(0,1/2)$, and the last step follows from the bound on $m$.
\end{proof}

\begin{remark}
Above results can be generalized to $\C$. Note in the random Fourier feature case, we have $d_1 = d$, $d_2 = 1$, $\phi(x,w) = e^{ -2 \pi \i w^\top x } \in \C$ and $p(\cdot)$ denotes the Fourier transform density distribution. In the Neural Tangent Kernel case, we have $d_1 = d_2 = d $, $\phi(x,w) = x\sigma'(w^\top x)$ and $p(\cdot)$ denotes the probability density function of standard Gaussian distribution $\N(0,I_d)$. So they are both special cases in our framework.
\end{remark}
   %%% Section C. Generalization result of leverage score sampling for approximating kernels
%\newpage

\section{Equivalence between training neural network with regularization and kernel ridge regression under leverage score sampling}\label{sec:eq_lev}
In this section, we connected the neural network theory with the leverage score sampling theory by showing a new equivalence result between training reweighed neural network with regularization under leverage score initialization and corresponding neural tangent kernel ridge regression. Specifically, we prove Theorem~\ref{thm:equivalence_train_lev}. Due to the similarity of the results to Section~\ref{sec:equiv}, we present this section in the same framework.

Section~\ref{sec:equiv_preli_D} introduces new notations and states the standard data assumptions again. Section~\ref{sec:def_D} restates and supplements the definitions in the paper. Section~\ref{sec:lemma_D} presents the key lemmas about the leverage score initialization and related properties, which are crucial to the proof. Section~\ref{sec:proof_sketch_D} provides a brief proof sketch. Section~\ref{sec:main_D} restates and proves the main result Theorem~\ref{thm:equivalence_train_lev} following the proof sketch.

Here, we list the locations where definitions and theorems in the paper are restated. Definition~\ref{def:f_nn_lev_intro} is restated in Definition~\ref{def:f_nn_lev}. Theorem~\ref{thm:equivalence_train_lev_intro} is restated in Theorem~\ref{thm:equivalence_train_lev}.

\subsection{Preliminaries}\label{sec:equiv_preli_D}
Let's define the following notations:
\begin{itemize}
    \item $\bar{u}_{\ntk}(t)=\kappa \bar{f}_{\ntk}(\beta(t),X) = \kappa\bar{\Phi}(X)\beta(t) \in \R^n$ be the prediction of the kernel ridge regression for the training data with respect to $\bar{H}(0)$ at time $t$. (See Definition~\ref{def:krr_ntk_lev})
    \item $\bar{u}^* = \lim_{t \rightarrow \infty} \bar{u}_{\ntk}(t)$ (See Eq.~\eqref{eq:def_u_*_lev})
    \item $\bar{u}_{\ntk,\test}(t) = \kappa \bar{f}_{\ntk}(\beta(t),x_{\test}) = \kappa\bar{\Phi}(x_{\test}) \beta(t) \in \R$ be the prediction of the kernel ridge regression for the test data with respect to $\bar{H}(0)$ at time $t$. (See Definition~\ref{def:krr_ntk_lev})
    \item $\bar{u}_{\test}^* = \lim_{t\rightarrow \infty} \bar{u}_{\ntk,\test}( t ) $ (See Eq.~\eqref{eq:def_u_test_*_lev})
    \item $\bar{\k}_t(x_{\test},X) \in \R^n$ be the induced kernel between the training data and test data at time $t$, where
    \begin{align*}
    [\bar{\k}_t(x_{\test},X)]_i 
    = \bar{\k}_t(x_{\test},x_i)
    = \left\langle \frac{\partial \bar{f}(W(t),x_{\test})}{\partial W(t)},\frac{\partial \bar{f}(W(t),x_i)}{\partial W(t)} \right\rangle
    \end{align*}
    (see Definition~\ref{def:dynamic_kernel_lev})
    \item $\bar{u}_{\nn}(t) = \kappa \bar{f}_{\nn}(W(t),X) \in \R^n$ be the prediction of the reweighed neural network with leverage score initialization for the training data at time $t$. (See Definition~\ref{def:f_nn_lev})
    \item $\bar{u}_{\nn,\test}(t) = \kappa \bar{f}_{\nn}(W(t),x_{\test}) \in \R$ be the prediction of the reweighed neural network with leverage score initialization for the test data at time $t$ (See Definition~\ref{def:f_nn_lev})
\end{itemize}

\begin{assumption}[data assumption]\label{ass:data_assumption_D}
We made the following assumptions: \\
1. For each $i \in [n]$, we assume $|y_i| = O(1)$.\\
2. $H^{\cts}$ is positive definite, i.e., $\Lambda_0:=\lambda_{\min}(H^{\cts})>0$.\\
3. All the training data and test data have Euclidean norm equal to 1.
\end{assumption}

\subsection{Definitions}\label{sec:def_D}
\begin{definition}[Training reweighed neural network with regularization, restatement of Definition~\ref{def:f_nn_lev_intro}]\label{def:f_nn_lev}
Given training data matrix $X\in\R^{n\times d}$ and corresponding label vector $Y\in\R^n$. Let $\kappa\in(0,1)$ be a small multiplier. Let $\lambda\in(0,1)$ be the regularization parameter. Given any probability density distribution $q(\cdot):\R^d\to\R_{> 0}$. Let $p(\cdot)$ denotes the Gaussian distribution $\N(0,I_d)$. We initialize the network as $a_r\overset{i.i.d.}{\sim} \unif[\{-1,1\}]$ and $w_r(0)\overset{i.i.d.}{\sim} q$. Then we consider solving the following optimization problem using gradient descent:
\begin{align}\label{problem:nn_lev}
\min_{W} \frac{1}{2}\| Y - \kappa \bar{f}_{\nn}(W,X) \|_2 + \frac{1}{2}\lambda\|W\|_F^2.
\end{align}
where $\bar{f}_{\nn}(W,x) = \frac{1}{\sqrt{m}} \sum_{r=1}^m a_r \sigma (w_r^\top X) \sqrt{\frac{p(w_r(0))}{q(w_r(0))}}$ and $\bar{f}_{\nn}(W,X) = [\bar{f}_{\nn}(W,x_1),\cdots,\bar{f}_{\nn}(W,x_n)]^\top\in\R^n$. We denote $w_r(t),r\in[m]$ as the variable at iteration $t$. We denote
\begin{align}\label{eq:nn_predict_train_lev}
	\bar{u}_{\nn}(t) = \kappa \bar{f}_{\nn}(W(t),X) = \frac{\kappa}{\sqrt{m}} \sum_{r=1}^m a_r \sigma (w_r(t)^\top X) \sqrt{\frac{p(w_r(0))}{q(w_r(0))}}
\end{align}
as the training data predictor at iteration $t$. Given any test data $x_{\test}\in\R^d$, we denote
\begin{align}\label{eq:nn_predict_test_lev}
	\bar{u}_{\nn,\test}(t) = \kappa \bar{f}_{\nn}(W(t),x_{\test}) = \frac{\kappa}{\sqrt{m}} \sum_{r=1}^m a_r \sigma (w_r(t)^\top x_{\test}) \sqrt{\frac{p(w_r(0))}{q(w_r(0))}}
\end{align}
as the test data predictor at iteration $t$.
\end{definition}

\begin{definition}[Reweighed dynamic kernel]\label{def:dynamic_kernel_lev}
Given $W(t) \in \R^{d \times m}$ as the parameters of the neural network at training time $t$ as defined in Definition~\ref{def:f_nn_lev}. For any data $x,z\in\R^d$, we define $\k_t(x,z)\in\R$ as
\begin{align*}
    \bar{\k}_t(x,z)
    = \left\langle \frac{\d \bar{f}_{\nn}(W(t),x)}{\d W(t)},\frac{\d \bar{f}_{\nn}(W(t),z)}{\d W(t)} \right\rangle
\end{align*}
Given training data matrix $X=[x_1,\cdots,x_n]^\top\in\R^{n\times d}$, we define $\bar{H}^{(t)}\in\R^{n\times n}$ as
\begin{align*}
	[\bar{H}(t)]_{i,j} = \bar{\k}_{t}(x_i, x_j)\in\R.
\end{align*}
We denote $\bar{\Phi}(x) = [x^\top\sigma'(w_1(0)^\top x)\sqrt{\frac{p(w_1(0))}{q(w_1(0))}},\cdots,x^\top\sigma'(w_m(0)^\top x)\sqrt{\frac{p(w_r(0))}{q(w_r(0))}}]^\top\in\R^{md}$ as the feature vector corresponding to $\bar{H}(0)$, which satisfies
\begin{align*}
	[\bar{H}(0)]_{i,j} = \bar{\Phi}(x_i)^\top \bar{\Phi}(x_j)
\end{align*}
for all $i,j\in[n]$. We denote $\bar{\Phi}(X)=[\bar{\Phi}(x_1),\cdots,\bar{\Phi}(x_n)]^\top\in\R^{n\times md}$. Further, given a test data $x_{\test}\in\R^d$, we define $\bar{\k}_t(x_{\test},X)\in\R^n$ as
\begin{align*}
    \bar{\k}_t(x_{\test},X) = [\bar{\k}_t(x_{\test},x_1), \cdots, \bar{\k}_t(x_{\test},x_n)]^\top\in\R^n.
\end{align*}
\end{definition}

\begin{definition}[Kernel ridge regression with $\bar{H}(0)$]\label{def:krr_ntk_lev}
Given training data matrix $X=[x_1,\cdots,x_n]^\top\in\R^{n\times d}$ and corresponding label vector $Y\in\R^n$. Let $\bar{\k}_{0}$, $\bar{H}(0)\in\R^{n\times n}$ and $\bar{\Phi}$ be the neural tangent kernel and corresponding feature functions defined as in Definition~\ref{def:dynamic_kernel_lev}. Let $\kappa\in(0,1)$ be a small multiplier. Let $\lambda\in(0,1)$ be the regularization parameter. Then we consider the following neural tangent kernel ridge regression problem:
\begin{align}\label{eq:krr_ntk_lev}
\min_{\beta} \frac{1}{2}\| Y - \kappa \bar{f}_{\ntk}(\beta,X) \|_2^2 + \frac{1}{2}\lambda\|\beta\|_2^2.
\end{align}
where $\bar{f}_{\ntk}(\beta,x) = \bar{\Phi}(x)^\top \beta$ denotes the prediction function is corresponding RKHS and $\bar{f}_{\ntk}(\beta,X) = [\bar{f}_{\ntk}(\beta,x_1),\cdots,\bar{f}_{\ntk}(\beta,x_n)]^\top\in\R^{n}$. Consider the gradient flow of solving problem~\eqref{eq:krr_ntk_lev} with initialization $\beta(0) = 0$.
We denote $\beta(t)\in\R^{md}$ as the variable at iteration $t$. We denote
\begin{align}\label{eq:ntk_predict_train_lev}
	\bar{u}_{\ntk}(t) = \kappa\bar{\Phi}(X)\beta(t)
\end{align} 
as the training data predictor at iteration $t$. Given any test data $x_{\test}\in\R^d$, we denote
\begin{align}\label{eq:ntk_predict_test_lev}
	\bar{u}_{\ntk,\test}(t) = \kappa\bar{\Phi}(x_{\test})^\top\beta(t)
\end{align} 
as the test data predictor at iteration $t$. Note the gradient flow converge the to optimal solution of problem~\eqref{eq:krr_ntk_lev} due to the strongly convexity of the problem. We denote
\begin{align}\label{eq:def_beta_*_lev}
	\bar{\beta}^* = \lim_{t\to\infty} \beta(t) = \kappa (\kappa^2 \bar{\Phi}(X)^\top \bar{\Phi}(X) + \lambda I)^{-1} \bar{\Phi}(X)^\top Y 
\end{align}
and the optimal training data predictor
\begin{align}\label{eq:def_u_*_lev}
	\bar{u}^* = \lim_{t\to\infty} \bar{u}_{\ntk}(t) = \kappa \bar{\Phi}(X)\bar{\beta}^* = \kappa^2 \bar{H}(0)(\kappa^2 \bar{H}(0)+\lambda I)^{-1}Y
\end{align}
and the optimal test data predictor
\begin{align}\label{eq:def_u_test_*_lev}
	\bar{u}_{\test}^* = \lim_{t\to\infty} \bar{u}_{\ntk,\test}(t) = \kappa \bar{\Phi}(x_{\test})^\top \bar{\beta}^* = \kappa^2\bar{\k}_{0}(x_{\test}, X)^\top (\kappa^2 \bar{H}(0)+\lambda I)^{-1}Y.
\end{align}
\end{definition}

\subsection{Leverage score sampling, gradient flow, and linear convergence}\label{sec:lemma_D}

Recall in the main body we connect the leverage score sampling theory and convergence theory of the neural network training by observing 
\begin{align*}
    \k_{\ntk}(x, z) & = \E \left[\left\langle \frac{\partial f_{\nn}(W,x)}{\partial W},\frac{\partial f_{\nn}(W,z)}{\partial W} \right\rangle \right] \\
    & = \E_{w \sim p}[\phi(x,w)^\top\phi(z,w)]
\end{align*}
where $\phi(x,w) = x\sigma'(w^\top x)\in\R^{d}$ and $p(\cdot)$ denotes the probability density function of standard Gaussian distribution $\N(0,I_d)$.
Thus, given regularization parameter $\lambda>0$, we can define the ridge leverage function with respect to $H^{\cts}$ defined in Definition~\ref{def:ntk_phi} as
\begin{align*}
	q_\lambda(w) = p(w) \Tr[\Phi(w)^\top ( H^{\cts} + \lambda I_n )^{-1} \Phi(w)]
\end{align*}
and corresponding probability density function 
\begin{align}\label{eq:lev_dis_lev}
q(w)=\frac{q_{\lambda}(w)}{s_{\lambda}(H^{\cts})}
\end{align}
where $\Phi(w) = [\phi(x_1,w)^\top, \cdots, \phi(x_n,w)^\top]^\top\in\R^{n\times d_2}$.

\begin{lemma}[property of leverage score sampling distribution]\label{lem:property_lev}
Let $p(\cdot)$ denotes the standard Gaussian distribution $\N(0,I_d)$. Let $q(\cdot)$ be defined as in \eqref{eq:lev_dis_lev}. Assume $\Tr[\Phi(w)\Phi(w)^\top]=O(n)$ and $\lambda \leq \Lambda_0/2$. Then for all $w\in\R^d$ we have
\begin{align*}
	c_1p(w) \leq q(w) \leq c_2 p(w)
\end{align*}
where $c_1=O(\frac{1}{n})$ and $c_2=O(\frac{1}{\Lambda_0})$.
\end{lemma}
\begin{proof}
	Note by assumption $s_{\lambda}(H^{\cts}) = \Tr[(H^{\cts} + \lambda I_n)^{-1}H^{\cts}] = O(n)$. Further, note for any $w\in\R^d$,
	\begin{align*}
		q_\lambda(w) = & ~ p(w) \Tr[\Phi(w)^\top ( H^{\cts} + \lambda I_n )^{-1} \Phi(w)]\\
		\leq & ~ p(w) \Tr[\Phi(w)^\top \Phi(w)]\cdot\frac{1}{\Lambda_0+\lambda}\\
		= & ~ p(w) \Tr[\sum_{i=1}^n \phi(x_i,w) \phi(x_i, w)^\top]\cdot\frac{1}{\Lambda_0+\lambda}\\
		= & ~ p(w) \sum_{i=1}^n \Tr[ \phi(x_i,w) \phi(x_i, w)^\top]\cdot\frac{1}{\Lambda_0+\lambda}\\
		= & ~ p(w) \sum_{i=1}^n \Tr[ \phi(x_i,w)^\top \phi(x_i, w)]\cdot\frac{1}{\Lambda_0+\lambda}\\
		= & ~ p(w) \sum_{i=1}^n \|x_i\|_2^2\sigma'(w^\top x_i)^2 \cdot\frac{1}{\Lambda_0+\lambda}\\
		\leq & ~ p(w)\frac{n}{\Lambda_0+\lambda}
	\end{align*}
	where the first step follows from $H^{\cts} \succeq \Lambda_0 I_n$, the second step follows from the definition of $\Phi$, the third step follows from the linearity of trace operator, the fourth step follows from $\Tr(AB)=\Tr(BA)$, the fifth step follows from the definition of $\phi$, and the last step follows from $\|x_i\|_2 = 1$ and $\sigma'(\cdot)\leq 1$.

	Thus, combining above facts, we have 
	\begin{align*}
		q(w) = \frac{q_{\lambda}(w)}{s_{\lambda}(H^{\cts})} \leq p(w)\frac{n}{(\Lambda_0+\lambda)s_{\lambda}(H^{\cts})} = c_2p(w)
	\end{align*}
	hold for all $w\in\R^d$, where $c_2 = O(\frac{1}{\Lambda_0})$.

	Similarly, note $H^{\cts}\preceq n I_n$, we have 
	\begin{align*}
		q_\lambda(w) \geq p(w)\Tr[\Phi(w)^\top \Phi(w)]\cdot \frac{1}{n+\lambda}
	\end{align*}
	which implies
	\begin{align*}
		q(w) = \frac{q_{\lambda}(w)}{s_{\lambda}(H^{\cts})} \geq p(w)\frac{\Tr[\Phi(w)^\top \Phi(w)]}{(n+\lambda)s_{\lambda}(H^{\cts})} = c_1 p(w)
	\end{align*}
	hold for all $w\in\R^d$, where $c_1 = O(\frac{1}{n})$. Combining above results, we complete the proof.
\end{proof}
\begin{lemma}[leverage score sampling]\label{lem:leverage_score_sampling}
Let $\bar{H}(t)$, $H^{\cts}$ be the kernel defined as in Definition~\ref{def:dynamic_kernel_lev} and Definition~\ref{def:ntk_phi}. Let $\Lambda_0 > 0$ be defined as in Definition~\ref{def:ntk_phi}. Let $p(\cdot)$ denotes the probability density function for Gaussian $\N(0,I_d)$. Let $q(\cdot)$ denotes the leverage sampling distribution with respect to $p(\cdot)$, $\bar{H}(0)$ and $\lambda$ defined in Definition~\ref{def:lev_distribution}. Let $\Delta\in(0,1/4)$. Then we have
\begin{align*}
	\E_{q}[\bar{H}(0)] = H^{\cts}.
\end{align*}
By choosing $m \geq \wt{O}(\Delta^{-2} s_{\lambda}(H^{\cts})$, with probability at least $1-\delta$,
\begin{align*}
	(1-\Delta)(H^{\cts} + \lambda I) \preceq \bar{H}(0) + \lambda I \preceq (1+\Delta)(H^{\cts} + \lambda I)
\end{align*}
Further, if $\lambda \leq \Lambda_0$, we have with probability at least $1-\delta$,
\begin{align*}
	\bar{H}(0) \succeq \frac{\Lambda_0}{2} I_n.
\end{align*}
Here $\wt{O}(\cdot)$ hides $\poly\log(s_{\lambda}(\bar{H}(0))/\delta)$.
\end{lemma}

\begin{proof}
Note for any $i,j\in[n]$,
\begin{align*}
	[\E_{q}[\bar{H}(0)]]_{i,j} = & ~ \E_{q} [\bar{\Phi}(x_i)^\top \bar{\Phi}(x_j)]\\
	= & ~ \E_{q}[\langle \frac{\d \bar{f}_{\nn}(W(0),x_i)}{\d W},\frac{\d \bar{f}_{\nn}(W(0),x_j)}{\d W} \rangle] \\
	= & ~ \E_{p}[\langle \frac{\d {f}_{\nn}(W(0),x_i)}{\d W},\frac{\d {f}_{\nn}(W(0),x_j)}{\d W} \rangle] \\
	= & ~ [H^{\cts}]_{i,j}
\end{align*}
where the first step follows from the definition of $\bar{H}(0)$, the second step follows the definition of $\bar{\Phi}$, the third step calculates the expectation, and the last step follows the definition of $H^{\cts}$.

Also, by applying Theorem~\ref{thm:leverage_score} directly, we have with probability at least $1-\delta$,
\begin{align*}
	(1-\Delta)(H^{\cts} + \lambda I) \preceq \bar{H}(0) + \lambda I \preceq (1+\Delta)(H^{\cts} + \lambda I)
\end{align*}
if $m \geq \wt{O}(\Delta^{-2} s_{\lambda}(\bar{H}(0))$.

Since 
$$\bar{H}(0) + \lambda I \succeq (1-\Delta)(H^{\cts} + \lambda I) \succeq (1-\Delta)(\Lambda_0 + \lambda) I_n ,$$
we have 
$$\bar{H}(0) \succeq [(1-\Delta)(\Lambda_0 + \lambda) - \lambda] I_n \succeq \frac{\Lambda_0}{2} I_n,$$
which completes the proof.
\end{proof}

\begin{lemma}[Gradient flow of kernel ridge regression, parallel to Lemma~\ref{lem:gradient_flow_of_krr}]\label{lem:gradient_flow_of_krr_lev}
Given training data matrix $X\in\R^{n\times d}$ and corresponding label vector $Y\in\R^n$. Let $\bar{f}_{\ntk}$, $\beta(t)\in\R^{md}$, $\kappa\in(0,1)$ and $\bar{u}_{\ntk}(t)\in\R^n$ be defined as in Definition~\ref{def:krr_ntk_lev}. Let $\bar{\Phi},~\bar{\k}_t$ be defined as in Definition~\ref{def:dynamic_kernel_lev}. Then for any data $z\in\R^d$, we have
\begin{align*}
	\frac{\d \bar{f}_{\ntk}(\beta(t), z)}{\d t} = \kappa \cdot \bar{\k}_{0}(z, X)^\top ( Y - \bar{u}_{\ntk}(t) ) - \lambda \cdot \bar{f}_{\ntk}(\beta(t), z).
\end{align*}
\end{lemma}
\begin{proof}
Denote $L(t)= \frac{1}{2}\|Y-\bar{u}_{\ntk}(t)\|_2^2+\frac{1}{2}\lambda\|\beta(t)\|_2^2$. By the rule of gradient descent, we have
\begin{align*}
	\frac{\d \beta(t)}{\d t}=-\frac{\d L}{\d \beta}=\kappa \bar{\Phi}(X)^\top(Y-\bar{u}_{\ntk}(t))-\lambda\beta(t).
\end{align*}
Thus we have
\begin{align*}
	\frac{\d \bar{f}_{\ntk}(\beta(t), z)}{\d t}
	= & ~ \frac{\d \bar{f}_{\ntk}(\beta(t), z)}{\d \beta(t)}\frac{\d \beta(t)}{\d t} \\
	= & ~ \bar{\Phi}(z)^\top (\kappa\bar{\Phi}(X)^\top(Y-\bar{u}_{\ntk}(t))-\lambda\beta(t)) \\
	= & ~ \kappa\bar{\k}_{0}(z, X)^\top (Y-\bar{u}_{\ntk}(t))-\lambda\bar{\Phi}(z)^\top\beta(t) \\
	= & ~ \kappa\bar{\k}_{0}(z, X)^\top (Y-\bar{u}_{\ntk}(t))-\lambda \bar{f}_{\ntk}(\beta(t), z),
\end{align*}
where the first step is due to chain rule, the second step follows from the fact $ \d \bar{f}_{\ntk}(\beta, z)/ \d \beta=\bar{\Phi}(z)$, the third step is due to the definition of the kernel $\bar{\k}_{0}(z, X)=\bar{\Phi}(X)\bar{\Phi}(z) \in \R^{n}$, and the last step is due to the definition of $\bar{f}_{\ntk}(\beta(t), z)\in\R$.
\end{proof}

\begin{corollary}[Gradient of prediction of kernel ridge regression, parallel to Corollary~\ref{cor:ntk_gradient}]\label{cor:ntk_gradient_lev}
Given training data matrix $X=[x_1,\cdots,x_n]^\top \in \R^{n\times d}$ and corresponding label vector $Y \in \R^n$. Given a test data $x_{\test}\in\R^d$. Let $\bar{f}_{\ntk}$, $\beta(t)\in\R^{md}$, $\kappa\in(0,1)$, $\bar{u}_{\ntk}(t)\in\R^n$ and $\bar{u}_{\ntk,\test}(t)\in\R$ be defined as in Definition~\ref{def:krr_ntk_lev}. Let $\bar{\k}_{t},~\bar{H}(0)\in\R^{n\times n}$ be defined as in Definition~\ref{def:dynamic_kernel_lev}. Then we have 
\begin{align*}
	\frac{\d \bar{u}_{\ntk}(t)}{\d t} & = \kappa^2 \bar{H}(0) ( Y - \bar{u}_{\ntk}(t) ) - \lambda \cdot \bar{u}_{\ntk}(t)\\
	\frac{\d \bar{u}_{\ntk, \test}(t)}{\d t} & = \kappa^2 \bar{\k}_{0}(x_{\test}, X)^\top  ( Y - \bar{u}_{\ntk}(t) ) - \lambda \cdot \bar{u}_{\ntk, \test}(t).
\end{align*}
\end{corollary}
\begin{proof}
Plugging in $z = x_i\in\R^d$ in Lemma~\ref{lem:gradient_flow_of_krr_lev}, we have
\begin{align*}
	\frac{\d \bar{f}_{\ntk}(\beta(t), x_i)}{\d t} = \kappa \bar{\k}_{0}(x_i, X)^\top ( Y - \bar{u}_{\ntk}(t) ) - \lambda \cdot \bar{f}_{\ntk}(\beta(t), x_i).
\end{align*}
Note $[\bar{u}_{\ntk}(t)]_i = \kappa \bar{f}_{\ntk}(\beta(t), x_i)$ and $[\bar{H}(0)]_{:,i} = \bar{\k}_{0}(x_i, X)$, so writing all the data in a compact form, we have
\begin{align*}
	\frac{\d \bar{u}_{\ntk}(t)}{\d t} = \kappa^2 \bar{H}(0) ( Y - \bar{u}_{\ntk}(t) ) - \lambda \cdot \bar{u}_{\ntk}(t).
\end{align*}
Plugging in data $z = x_{\test}\in\R^d$ in Lemma~\ref{lem:gradient_flow_of_krr_lev}, we have
\begin{align*}
	\frac{\d \bar{f}_{\ntk}(\beta(t), x_{\test})}{\d t} = \kappa \bar{\k}_{0}(x_{\test}, X)^\top ( Y - \bar{u}_{\ntk}(t) ) - \lambda \cdot \bar{f}_{\ntk}(\beta(t), x_{\test}).
\end{align*}
Note by definition, $\bar{u}_{\ntk,\test}(t) = \kappa \bar{f}_{\ntk}(\beta(t), x_{\test}) \in \R$, so we have
\begin{align*}
	\frac{\d \bar{u}_{\ntk, \test}(t)}{\d t} = \kappa^2 \bar{\k}_{0}(x_{\test}, X)^\top ( Y - \bar{u}_{\ntk}(t) ) - \lambda \cdot \bar{u}_{\ntk, \test}(t).
\end{align*}
\end{proof}

\begin{lemma}[Linear convergence of kernel ridge regression, parallel to Lemma~\ref{lem:linear_converge_krr}]\label{lem:linear_converge_krr_lev}
Given training data matrix $X=[x_1,\cdots,x_n]^\top \in \R^{n\times d}$ and corresponding label vector $Y\in\R^n$. Let $\kappa\in(0,1)$, $\bar{u}_{\ntk}(t) \in \R^n$ and $\bar{u}^* \in \R^n$ be defined as in Definition~\ref{def:krr_ntk_lev}. Let $\Lambda_0 > 0$ be defined as in Definition~\ref{def:ntk_phi}. Let $\lambda > 0$ be the regularization parameter. Then we have
\begin{align*}
\frac{\d \|\bar{u}_{\ntk}(t)-\bar{u}^*\|_2^2}{\d t} \le - (\kappa^2 \Lambda_0+\lambda) \|\bar{u}_{\ntk}(t)-\bar{u}^*\|_2^2.
\end{align*}
Further, we have
\begin{align*}
	\|u_{\ntk}(t)-u^*\|_2 \leq e^{-(\kappa^2 \Lambda_0+\lambda)t/2} \|u_{\ntk}(0)-u^*\|_2.
\end{align*}

\end{lemma}
\begin{proof}
Let $\bar{H}(0) \in \R^{n \times n}$ be defined as in Definition~\ref{def:dynamic_kernel_lev}. Then
\begin{align}\label{eq:322_1_lev}
	\kappa^2 \bar{H}(0)(Y-\bar{u}^*) 
	= & ~ \kappa^2 \bar{H}(0)(Y-\kappa^2 \bar{H}(0)(\kappa^2 \bar{H}(0)+\lambda I_n)^{-1}Y) \notag \\
	= & ~ \kappa^2 \bar{H}(0))(I_n-\kappa^2 \bar{H}(0))(\kappa^2 \bar{H}(0)+\lambda I)^{-1})Y \notag \\
	= & ~ \kappa^2 \bar{H}(0)(\kappa^2 \bar{H}(0)+\lambda I_n - \kappa^2 \bar{H}(0))(\kappa^2 \bar{H}(0)+\lambda I_n)^{-1} Y \notag \\
	= & ~ \kappa^2 \lambda \bar{H}(0)(\kappa^2 \bar{H}(0)+\lambda I_n)^{-1} Y \notag \\
	= & ~ \lambda \bar{u}^*,
\end{align}
where the first step follows the definition of $\bar{u}^* \in \R^n$, the second to fourth step simplify the formula, and the last step use the definition of $\bar{u}^* \in \R^n$ again.
So we have
\begin{align}\label{eq:322_2_lev}
	\frac{\d \|\bar{u}_{\ntk}(t)-\bar{u}^*\|_2^2}{\d t} 
	= & ~ 2(\bar{u}_{\ntk}(t)-\bar{u}^*)^\top \frac{\d \bar{u}_{\ntk}(t)}{\d t} \notag\\
	= & ~ -2\kappa^2 (\bar{u}_{\ntk}(t)-\bar{u}^*)^\top \bar{H}(0) (\bar{u}_{\ntk}(t) - Y) -2\lambda(\bar{u}_{\ntk}(t)-\bar{u}^*)^\top \bar{u}_{\ntk}(t) \notag\\
	= & ~ -2\kappa^2 (\bar{u}_{\ntk}(t)-\bar{u}^*)^\top \bar{H}(0) (\bar{u}_{\ntk}(t) - \bar{u}^*) + 2\kappa^2 (\bar{u}_{\ntk}(t)-\bar{u}^*)^\top \bar{H}(0)) (Y-\bar{u}^*) \notag\\
	& ~ -2\lambda(\bar{u}_{\ntk}(t)-\bar{u}^*)^\top \bar{u}_{\ntk}(t)\notag\\
	= & ~ -2\kappa^2 (\bar{u}_{\ntk}(t)-\bar{u}^*)^\top \bar{H}(0) (\bar{u}_{\ntk}(t) -\bar{u}^*) + 2\lambda(\bar{u}_{\ntk}(t)-\bar{u}^*)^\top \bar{u}^* \notag\\
	& ~ -2\lambda(\bar{u}_{\ntk}(t)-\bar{u}^*)^\top \bar{u}_{\ntk}(t) \notag\\
	= & ~ -2(\bar{u}_{\ntk}(t)-\bar{u}^*)^\top (\kappa^2 \bar{H}(0)+\lambda I) (\bar{u}_{\ntk}(t) - \bar{u}^*) \notag\\
	\leq & ~ -(\kappa^2 \Lambda_0 + \lambda)\|\bar{u}_{\ntk}(t)-\bar{u}^*\|_2^2,
\end{align}
where the first step follows the chain rule, the second step follows Corollary~\ref{cor:ntk_gradient_lev}, the third step uses basic linear algebra, the fourth step follows Eq.~\eqref{eq:322_1_lev}, the fifth step simplifies the expression, and the last step follows from Lemma~\ref{lem:leverage_score_sampling}.
Further, since 
\begin{align*}
	& ~ \frac{\d (e^{(\kappa^2 \Lambda_0+\lambda)t}\|\bar{u}_{\ntk}(t)-\bar{u}^*\|_2^2)}{\d t} \\
	= & ~ (\kappa^2 \Lambda_0+\lambda)e^{(\kappa^2 \Lambda_0+\lambda)t}\|\bar{u}_{\ntk}(t)-\bar{u}^*\|_2^2 + e^{(\kappa^2 \Lambda_0+\lambda)t}\cdot\frac{\d \|\bar{u}_{\ntk}(t)-\bar{u}^*\|_2^2}{\d t} \\
	\leq & ~ 0,
	% \leq & ~ 
\end{align*}
where the first step calculates the gradient, and the second step follows from Eq.~\eqref{eq:322_2_lev}. Thus, $e^{(\kappa^2 \Lambda_0+\lambda)t}\|\bar{u}_{\ntk}(t)-\bar{u}^*\|_2^2$ is non-increasing, which implies
\begin{align*}
		\|\bar{u}_{\ntk}(t)-\bar{u}^*\|_2 \leq e^{-(\kappa^2 \Lambda_0+\lambda)t/2} \|\bar{u}_{\ntk}(0)-\bar{u}^*\|_2.
\end{align*}
\end{proof}

\begin{lemma}[Gradient flow of neural network training, Parallel to Lemma~\ref{lem:gradient_flow_of_nn}]\label{lem:gradient_flow_of_nn_lev}
Given training data matrix $X \in \R^{n\times d}$ and corresponding label vector $Y\in\R^n$. Let $\bar{f}_{\nn}: \R^{d\times m} \times \R^{d} \rightarrow \R$, $W(t) \in \R^{d \times m}$, $\kappa\in(0,1)$ and $\bar{u}_{\nn}(t)\in\R^n$ be defined as in Definition~\ref{def:f_nn_lev}. Let $\bar{\k}_{t}: \R^d \times \R^{n\times d} \rightarrow \R^n$ be defined as in Definition~\ref{def:dynamic_kernel_lev}. Then for any data $z \in \R^d$, we have
\begin{align*}
\frac{\d \bar{f}_{\nn}(W(t),z)}{\d t} = \kappa \bar{\k}_{t}(z,X)^\top ( Y - \bar{u}_{\nn}(t) )- \lambda \cdot \bar{f}_{\nn}(W(t),z).
\end{align*}
%where $u(t)=f(W(t),a,X)=\frac{1}{\sqrt{m}}\sum_{r=1}^m a_r\sigma(w_r(t)^\top X)$.
\end{lemma}
\begin{proof}
Denote $L(t)=\frac{1}{2}\|Y-\bar{u}_{\nn}(t)\|_2^2+\frac{1}{2}\lambda\|W(t)\|_F^2$. By the rule of gradient descent, we have
\begin{align}\label{eq:323_1_lev}
	\frac{\d w_r}{\d t} = -\frac{\partial L}{\partial w_r}=(\frac{\partial \bar{u}_{\nn}}{\partial w_r})^\top(Y-\bar{u}_{\nn})-\lambda w_r.
\end{align} 
Also note for ReLU activation $\sigma$, we have
\begin{align}\label{eq:323_2_lev}
	\langle \frac{\d \bar{f}_{\nn}(W(t),z)}{\d W(t)},\lambda W(t)\rangle = & ~ \sum_{r=1}^m \Big(\frac{1}{\sqrt{m}}a_r z \sigma'(w_r(t)^\top z)\sqrt{\frac{p(w_r(0))}{q(w_r(0))}}\Big)^\top (\lambda w_r(t)) \notag \\
	= & ~ \frac{\lambda}{\sqrt{m}}\sum_{r=1}^m a_r w_r(t)^\top z \sigma'(w_r(t)^\top z)\sqrt{\frac{p(w_r(0))}{q(w_r(0))}} \notag \\
	= & ~ \frac{\lambda}{\sqrt{m}}\sum_{r=1}^m a_r \sigma(w_t(t)^\top z)\sqrt{\frac{p(w_r(0))}{q(w_r(0))}}\notag \\
	= & ~ \lambda \bar{f}_{\nn}(W(t),z),
\end{align}
where the first step calculates the derivatives, the second step follows basic linear algebra, the third step follows the property of ReLU activation: $\sigma(l) = l\sigma'(l)$, and the last step follows from the definition of $\bar{f}_{\nn}$.
Thus, we have
\begin{align*}
 & ~ \frac{\d \bar{f}_{\nn}(W(t),z)}{\d t} \\
= & ~ \langle\frac{\d \bar{f}_{\nn}(W(t),z)}{\d W(t)}, \frac{\d W(t)}{\d t}\rangle \notag \\
= & ~ \sum_{j=1}^{n}(y_j - \kappa \bar{f}_{\nn}(W(t),x_j)) \langle \frac{\d \bar{f}_{\nn}(W(t),z)}{\d W(t)},\frac{\d \kappa \bar{f}_{\nn}(W(t),x_j)}{\d W(t)} \rangle-\langle \frac{\d \bar{f}_{\nn}(W(t),z)}{\d W(t)},\lambda W(t)\rangle\notag\\
= & ~ \kappa \sum_{j=1}^{n}(y_j- \kappa \bar{f}_{\nn}(W(t),x_j)) \bar{\k}_{t}(z,x_j)-\lambda \cdot \bar{f}_{\nn}(W(t),z)\notag\\
= & ~ \kappa \bar{\k}_{t}(z,X)^\top ( Y - \bar{u}_{\nn}(t) )- \lambda \cdot \bar{f}_{\nn}(W(t),z),
\end{align*}
where the first step follows from chain rule, the second step follows from Eq.~\eqref{eq:323_1_lev}, the third step follows from the definition of $\bar{\k}_{t}$ and Eq.~\eqref{eq:323_2_lev}, and the last step rewrites the formula in a compact form.
% \begin{align}\label{eq:3}
% \sum_{r=1}^m\langle\frac{1}{\sqrt{m}}a_r x_i \sigma'(w_r(t)^\top x_i),\lambda w_r(t)\rangle = \lambda \frac{1}{\sqrt{m}}\sum_{r=1}^m a_r w_r^\top x_i {\bf 1}[w_r^\top x_i\ge 0] = \lambda u_i,
% \end{align}
% where first equality we plug in the derivative of the ReLU activation function, and the second equality is due to the definition of $u$.}
% Combining~\eqref{eq:2} and~\eqref{eq:3}, we have
% \begin{align*}
% \frac{\d u_i}{\d t}=\sum_{j=1}^n (y_j-u_j(t))H(t)_{i,j}-\lambda u_i(t)
% \end{align*}
% Write~\eqref{eq:2} in a compact form, we have 
% \begin{align*}
% \frac{\d u(t)}{\d t} = H(t) \cdot (Y-u(t)) -\lambda u(t),
% \end{align*}
% which completes the proof.
\end{proof}

\begin{corollary}[Gradient of prediction of neural network, Parallel to Lemma~\ref{cor:nn_gradient}]\label{cor:nn_gradient_lev}
Given training data matrix $X = [x_1,\cdots,x_n]^\top\in\R^{n\times d}$ and corresponding label vector $Y \in \R^n$. Given a test data $x_{\test} \in \R^d$. Let $\bar{f}_{\nn}: \R^{d\times m} \times \R^d \rightarrow \R$, $W(t) \in \R^{d\times m}$, $\kappa\in(0,1)$ and $\bar{u}_{\nn}(t) \in \R^n$ be defined as in Definition~\ref{def:f_nn_lev}. Let $\bar{\k}_{t} : \R^d \times \R^{n \times d} \rightarrow \R^n,~\bar{H}(t) \in \R^{n \times n}$ be defined as in Definition~\ref{def:dynamic_kernel_lev}. Then we have 
\begin{align*}
	\frac{\d \bar{u}_{\nn}(t)}{\d t} = & ~ \kappa^2 \bar{H}(t) ( Y - \bar{u}_{\nn}(t) ) - \lambda \cdot \bar{u}_{\nn}(t)\\
	\frac{\d \bar{u}_{\nn,\test}(t)}{\d t} = & ~ \kappa^2 \bar{\k}_{t}(x_{\test}, X)^\top ( Y - \bar{u}_{\nn}(t) ) - \lambda \cdot \bar{u}_{\nn,\test}(t).
\end{align*}
\end{corollary}
\begin{proof}
Plugging in $z = x_i\in\R^d$ in Lemma~\ref{lem:gradient_flow_of_nn_lev}, we have
\begin{align*}
	\frac{\d \bar{f}_{\nn}(W(t), x_i)}{\d t} = \kappa \bar{\k}_{t}(x_i, X)^\top ( Y - \bar{u}_{\nn}(t) ) - \lambda \cdot \bar{f}_{\nn}(W(t), x_i).
\end{align*}
Note $[\bar{u}_{\nn}(t)]_i = \kappa \bar{f}_{\nn}(W(t), x_i)$ and $[\bar{H}(t))]_{:,i} = \bar{\k}_{t}(x_i, X)$, so writing all the data in a compact form, we have
\begin{align*}
	\frac{\d \bar{u}_{\nn}(t)}{\d t} = \kappa^2 \bar{H}(t) ( Y - \bar{u}_{\nn}(t) ) - \lambda \cdot \bar{u}_{\nn}(t).
\end{align*}
Plugging in data $z = x_{\test}\in\R^d$ in Lemma~\ref{lem:gradient_flow_of_nn_lev}, we have
\begin{align*}
	\frac{\d \bar{f}_{\nn}(W(t), x_{\test})}{\d t} = \kappa \bar{\k}_{t}(x_{\test}, X)^\top ( Y - \bar{u}_{\nn}(t) ) - \lambda \cdot \bar{f}_{\nn}(W(t), x_{\test}).
\end{align*}
Note by definition, $\bar{u}_{\nn,\test}(t) = \kappa \bar{f}_{\nn}(W(t), x_{\test}) $, so we have
\begin{align*}
	\frac{\d \bar{u}_{\nn, \test}(t)}{\d t} = \kappa^2 \bar{\k}_{t}(x_{\test}, X)^\top ( Y - \bar{u}_{\nn}(t) ) - \lambda \cdot \bar{u}_{\nn, \test}(t).
\end{align*}
\end{proof}

\begin{lemma}[Linear convergence of neural network training, Parallel to Lemma~\ref{lem:linear_converge_nn}]\label{lem:linear_converge_nn_lev}
Given training data matrix $X=[x_1,\cdots,x_n]^\top\in\R^{n\times d}$ and corresponding label vector $Y \in \R^n$. Let $\kappa\in(0,1)$ and $\bar{u}_{\nn}(t) \in \R^{n \times n}$ be defined as in Definition~\ref{def:f_nn_lev}. Let $u^* \in \R^n$ be defined in Eq.~\eqref{eq:def_u_*_lev}. Let $\bar{H}(t) \in \R^{n \times n}$ be defined as in Definition~\ref{def:dynamic_kernel_lev}. Let $\lambda \in(0, \Lambda_0)$ be the regularization parameter. Assume $\|\bar{H}(t) - \bar{H}(0)\| \leq \Lambda_0/4$ holds for all $t\in[0,T]$. Then we have
\begin{align*}
	\frac{\d \|\bar{u}_{\nn}(t)-\bar{u}^*\|_2^2}{\d t} \le - \frac{1}{2}( \kappa^2 \Lambda_0+\lambda) \|\bar{u}_{\nn}(t)-\bar{u}^*\|_2^2+ 2 \kappa^2 \| \bar{H}(t) - \bar{H}(0) \|  \cdot \| \bar{u}_{\nn}(t) - \bar{u}^* \|_2 \cdot \| Y - \bar{u}^* \|_2.
\end{align*}
% holds for all $t\in[0,T]$. Further,
% \begin{align*}
	
% \end{align*}
\end{lemma}

\begin{proof}
Note same as in Lemma~\ref{lem:linear_converge_krr_lev}, we have
\begin{align}\label{eq:325_1_lev}
	\kappa^2 \bar{H}(0)(Y-\bar{u}^*) 
	= & ~ \kappa^2 \bar{H}(0)(Y-\kappa^2 \bar{H}(0)(\kappa^2 \bar{H}(0)+\lambda I_n)^{-1}Y) \notag \\
	= & ~ \kappa^2 \bar{H}(0))(I_n-\kappa^2 \bar{H}(0))(\kappa^2 \bar{H}(0)+\lambda I)^{-1})Y \notag \\
	= & ~ \kappa^2 \bar{H}(0)(\kappa^2 \bar{H}(0)+\lambda I_n - \kappa^2 \bar{H}(0))(\kappa^2 \bar{H}(0)+\lambda I_n)^{-1} Y \notag \\
	= & ~ \kappa^2 \lambda \bar{H}(0)(\kappa^2 \bar{H}(0)+\lambda I_n)^{-1} Y \notag \\
	= & ~ \lambda \bar{u}^*,
\end{align}
where the first step follows the definition of $\bar{u}^* \in \R^n$, the second to fourth step simplify the formula, and the last step use the definition of $\bar{u}^* \in \R^n$ again.
Thus, we have
\begin{align*}
 & ~ \frac{\d \|\bar{u}_{\nn}(t)-\bar{u}^*\|_2^2}{\d t} \\
= & ~  2(\bar{u}_{\nn}(t)-\bar{u}^*)^\top \frac{\d \bar{u}_{\nn}(t)}{\d t}\\
= & ~ -2 \kappa^2 (\bar{u}_{\nn}(t)-\bar{u}^*)^\top \bar{H}(t) (\bar{u}_{\nn}(t) - Y) -2\lambda(\bar{u}_{\nn}(t)-\bar{u}^*)^\top \bar{u}_{\nn}(t)\\
% = & ~ -2(u_{\nn}(t)-u^*)^\top H(t) (u_{\nn}(t) - u^*) + 2(u_{\nn}(t)-u^*)^\top H(t) (Y-u^*) -2\lambda(u_{\nn}(t)-u^*)^\top u_{\nn}(t)\\
= & ~ -2 \kappa^2 (\bar{u}_{\nn}(t)-u^*)^\top \bar{H}(t) (\bar{u}_{\nn}(t) - \bar{u}^*) + 2 \kappa^2 (\bar{u}_{\nn}(t)-\bar{u}^*)^\top \bar{H}(0) (Y-\bar{u}^*)\\
& ~  + 2 \kappa^2 (\bar{u}_{\nn}(t)-\bar{u}^*)^\top (\bar{H}(t) - \bar{H}(0)) (Y-\bar{u}^*) -2\lambda(\bar{u}_{\nn}(t)-\bar{u}^*)^\top \bar{u}_{\nn}(t)\\
= & ~ -2 \kappa^2 (\bar{u}_{\nn}(t)-\bar{u}^*)^\top \bar{H}(t) (\bar{u}_{\nn}(t) - \bar{u}^*) + 2\lambda(\bar{u}_{\nn}(t)-\bar{u}^*)^\top \bar{u}^*\\
& ~ +2 \kappa^2 (\bar{u}_{\nn}(t)-\bar{u}^*)^\top (\bar{H}(t) - \bar{H}(0))) (Y-\bar{u}^*) -2\lambda(\bar{u}_{\nn}(t)-\bar{u}^*)^\top \bar{u}_{\nn}(t)\\
= & ~ -2(\bar{u}_{\nn}(t)-\bar{u}^*)^\top ( \kappa^2 \bar{H}(t)+\lambda I) (\bar{u}_{\nn}(t) - \bar{u}^*) +2 \kappa^2 (\bar{u}_{\nn}(t)-\bar{u}^*)^\top (\bar{H}(t) - \bar{H}(0)) (Y-\bar{u}^*)\\
%\le & ~ -2(u(t)-u^*)^\top (H(t)+\lambda I) (u(t) - u^*) + (u(t)-u^*)^\top (H(t)-H^{\cts})(u(t)-u^*)\\
%& ~ + (Y-u^*)^\top (H(t)-H^{\cts})(Y-u^*)\\
\leq & ~ - \frac{1}{2}( \kappa^2 \Lambda_0 + \lambda) \| \bar{u}_{\nn}(t) - \bar{u}^* \|_2^2 + 2 \kappa^2 \| \bar{H}(t) - \bar{H}(0)) \| \| \bar{u}_{\nn}(t) - \bar{u}^* \|_2 \| Y - \bar{u}^* \|_2 % + \Lambda_{\max} ( H(t) - H^{\cts} ) \| u(t) - u^* \|_2^2 + \Lambda_{\max} ( H(t) - H^{\cts} ) \| Y - u^* \|_2^2
\end{align*}
where the first step follows the chain rule, the second step follows Corollary~\ref{cor:nn_gradient_lev}, the third step uses basic linear algebra, the fourth step follows Eq.~\eqref{eq:325_1_lev}, the fifth step simplifies the expression, and the last step follows the assumption $\| \bar{H}(t) - \bar{H}(0)) \| \leq \Lambda_0/4$ and the fact $\|\bar{H}(0))\| \leq \Lambda_0/2$.
\end{proof}

\subsection{Proof sketch}\label{sec:proof_sketch_D}

We introduce a new kernel ridge regression problem with respect to $\bar{H}(0)$ to decouple the prediction perturbation resulted from initialization phase and training phase. Specifically, given any accuracy $\epsilon\in(0,1)$, we divide this proof into following steps:
\begin{enumerate}
	\item Firstly, we bound the prediction perturbation resulted from initialization phase $\|u^*-\bar{u}^*\|_2 \leq \epsilon/2$ by applying the leverage score sampling theory, as shown in Lemma~\ref{lem:u*_minus_bar_u*}.
	\item Then we use the similar idea as section B to bound the prediction perturbation resulted from training phase $\|\bar{u}_{\nn}(T)-\bar{u}^*\|_2 \leq \epsilon/2$ by showing the over-parametrization and convergence property of neural network inductively, as shown in Lemma~\ref{lem:induction_lev} and Corollary~\ref{cor:train_lev}.
	\item Lastly, we combine the results of step 1 and 2 using triangle inequality to show $\|\bar{u}_{\nn}(T) - u^*\|_2 \leq \epsilon$, as shown in Theorem~\ref{thm:equivalence_train_lev}.
\end{enumerate}

\subsection{Main result}\label{sec:main_D}

In this section, we prove Theorem~\ref{thm:equivalence_train_lev} following the above proof sketch.

\subsubsection{Upper bounding $\|u^*-\bar{u}^*\|_2$}
\begin{lemma}\label{lem:u*_minus_bar_u*}
Let $u^*\in\R^n$ and $\bar{u}^*\in\R^n$ be the optimal training data predictors defined in Definition~\ref{def:krr_ntk} and Definition~\ref{def:krr_ntk_lev}. Let $\bar{H}(0)\in\R^{n \times n}$ be defined in Definition~\ref{def:dynamic_kernel_lev}. Let $p(\cdot)$ denotes the probability density function for Gaussian $\N(0,I_d)$. Let $q(\cdot)$ denotes the leverage sampling distribution with respect to $p(\cdot)$, $\bar{H}(0)$ and $\lambda$ defined in Definition~\ref{def:lev_distribution}. Let $\Delta\in(0,1/2)$. If $m\geq \wt{O}(\Delta^{-2} s_{\lambda}(H^{\cts}))$, then we have
\begin{align*}
	\| \bar{u}^* - u^* \|_2 \leq \frac{\lambda\Delta\sqrt{n}}{\Lambda_0+\lambda}
\end{align*}
with probability at least $1-\delta$. Particularly, given arbitrary $\epsilon\in(0,1)$, if $m \geq \wt{O}(\frac{n}{\epsilon\Lambda_0})$ and $\lambda \leq \wt{O}(\frac{1}{\sqrt{m}})$, we have 
\begin{align*}
	\| \bar{u}^* - u^* \|_2 \leq \epsilon/2.
\end{align*}
Here $\wt{O}(\cdot)$ hides $\poly\log(s_{\lambda}(H^{\cts})/\delta)$.
\end{lemma}

\begin{proof}
Note
\begin{align*}
	Y - u^* = \lambda(H^{\cts} + \lambda I_n)^{-1} Y
\end{align*}
and
\begin{align*}
	Y - \bar{u}^* = \lambda(\bar{H}(0) + \lambda I_n)^{-1} Y
\end{align*}
So
\begin{align*}
	\bar{u}^* - u^* = \lambda[(H^{\cts} + \lambda I_n)^{-1} -(\bar{H}(0) + \lambda I_n)^{-1}]Y
\end{align*}
By Lemma~\ref{lem:leverage_score_sampling}, if $m \geq \wt{O}(\Delta^{-2} s_{\lambda}(H^{\cts})$, we have 
\begin{align*}
	(1-\Delta)(H^{\cts} + \lambda I) \preceq \bar{H}(0) + \lambda I \preceq (1+\Delta)(H^{\cts} + \lambda I)
\end{align*}
which implies 
\begin{align*}
	\frac{1}{1+\Delta}(H^{\cts} + \lambda I)^{-1} \preceq (\bar{H}(0) + \lambda I)^{-1} \preceq \frac{1}{1-\Delta}(H^{\cts} + \lambda I)^{-1}
\end{align*}
i.e.,
\begin{align*}
	-\frac{\Delta}{1+\Delta}(H^{\cts} + \lambda I)^{-1} \preceq (\bar{H}(0) + \lambda I)^{-1} -(H^{\cts} + \lambda I)^{-1} \preceq \frac{\Delta}{1-\Delta}(H^{\cts} + \lambda I)^{-1}
\end{align*}
Assume $\Delta\in(0,1/2)$, we have
\begin{align}\label{eq:kmn_lev}
	-{\Delta}(H^{\cts} + \lambda I)^{-1} \preceq (\bar{H}(0) + \lambda I)^{-1} -(H^{\cts} + \lambda I)^{-1} \preceq 2{\Delta}(H^{\cts} + \lambda I)^{-1}
\end{align}
Thus,
\begin{align*}
	\| \bar{u}^* - u^* \|_2 \leq & ~ \lambda \|(H^{\cts} + \lambda I_n)^{-1} -(\bar{H}(0) + \lambda I_n)^{-1}\| \|Y\|_2 \\
	\leq & ~ 2\lambda \Delta \|(H^{\cts} + \lambda I)^{-1} \| \|Y\|_2\\
	\leq & ~ O(\frac{\lambda\Delta\sqrt{n}}{\Lambda_0+\lambda})
\end{align*}
where the first step follows from Cauchy-Schwartz inequality, the second step follows from Eq.~\eqref{eq:kmn_lev}, and the last step follows from the definition of $\Lambda_0$ and $\|Y\|_2=O(\sqrt{n})$.
\end{proof}

\subsubsection{Upper bounding $\|\bar{u}_{\nn}(T)-\bar{u}^*\|_2$}

\begin{lemma}[Bounding kernel perturbation, Parallel to Lemma~\ref{lem:induction}]\label{lem:induction_lev}
Given training data $X\in\R^{n\times d}$, $Y\in\R^n$ and a test data $x_{\test}\in\R^d$. Let $T > 0$ denotes the total number of iterations, $m >0 $ denotes the width of the network,  $\epsilon_{\train}$ denotes a fixed training error threshold, $\delta > 0$ denotes the failure probability. Let $\bar{u}_{\nn}(t) \in \R^n$ be the training data predictors defined in Definition~\ref{def:f_nn_lev}. Let $\kappa\in(0,1)$ be the corresponding multiplier. Let $\bar{\k}_{t}(x_{\test},X) \in \R^n,~\bar{H}(t) \in \R^{n \times n},~\Lambda_0 > 0$ be the kernel related quantities defined in Definition~\ref{def:dynamic_kernel_lev}. Let $\bar{u}^* \in \R^n$ be defined as in Eq.~\eqref{eq:def_u_*_lev}. Let $\lambda > 0$ be the regularization parameter. Let $W(t) = [w_1(t),\cdots,w_m(t)]\in\R^{d\times m}$ be the parameters of the neural network defined in Definition~\ref{def:f_nn_lev}.
Given any accuracy $\epsilon\in(0,1/10)$ and failure probability $\delta \in (0,1/10)$. If $\kappa = 1$, $T=\wt{O}(\frac{1}{\Lambda_0})$, $\epsilon_{\train} = \epsilon/2$, network width $m \geq \wt{O}(\frac{n^4d}{\lambda_0^4\epsilon})$ and regularization parameter $\lambda \leq \wt{O}(\frac{1}{\sqrt{m}})$,
% If $\kappa=\wt{O}(\frac{\epsilon\Lambda_0}{n}) ,~T=\wt{O}(\frac{1}{\kappa^2(\Lambda_0+\lambda)})$, $\epsilon_{\train} = \wt{O}(\|u_{\nn}(0)-u^*\|_2)$, $m \geq\wt{O}(\frac{n^{10} d}{\epsilon^6 \Lambda_0^{10}})$ and $\lambda=\wt{O}(\frac{1}{\sqrt{m}})$, with probability $1-\delta$, 
then there exist $\epsilon_W,~\epsilon_H',~\epsilon_K'>0$ that are independent of $t$, such that the following hold for all $0 \leq t \le T$:
\begin{itemize}
    \item 1. $\| w_r(0) - w_r(t) \|_2 \leq \epsilon_W $, $\forall r \in [m]$
    \item 2. $\| \bar{H}(0) - \bar{H}(t) \|_2 \leq \epsilon_H'$ 
    \item 3. $\| \bar{u}_{\nn}(t) - \bar{u}^* \|_2^2 \leq \max\{\exp(-(\kappa^2\Lambda_0 + \lambda) t/4) \cdot \| \bar{u}_{\nn}(0) - \bar{u}^* \|_2^2, ~ \epsilon_{\train}^2\}$
    % \item 4. $\| \bar{\k}_0( x_{\test} , X )- \bar{\k}_{t} (x_{\test}, X) \|_2 \leq \epsilon_K'$
\end{itemize}
% Further, $\epsilon_W \leq \wt{O}(\frac{\epsilon \lambda_0^2}{n^2})$,  $\epsilon_H' \leq \wt{O}(\frac{\epsilon \lambda_0^2}{n})$ and  $\epsilon_K' \leq \wt{O}(\frac{\epsilon \lambda_0^2}{n^{1.5}})$. 
Here $\wt{O}(\cdot)$ hides the  $\poly\log( n / ( \epsilon  \delta  \Lambda_0 ) )$.
\end{lemma}

%\subsubsection{Proof of Lemma~\ref{lem:induction}}

% \subsubsection{Random initialization}
We first state the following concentration result for the random initialization that can help us prove the lemma.

\begin{lemma}[Random initialization result]\label{lem:random_init_lev}
Assume initial value $w_r(0) \in \R^d ,~r=1,\cdots,m$ are drawn independently according to leverage score sampling distribution $q(\cdot)$ defined in \eqref{eq:lev_dis_lev}, then with probability $1-\delta$ we have
\begin{align}
	\|w_r(0)\|_2 \leq & ~ 2\sqrt{d} + 2\sqrt{\log{(mc_2/\delta)}}:=\alpha_{w,0}\label{eq:3322_1_lev} 
	% \|\bar{H}(0)- \bar{H}^{\cts} \| \leq & ~ 4n ( \log(n/\delta) / m )^{1/2}=O(1/\sqrt{m}) := \epsilon_{H,0} \label{eq:3322_2_lev}\\
	% \|\bar{\k}_{0}( x_{\test} , X ) - \bar{\k}_{\ntk} ( x_{\test} , X )\|_2 \leq & ~ ( 2n \log{(2n/\delta)} / m )^{1/2}=O(1/\sqrt{m}) : = \epsilon_{\k,0} \label{eq:3322_3_lev}
\end{align}
hold for all $r\in[m]$, where $c_2=O(n)$.
\end{lemma}
\begin{proof}
By lemma~\ref{lem:chi_square_tail}, if $w_r(0)\sim\N(0,I_n)$, then with probability at least $1-\delta$,
\begin{align*}
	\| w_r(0) \|_2 \leq 2\sqrt{d} + 2\sqrt{\log(m/\delta)}
\end{align*}
holds for all $r\in[m]$. By Lemma~\ref{lem:property_lev}, we have $q(w) \leq c_2 p(w)$ holds for all $w\in\R^d$, where $c_2 = O(1/\Lambda_0)$ and $p(\cdot)$ is the probability density function of $\N(0,I_d)$. Thus, if $w_r(0)\sim q$, we have with probability at least $1-\delta$,
\begin{align*}
	\| w_r(0) \|_2 \leq 2\sqrt{d} + 2\sqrt{\log(mc_2/\delta)}
\end{align*}
holds for all $r\in[m]$.

% Using Lemma~\ref{lem:lemma_4.1_in_sy19} in \cite{sy19}, we have
% %Using Lemma 4.1 in \cite{sy19}, we have
% \begin{align*}
%     \| H(0) - H^{\cts} \| \leq \epsilon_H'' = 4 n ( \log{(n/\delta)} / m )^{1/2}
% \end{align*}
% holds with probability at least $1-\delta$.\\
% Note by definition,
% \begin{align*}
%     \E[\k_0 ( x_{\test}, x_i )] =  \k_{\ntk} ( x_{\test}, x_i )
% \end{align*}
% holds for any training data $x_i$. By Hoeffding inequality, we have for any $t>0$,
% \begin{align*}
%     \Pr[|\k_0 ( x_{\test}, x_i ) - \k_{\ntk} ( x_{\test}, x_i )|\ge t] \le 2\exp{(-mt^2/2)}.
% \end{align*}
% Setting $t=(\frac{2}{m}\log{(2n/\delta)})^{1/2}$, we can apply union bound on all training data $x_i$ to get with probability at least $1-\delta$, for all $i\in[n]$,
% \begin{align*}
%     |\k_0 ( x_{\test}, x_i ) - \k_{\ntk} ( x_{\test}, x_i )| \le (2 \log(2n/\delta) / m)^{1/2}.
% \end{align*}
% Thus, we have
% \begin{align}
%     \| \k_0 ( x_{\test}, X ) - \k_{\ntk} ( x_{\test}, X ) \|_2 \le ( 2n\log(2n/\delta) / m )^{1/2}
% \end{align}
% holds with probability at least $1-\delta$.\\
% Using union bound over above three events, we finish the proof.
\end{proof}

Now conditioning on Eq.~\eqref{eq:3322_1},~\eqref{eq:3322_2},~\eqref{eq:3322_3} holds, We show all the four conclusions in Lemma~\ref{lem:induction} holds using induction.

We define the following quantity:
\begin{align}
	\epsilon_W :=  & ~ \frac{ \sqrt{n} }{ \sqrt{m} } \max\{4\| \bar{u}_{\nn}(0) - \bar{u}^* \|_2/(\kappa^2\Lambda_0+\lambda), \epsilon_{\train} \cdot T\} \notag \\
	& ~ + \Big( \frac{ \sqrt{n} }{ \sqrt{m} } \| Y-\bar{u}^* \|_2 + 2\lambda \alpha_{w,0} \Big) \cdot T\label{eq:def_epsilon_W_lev}\\
    \epsilon_H' := & ~ 2n\epsilon_W\notag \\
    \epsilon_K := & ~ 2\sqrt{n}\epsilon_W\notag
\end{align}{}
which are independent of $t$. Here $\alpha_{w,0}$ are defined in Eq.~\eqref{eq:3322_1_lev}.

Note the base case when $t=0$ trivially holds. Now assuming Lemma~\ref{lem:induction_lev} holds before time $t\in[0,T]$, we argue that it also holds at time $t$. To do so, Lemmas~\ref{lem:general_hypothesis_1_lev},~\ref{lem:general_hypothesis_2_lev},~\ref{lem:general_hypothesis_3_lev} argue these conclusions one by one.

% \subsubsection{Induction hypothesis 1}

\begin{lemma}[Conclusion 1]\label{lem:general_hypothesis_1_lev}
% Let $c$ denote a fixed constant. 
If for any $\tau < t$, we have
% $m \geq  ??? \cdot 1/\epsilon_W$ and 
\begin{align*}
    \| \bar{u}_{\nn}(\tau) - \bar{u}^* \|_2^2 \leq ~ \max\{\exp(-(\kappa^2 \Lambda_0 + \lambda) \tau/4) \cdot \| \bar{u}_{\nn}(0) - \bar{u}^* \|_2^2,~\epsilon_{\train}^2\}
\end{align*}
and
% \begin{align*}
%     \| w_r(0) - w_r(\tau) \|_2 \leq 1
% \end{align*}
\begin{align*}
     \| w_r(0) - w_r(\tau) \|_2 \leq \epsilon_W \leq 1
\end{align*}
and
% \begin{align*}
% 	\| w_r(0) \|_2 \leq ~ \sqrt{d} + \sqrt{\log(m/\delta)} ~ \text{for all}~r\in[m]% \text{ for some constant } c
% \end{align*}
\begin{align*}
	\| w_r(0) \|_2 \leq ~ \alpha_{w,0} ~ \text{for all}~r\in[m]% \text{ for some constant } c
\end{align*}
hold,
then
\begin{align*}
    \| w_r(0) - w_r(t) \|_2 \leq \epsilon_W 
    % := & ~ \frac{ \sqrt{n} }{ \sqrt{m} } \max\{4\| u_{\nn}(0) - u^* \|_2/(\Lambda_0+\lambda), \epsilon_{\train} \cdot T\} + (\frac{ \sqrt{n} }{ \sqrt{m} } \| Y-u^* \|_2\\
    % & ~ + \lambda (\sqrt{d} + \sqrt{\log(m/\delta)} + 1) )\cdot T.
\end{align*}

\end{lemma}
% \begin{remark}
% By lemma~\ref{lem:chi_square_tail}, with probability at least $1-\delta$,
% \begin{align*}
% 	\| w_r(0) \|_2 \leq \sqrt{d} + \sqrt{\log(m/\delta)}
% \end{align*}
% holds for all $r\in[m]$.
% \end{remark}
\begin{proof}
Recall the gradient flow as Eq.~\eqref{eq:323_1_lev}
\begin{align}\label{eq:332_2_lev}
    \frac{ \d w_r( \tau ) }{ \d \tau } = ~ \sum_{i=1}^n \frac{1}{\sqrt{m}} a_r ( y_i - \bar{u}_{\nn}(\tau)_i ) x_i \sigma'( w_r(\tau)^\top x_i ) - \lambda w_r(\tau)
\end{align}
So we have
\begin{align}\label{eq:332_1_lev}
	\Big\| \frac{ \d w_r( \tau ) }{ \d \tau } \Big\|_2 
	= & ~ \left\| \sum_{i=1}^n \frac{1}{\sqrt{m}} a_r ( y_i - \bar{u}_{\nn}(\tau)_i ) x_i \sigma'( w_r(\tau)^\top x_i ) - \lambda w_r(\tau) \right\|_2 \notag \\
	\leq & ~ \frac{1}{\sqrt{m}} \sum_{i=1}^n |y_i-\bar{u}_{\nn}(\tau)_i| + \lambda \| w_r(\tau) \|_2 \notag\\
	\leq & ~ \frac{ \sqrt{n} }{ \sqrt{m} } \| Y-\bar{u}_{\nn}(\tau) \|_2 + \lambda \| w_r (\tau) \|_2 \notag \\
	\leq & ~ \frac{ \sqrt{n} }{ \sqrt{m} } \| Y-\bar{u}_{\nn}(\tau) \|_2 + \lambda (\| w_r(0) \|_2 + \| w_r(\tau) - w_r(0) \|_2) \notag\\ 
	\leq & ~ \frac{ \sqrt{n} }{ \sqrt{m} } \| Y-\bar{u}_{\nn}(\tau) \|_2 + \lambda (\alpha_{W,0} + 1) \notag\\ 
	\leq & ~ \frac{ \sqrt{n} }{ \sqrt{m} } \| Y-\bar{u}_{\nn}(\tau) \|_2 + 2\lambda \alpha_{W,0} \notag\\
	\leq & ~ \frac{ \sqrt{n} }{ \sqrt{m} } (\| Y-\bar{u}^*\|_2 + \| \bar{u}_{\nn}(\tau) - \bar{u}^*\|_2) + 2\lambda\alpha_{W,0} \notag\\
	= & ~ \frac{ \sqrt{n} }{ \sqrt{m} } \| \bar{u}_{\nn}(\tau) - \bar{u}^*\|_2 \notag \\
	& ~ + \frac{ \sqrt{n} }{ \sqrt{m} } \| Y-\bar{u}^* \|_2 + 2\lambda \alpha_{W,0} \notag \\
	\leq & ~ \frac{ \sqrt{n} }{ \sqrt{m} } \max\{e^{-(\kappa^2 \Lambda_0+\lambda)\tau/8} \| \bar{u}_{\nn}(0) - \bar{u}^* \|_2, \epsilon_{\train}\} \notag \\
	& ~ + \frac{ \sqrt{n} }{ \sqrt{m} } \| Y-\bar{u}^* \|_2 + 2\lambda\alpha_{W,0} ,
\end{align}
where the first step follows from Eq.~\eqref{eq:332_2_lev}, the second step follows from triangle inequality, the third step follows from Cauchy-Schwartz inequality, the forth step follows from triangle inequality, the fifth step follows from condition $\| w_r(0) - w_r(\tau) \|_2 \leq 1,~\| w_r(0) \|_2 \leq \alpha_{W,0}$, the seventh step follows from triangle inequality, the last step follows from $\| \bar{u}_{\nn}(\tau) - \bar{u}^* \|_2^2 \leq \max\{\exp(-(\kappa^2\Lambda_0 + \lambda) \tau/4) \cdot \| \bar{u}_{\nn}(0) - \bar{u}^* \|_2^2,~\epsilon_{\train}^2\}$.

Thus, for any $t \le T$,
\begin{align*}
	\| w_r(0) - w_r(t) \|_2 \leq & ~ \int_0^t \Big\| \frac{ \d w_r( \tau ) }{ \d \tau } \Big\|_2 d\tau \\
	\leq & ~ \frac{ \sqrt{n} }{ \sqrt{m} } \max\{4\| \bar{u}_{\nn}(0) - \bar{u}^* \|_2/(\kappa^2\Lambda_0+\lambda), \epsilon_{\train}\cdot T \} \\
	& ~ + \Big( \frac{ \sqrt{n} }{ \sqrt{m} } \| Y-\bar{u}^* \|_2 + 2\lambda \alpha_{W,0} \Big) \cdot T\\
    = & ~ \epsilon_W
\end{align*}
where the first step follows triangle inequality, the second step follows Eq.~\eqref{eq:332_1_lev}, and the last step follows the definition of $\epsilon_W$ as Eq.~\eqref{eq:def_epsilon_W_lev}.

\end{proof}

% \subsubsection{Induction hypothesis 2}

\begin{lemma}[Conclusion 2]\label{lem:general_hypothesis_2_lev}
% Fix $\epsilon_W\in(0,1)$ independent of $t$. 
If $\forall r \in [m]$,
\begin{align*}
    \| w_r(0) - w_r(t) \|_2 \leq \epsilon_W < 1,
\end{align*}
then
\begin{align*}
    \| H(0) - H(t) \|_F \leq 2n \epsilon_W
\end{align*}
holds with probability $1-n^2 \cdot \exp{(-m\epsilon_Wc_1/10)}$, where $c_1=O(1/n)$.
\end{lemma}
\begin{proof}
Directly applying Lemma~\ref{lem:lemma_4.2_in_sy19_lev}, we finish the proof.
%Lemma 4.2 in \cite{sy19}, we finish the proof.
\end{proof}

\begin{lemma}[perturbed $w$]\label{lem:lemma_4.2_in_sy19_lev}
Let $R \in (0,1)$. If $\wt{w}_1, \cdots, \wt{w}_m$ are i.i.d. generated from the leverage score sampling distribution $q(\cdot)$ as in \eqref{eq:lev_dis_lev}. Let $p(\cdot)$ denotes the standard Gaussian distribution $\N(0,I_d)$. For any set of weight vectors $w_1, \cdots, w_m \in \R^d$ that satisfy for any $r\in [m]$, $\| \wt{w}_r - w_r \|_2 \leq R$, then the $H : \R^{m \times d} \rightarrow \R^{n \times n}$ defined
\begin{align*}
    \bar{H}(w)_{i,j} =  \frac{1}{m} x_i^\top x_j \sum_{r=1}^m {\bf 1}_{ w_r^\top x_i \geq 0, w_r^\top x_j \geq 0 } \frac{p(\wt{w}_r)}{q(\wt{w}_r)}.
\end{align*}
Then we have
\begin{align*}
\| \bar{H} (w) - \bar{H}(\wt{w}) \|_F < 2 n R,
\end{align*}
holds with probability at least $1-n^2 \cdot \exp(-m R c_1 /10)$, where $c_1 = O(1/n)$.
\end{lemma}
\begin{proof}

The random variable we care is

\begin{align*}
& ~ \sum_{i=1}^n \sum_{j=1}^n | \bar{H}(\wt{w})_{i,j} - \bar{H}(w)_{i,j} |^2 \\
\leq & ~ \frac{1}{m^2} \sum_{i=1}^n \sum_{j=1}^n \left( \sum_{r=1}^m {\bf 1}_{ \wt{w}_r^\top x_i \geq 0, \wt{w}_r^\top x_j \geq 0} - {\bf 1}_{ w_r^\top x_i \geq 0 , w_r^\top x_j \geq 0 } \frac{p(\wt{w}_r)}{q(\wt{w}_r)}\right)^2 \\
= & ~ \frac{1}{m^2} \sum_{i=1}^n \sum_{j=1}^n  \Big( \sum_{r=1}^m s_{r,i,j} \Big)^2 ,
\end{align*}

where the last step follows from for each $r,i,j$, we define
\begin{align*}
s_{r,i,j} :=  ({\bf 1}_{ \wt{w}_r^\top x_i \geq 0, \wt{w}_r^\top x_j \geq 0} - {\bf 1}_{ w_r^\top x_i \geq 0 , w_r^\top x_j \geq 0 })\frac{p(\wt{w}_r)}{q(\wt{w}_r)}.
\end{align*} 

We consider $i,j$ are fixed. We simplify $s_{r,i,j}$ to $s_r$.

Then $s_r$ is a random variable that only depends on $\wt{w}_r$.
Since $\{\wt{w}_r\}_{r=1}^m$ are independent,
$\{s_r\}_{r=1}^m$ are also mutually independent.

Now we define the event
\begin{align*}
A_{i,r} = \left\{ \exists u : \| u - \wt{w}_r \|_2 \leq R, {\bf 1}_{ x_i^\top \wt{w}_r \geq 0 } \neq {\bf 1}_{ x_i^\top u \geq 0 } \right\}.
\end{align*}
Then we have
\begin{align}\label{eq:Air_bound}
\Pr_{ \wt{w}_r \sim {\cal N}(0,I) }[ A_{i,r} ] = \Pr_{ z \sim \N(0,1) } [ | z | < R ] \leq \frac{ 2 R }{ \sqrt{2\pi} }.
\end{align}
where the last step follows from the anti-concentration inequality of Gaussian (Lemma~\ref{lem:anti_gaussian}).

If   $\neg A_{i,r}$ and $\neg A_{j,r}$ happen,
then 
\begin{align*}
\left| {\bf 1}_{ \wt{w}_r^\top x_i \geq 0, \wt{w}_r^\top x_j \geq 0} - {\bf 1}_{ w_r^\top x_i \geq 0 , w_r^\top x_j \geq 0 } \right|=0.
\end{align*}
If   $A_{i,r}$ or $A_{j,r}$ happen,
then 
\begin{align*}
\left| {\bf 1}_{ \wt{w}_r^\top x_i \geq 0, \wt{w}_r^\top x_j \geq 0} - {\bf 1}_{ w_r^\top x_i \geq 0 , w_r^\top x_j \geq 0 } \right|\leq 1.
\end{align*}
So we have 
\begin{align*}
 \E_{\wt{w}_r\sim q}[s_r] \leq & ~ \E_{\wt{w}_r\sim q} \left[ {\bf 1}_{A_{i,r}\vee A_{j,r}}\frac{p(\wt{w}_r)}{q(\wt{w}_r)} \right] \\
 = & ~ \E_{\wt{w}_r\sim p} \left[ {\bf 1}_{A_{i,r}\vee A_{j,r}} \right]\\
 \leq & ~ \Pr_{ \wt{w}_r \sim \N(0,I_n) }[A_{i,r}]+\Pr_{ \wt{w}_r \sim \N(0,I_n) }[A_{j,r}] \\
 \leq & ~ \frac {4 R }{\sqrt{2\pi}}\\
 \leq & ~  2R,
\end{align*}
and 
\begin{align*}
    \E_{\wt{w}_r\sim q} \left[ \left(s_r-\E_{\wt{w}_r\sim q}[s_r] \right)^2 \right]
    = & ~ \E_{\wt{w}_r\sim q}[s_r^2]-\E_{\wt{w}_r\sim q}[s_r]^2 \\
    \leq & ~ \E_{\wt{w}_r\sim q}[s_r^2]\\
    \leq & ~\E_{\wt{w}_r\sim q} \left[ \left( {\bf 1}_{A_{i,r}\vee A_{j,r}} \frac{p(\wt{w}_r)}{q(\wt{w}_r)}\right)^2 \right] \\
    = & ~ \E_{\wt{w}_r\sim p} \left[  {\bf 1}_{A_{i,r}\vee A_{j,r}} \frac{p(\wt{w}_r)}{q(\wt{w}_r)}\right] \\
    \leq & \frac{2R}{c_1}
\end{align*}
where the last step follows from Lemma~\ref{lem:property_lev} and $c_1= O(n)$. We also have $|s_r|\leq 1/c_1$.
So we can apply Bernstein inequality (Lemma~\ref{lem:bernstein}) to get for all $t>0$,
\begin{align*}
    \Pr \left[\sum_{r=1}^m s_r\geq 2m R +mt \right]
    \leq & ~ \Pr \left[\sum_{r=1}^m (s_r-\E[s_r])\geq mt \right]\\
    \leq & ~ \exp \left( - \frac{ m^2t^2/2 }{ 2m R/c_1  + mt/3c_1 } \right).
\end{align*}
Choosing $t = R$, we get
\begin{align*}
    \Pr \left[\sum_{r=1}^m s_r\geq 3mR  \right]
    \leq & ~ \exp \left( -\frac{ m^2 R^2 /2 }{ 2 m  R/c_1 +  m R /3c_1 } \right) \\
     \leq & ~ \exp \left( - m R c_1 / 10 \right) .
\end{align*}
Plugging back, we complete the proof.
% Thus, we can have
% \begin{align*}
% \Pr \left[ \frac{1}{m} \sum_{r=1}^m s_r \geq 2 (\alpha_i + \alpha_j) R \right] \leq \exp( - m (\alpha_i + \alpha_j)R /10 ).
% \end{align*}
% We sum over all $(i,j)$ pairs
% \begin{align*}
% \left( 4 \sum_{i=1}^n \sum_{j=1}^n (\alpha_i + \alpha_j)^2 R^2 \right)^{1/2}
% \leq & ~ \left( 8 \sum_{i=1}^n \sum_{j=1}^n (\alpha_i^2 + \alpha_j^2) R^2 \right)^{1/2} \\
% % = & ~ ( 8 \| \alpha \|_2^4 R^2 )^{1/2} \\
% = & ~ \sqrt{8} \| \alpha\|_2^2 R \\
% \leq & \frac{3n^3R}{\Lambda_0+\lambda}
% \end{align*}
% where the last step follows from $\alpha_i = n/(\Lambda_0+\lambda)$ for all $i\in[n]$.
\end{proof}

% \subsubsection{Induction hypothesis 3}

\begin{lemma}[Conclusion 3]\label{lem:general_hypothesis_3_lev}
Fix $\epsilon_H'>0$ independent of $t$. If for all $\tau < t$
% \begin{align*}
%     \| H(0) - H^{\cts} \| \leq 4n ( \log{(n/\delta)} / m )^{1/2} \leq \Lambda_0/4
% \end{align*}
% and
\begin{align*}
    \| H(0) - H(\tau) \| \leq \epsilon_H' \leq \Lambda_0/4
\end{align*}
% and
% \begin{align}\label{eq:335_2_lev}
% 	4 n ( \log ( n / \delta ) / m )^{1/2} \leq \frac{\epsilon_{\train}}{8\kappa^2\|Y-u^*\|_2}(\kappa^2\Lambda_0+\lambda)
% \end{align}
and 
\begin{align}\label{eq:335_3_lev}
    \epsilon_H' \leq \frac{\epsilon_{\train}}{8\kappa^2\|Y-u^*\|_2}(\kappa^2\Lambda_0+\lambda)
\end{align}
% and 
% \begin{align*}
%     \| u_{\nn}(t) - u^* \|_2^2 \leq \max\{\exp(-(\Lambda_0 + \lambda) (t) ) \cdot \| u_{\nn}(0) - u^* \|_2^2, ~ \epsilon^2\}~\text{holds for all}~\tau < t,
% \end{align*}
% and 
% \begin{align*}
%     \| u_{\nn}(t) - u^* \|_2^2 \leq \exp(-(\Lambda_0 + \lambda) (t) ) \cdot \| u_{\nn}(0) - u^* \|_2^2~\text{holds for all}~\tau < t,
% \end{align*}
then we have
\begin{align*}
    \| u_{\nn}(t) - u^* \|_2^2 \leq \max\{\exp(-(\kappa^2\Lambda_0 + \lambda)t/2 ) \cdot \| u_{\nn}(0) - u^* \|_2^2, ~ \epsilon_{\train}^2\}.
\end{align*}
\end{lemma}

\begin{proof}
% By triangle inequality we have
% \begin{align}\label{eq:335_1_lev}
%     \| H(\tau) - H^{\cts} \| 
%     \leq & ~ \| H(0) - H(\tau) \| + \| H(0) - H^{\cts} \| \notag \\
%     \leq & ~ \epsilon_H' + 4n( \log{(n/\delta)} / m)^{1/2} \notag\\
%     \leq & ~ \Lambda_0/2
% \end{align}
% holds for all $\tau < t$. Denote $ \epsilon_H = \epsilon_H' + 4n( \log{(n/\delta)} / m )^{1/2}$, we have $\| H(\tau) - H^{\cts} \| \leq\epsilon_H \leq \Lambda_0/2$, which 
Note $\| \bar{H}(\tau) - \bar{H}(0) \| \leq \epsilon_H \leq \Lambda_0/4$.
By Lemma~\ref{lem:linear_converge_nn_lev}, for any $\tau < t$, we have
\begin{align}\label{eq:induction_linear_convergence_lev}
    \frac{ \d \| \bar{u}_{\nn}(\tau) - \bar{u}^* \|_2^2 }{ \d \tau } 
    \leq & ~ - \frac{1}{2}( \kappa^2 \Lambda_0 + \lambda ) \cdot \| \bar{u}_{\nn} (\tau) - \bar{u}^* \|_2^2 + 2 \kappa^2 \| \bar{H}(\tau) - \bar{H}(0) \| \cdot \| \bar{u}_{\nn} (\tau) - \bar{u}^* \|_2 \cdot \| Y - \bar{u}^* \|_2 \notag\\
    \leq & ~ - \frac{1}{2}( \kappa^2 \Lambda_0 + \lambda ) \cdot \| \bar{u}_{\nn} (\tau) - \bar{u}^* \|_2^2 + 2 \kappa^2 \epsilon_H' \cdot \| \bar{u}_{\nn} (\tau) - \bar{u}^* \|_2 \cdot \|Y-\bar{u}^*\|_2 
    % \\
    % \leq & ~ - ( \Lambda_0 + \lambda ) \cdot \| u_{\nn} (t) - u^* \|_2^2 ,
\end{align}
where the first step follows from Lemma~\ref{lem:linear_converge_nn_lev}, the second step follows from definition of $\epsilon_H'$. 

Now let us discuss two cases:

{\bf Case 1.} If for all $\tau < t$, $\|u_{\nn}(\tau) - u^*\|_2 \geq \epsilon_{\train}$ always holds, we want to argue that 
	\begin{align*}
		\| \bar{u}_{\nn}(t) - \bar{u}^* \|_2^2 \leq \exp(-(\kappa^2 \Lambda_0 + \lambda) t/4) \cdot \| \bar{u}_{\nn}(0) - \bar{u}^* \|_2.
	\end{align*}
	Note by assumption~\eqref{eq:335_3_lev}, we have 
	\begin{align*}
		\epsilon_H' \leq \frac{\epsilon_{\train}}{8 \kappa^2 \|Y-u^*\|_2}(\kappa^2 \Lambda_0+\lambda)
	\end{align*}
	implies
	\begin{align*}
	2 \kappa^2 \epsilon_H \cdot \|Y-\bar{u}^*\|_2 \leq (\kappa^2 \Lambda_0 + \lambda)/4\cdot \|\bar{u}_{\nn}(\tau)-\bar{u}^*\|_2
	\end{align*}
	holds for any $\tau < t$. Thus, plugging into~\eqref{eq:induction_linear_convergence_lev},
	\begin{align*}
	    \frac{ \d \| \bar{u}_{\nn}(\tau) - \bar{u}^* \|_2^2 }{ \d \tau } 
	    \leq  ~ - ( \kappa^2 \Lambda_0 + \lambda )/4 \cdot \| \bar{u}_{\nn} (\tau) - \bar{u}^* \|_2^2,
	    % \\
	    % \leq & ~ - ( \Lambda_0 + \lambda ) \cdot \| u_{\nn} (t) - u^* \|_2^2 ,
	\end{align*}
	holds for all $\tau < t$, which implies
	\begin{align*}
		\| \bar{u}_{\nn}(t) - \bar{u}^* \|_2^2 \leq \exp{(-(\kappa^2 \Lambda_0 + \lambda)t/4)} \cdot \| \bar{u}_{\nn}(0) - \bar{u}^* \|_2^2.
	\end{align*}

{\bf Case 2.} If there exist $\bar{\tau} < t$, such that $\|\bar{u}_{\nn}(\bar{\tau}) - \bar{u}^*\|_2 < \epsilon_{\train}$, we want to argue that $\|\bar{u}_{\nn}(t) - \bar{u}^*\|_2 < \epsilon_{\train}$.
	Note by assumption~\eqref{eq:335_3_lev}, we have 
	\begin{align*}
		\epsilon_H' \leq \frac{\epsilon_{\train}}{8\kappa^2\|Y-\bar{u}^*\|_2}(\kappa^2\Lambda_0+\lambda)
	\end{align*}
	implies
	\begin{align*}
	4\kappa^2 \epsilon_H' \cdot \|\bar{u}_{\nn}(\bar{\tau}) - \bar{u}^*\|_2 \cdot \|Y-\bar{u}^*\|_2 \leq (\kappa^2 \Lambda_0 + \lambda) \cdot \epsilon_{\train}^2.
	\end{align*}
	Thus, plugging into~\eqref{eq:induction_linear_convergence_lev},
	\begin{align*}
	    \frac{ \d (\| \bar{u}_{\nn}(\tau) - \bar{u}^* \|_2^2 -\epsilon_{\train}^2)}{ \d \tau } 
	    \leq  ~ - ( \kappa^2 \Lambda_0 + \lambda )/2 \cdot (\| \bar{u}_{\nn} (\tau) - \bar{u}^* \|_2^2 - \epsilon_{\train}^2)
	    % \\
	    % \leq & ~ - ( \Lambda_0 + \lambda ) \cdot \| u_{\nn} (t) - u^* \|_2^2 ,
	\end{align*}
	holds for $\tau = \bar{\tau}$, which implies $ e^{( \kappa^2\Lambda_0 + \lambda )\tau/2} (\| \bar{u}_{\nn}(\tau) - \bar{u}^* \|_2^2 -\epsilon_{\train}^2) $ is non-increasing at $\tau = \bar{\tau}$. Since $\| \bar{u}_{\nn}(\bar{\tau}) - \bar{u}^* \|_2^2 - \epsilon_{\train}^2 < 0$, by induction, $ e^{( \kappa^2\Lambda_0 + \lambda )\tau/2} (\| \bar{u}_{\nn}(\tau) - \bar{u}^* \|_2^2 -\epsilon_{\train}^2) $ being non-increasing and $\| \bar{u}_{\nn}(\tau) - \bar{u}^* \|_2^2 - \epsilon_{\train}^2 < 0$ holds for all $\bar{\tau} \leq \tau \leq t$, which implies
	\begin{align*}
	\|\bar{u}_{\nn}(t) - \bar{u}^*\|_2 < \epsilon_{\train}.
	\end{align*}

Combine above two cases, we conclude
\begin{align*}
\| \bar{u}_{\nn}(t) - \bar{u}^* \|_2^2 \leq \max\{\exp(-(\kappa^2\Lambda_0 + \lambda)t/2 ) \cdot \| \bar{u}_{\nn}(0) - \bar{u}^* \|_2^2, ~ \epsilon_{\train}^2\}.
\end{align*}
\end{proof}

Now we summarize all the conditions need to be satisfied so that the induction works as in Table~\ref{tab:condition_only_lev}.

\begin{table}[htb]\caption{Summary of conditions for induction}\label{tab:condition_only_lev}

\centering
% {\footnotesize
{
  \begin{tabular}{| l | l | l |}
    \hline
    {\bf No.} & {\bf Condition} & {\bf Place} \\ \hline
    1 & $\epsilon_W \leq 1$ & Lem.~\ref{lem:general_hypothesis_1_lev} \\ \hline
    2 & $ \epsilon_H' \leq \Lambda_0/4$ & Lem.~\ref{lem:general_hypothesis_3_lev} \\ \hline 
    % 3 & $4n(\frac{\log{(n/\delta)}}{m})^{1/2} \leq \Lambda_0/4$ & Lem.~\ref{lem:general_hypothesis_3_lev} \\ \hline
    % 4 & $4n(\frac{\log{(n/\delta)}}{m})^{1/2} \leq \frac{\epsilon_{\train}}{8\kappa^2\|Y-u^*\|_2}(\kappa^2\Lambda_0+\lambda)$ & Lem.~\ref{lem:general_hypothesis_3_lev} \\ \hline
    3 & $\epsilon_H' \leq \frac{\epsilon_{\train}}{8\kappa^2\|Y-u^*\|_2}(\kappa^2\Lambda_0+\lambda)$ & Lem.~\ref{lem:general_hypothesis_3_lev} \\
    \hline
  \end{tabular}
}
\end{table}

Compare Table~\ref{tab:condition_only} and Table~\ref{tab:condition_only_lev}, we can see by picking the same value for the parameters as in Theorem~\ref{thm:main_train_equivalence}, we have the induction holds, which completes the proof.

As a direct corollary, we have 
\begin{corollary}\label{cor:train_lev}
Given any accuracy $\epsilon\in(0,1/10)$ and failure probability $\delta \in (0,1/10)$. If $\kappa = 1$, $T=\wt{O}(\frac{1}{\Lambda_0})$, network width $m \geq \wt{O}(\frac{n^4d}{\lambda_0^4\epsilon})$ and regularization parameter $\lambda \leq \wt{O}(\frac{1}{\sqrt{m}})$, then with probability at least $1-\delta$, 
\begin{align*}
	\|\bar{u}_{\nn}(T) - \bar{u}^*\|_2 \leq \epsilon/2.
\end{align*}
Here $\wt{O}(\cdot)$ hides the  $\poly\log( n / ( \epsilon  \delta  \Lambda_0 ) )$.
\end{corollary}
\begin{proof}
	By choosing $\epsilon_{\train} = \epsilon/2$ in Lemma~\ref{lem:induction_lev}, the induction shows
	\begin{align*}
		\| \bar{u}_{\nn}(t) - \bar{u}^* \|_2^2 \leq \max\{\exp(-(\kappa^2\Lambda_0 + \lambda) t/4) \cdot \| \bar{u}_{\nn}(0) - \bar{u}^* \|_2^2, ~ \epsilon^2/4\}
	\end{align*}
	holds for all $t\leq T$. By picking $T=\wt{O}(\frac{1}{\Lambda_0})$, we have 
	\begin{align*}
		\exp(-(\kappa^2\Lambda_0 + \lambda) T/4) \cdot \| \bar{u}_{\nn}(0) - \bar{u}^* \|_2^2 \leq \epsilon^2/4
	\end{align*}
	which implies $\|\bar{u}_{\nn}(T) - \bar{u}^*\|_2^2 \leq \max\{\epsilon^2/4, \epsilon^2/4\} = \epsilon^2/4$.
\end{proof}

\subsubsection{Main result for equivalence with leverage score sampling initialization}

\begin{theorem}[Equivalence between training reweighed neural net with regularization under leverage score initialization and kernel ridge regression for training data prediction, restatement of Theorem~\ref{thm:equivalence_train_lev_intro}]\label{thm:equivalence_train_lev}
Given training data matrix $X \in \R^{n \times d}$ and corresponding label vector $Y \in \R^n$. Let $T > 0$ be the total number of iterations. Let $\bar{u}_{\nn}(t) \in \R^n$ and $u^* \in \R^n$ be the training data predictors defined in Definition~\ref{def:f_nn_lev} and Definition~\ref{def:krr_ntk} respectively. Let $\kappa=1$ be the corresponding multiplier. Given any accuracy $\epsilon\in(0,1)$, if $\kappa = 1$, $T=\wt{O}(\frac{1}{\Lambda_0})$, network width $m \geq \wt{O}(\frac{n^4d}{\lambda_0^4\epsilon})$ and regularization parameter $\lambda \leq \wt{O}(\frac{1}{\sqrt{m}})$, then with probability at least $1-\delta$ over the random initialization, we have
\begin{align*}
	\|\bar{u}_{\nn}(T) - u^*\|_2 \leq \epsilon.
\end{align*}
Here $\wt{O}(\cdot)$ hides $\poly\log(n/(\epsilon \delta \Lambda_0 ))$.
\end{theorem}
\begin{proof}
Combining results of Lemma~\ref{lem:u*_minus_bar_u*} and Corollary~\ref{cor:train_lev} using triangle inequality, we finish the proof.
\end{proof}
\begin{remark}
Despite our given upper-bound of network width under leverage score sampling is asymptotically the same as the Gaussian initialization, we point out the potential benefits of introducing leverage score sampling to training regularized neural networks. 

Note the bound for the width consists of two parts: 1) initialization and 2) training. Part 1, requires the width to be large enough, so that the initialized dynamic kernels $H(0)$ and $\ov{H}(0)$ are close enough to NTK by concentration, see Lem~\ref{lem:epsilon_init} and~\ref{lem:u*_minus_bar_u*}. Part 2, requires the width to be large enough, so that the dynamic kernels $H(t)$ and $\ov{H}(t)$ are close enough to the NTK during the training by the over-parameterization property, see Lem~\ref{lem:induction} and~\ref{lem:induction_lev}. Leverage score sampling optimizes the bound for part 1 while keeping the bound for part 2 the same. The current state-of-art analysis gives a tighter bound in part 2, so the final bound for width is the same for both cases. If analysis for part 2 can be improved and part 1 dominates, then initializing using leverage score will be beneficial in terms of the width needed.
\end{remark}

% \subsection{other}

% \begin{remark}
% When probability density function $q(\cdot)$ in Theorem~\ref{thm:main_equivalence} is $p(\cdot))$, the pdf of $\N(0,I_d)$. We prove the equivalence between training NN with regularization and neural tangent kernel regression in the normal case without importance sampling.

% When probability density function $q(\cdot)$ in Theorem~\ref{thm:main_equivalence} is the leverage score $q_\lambda(\cdot)$ defined in Definition~\ref{def:leverage_score}. 
% \end{remark}

  %%% Section D. Equivalence between training neural network with regularization and kernel ridge regression under leverage score sampling
%\newpage
\section{Extension to other neural network models}\label{sec:gen_dnn}
In previous sections, we discuss a simple neural network model: 2-layer ReLu neural network with first layer trained. We remark that our results can be naturally extended to multi-layer ReLU deep neural networks with all parameters training together. 

Note the core of the connection between regularized NNs and KRR is to show the similarity between their gradient flows, as shown in Corollary~\ref{cor:ntk_gradient} and Corollary~\ref{cor:nn_gradient}: their gradient flow are given by 
\begin{align*}
    \frac{\d u_{\ntk,\test}(t)}{\d t}&= - \kappa^2 \k_{\ntk}(x_{\test},X)^\top(u_{\ntk}(t)-Y)-\lambda u_{\ntk,\test}(t)\\%\label{eq:320_1_ex}\\
    \frac{\d u_{\nn,\test}(t)}{\d t}&= - \kappa^2 \k_{t}(x_{\test},X)^\top(u_{\nn}(t)-Y)-\lambda u_{\nn,\test}(t) %\label{eq:320_2_ex}
\end{align*}
Note these gradient flows consist of two terms: the first term $- \kappa^2 \k(x_{\test},X)^\top(u(t)-Y)$ comes from the normal neural network training without $\ell_2$regularization, the second term $-\lambda u_{\test}(t)$ comes from the regularizer and can be directly derived using the piece-wise linearity property of the 2-layer ReLu NN (in this case, with respect to the parameters in the first layer).

Now consider the case of training multi-layer ReLu neural network with regularization. We claim above similarity between the gradient flows of NN and KRR still holds as long as we scale up the network width by the number of layers trained: as 1) the similarity of the first term $- \kappa^2 \k(x_{\test},X)^\top(u(t)-Y)$ has already been shown in previous literature \cite{adhlsw19,als19b}, and 2) the similarity of the second term $-\lambda u_{\test}(t)$ comes from the piece-wise linearity property of deep ReLu neural network with respect to all training parameters. In the common case where we train all the parameters together, the equivalence still holds as long as we scale up the network width by the number of layers, as shown in the following theorem:
\begin{theorem}\label{thm:dnn}
Consider training a $L$-layer ReLU neural network with $\ell_2$ regularization. Let $u_{\nn,\test}(t)$ denote the neural network predictor at time $t$, and $u_{\test}^*$ denote the kernel ridge regression predictor. Then Given any accuracy $\epsilon\in(0,1/10)$ and failure probability $\delta\in(0,1/10)$.
Let multiplier $\kappa = \poly(\epsilon,\Lambda_0,1/n,1/L)$, number of iterations $T=\poly(1/\epsilon,1/\Lambda_0,n,L)$, network width $m \geq \poly(n,d,1/\epsilon,1/\Lambda_0, L)$ and regularization parameter $\lambda \leq \poly(1/n,1/d,\epsilon,\Lambda_0,1/L)$. Then with probability at least $1-\delta$ over random initialization, we have
\begin{align*}
\| u_{\nn,\test}(T) - u_{\test}^* \|_2 \leq \epsilon.
\end{align*}
Here we omit $\poly\log(n/(\epsilon \delta \Lambda_0 ))$ factors.
\end{theorem} 
The results under leverage score sampling can be argued in the same way.

We also remark that it is possible to extend our results further to the model of convolutional neural network (CNN) by making use the convolutional neural tangent kernel (CNTK) discussed in \cite{adhlsw19}, and to the case using stochastic gradient descent in training rather than gradient descent. However, these discussion require more detailed proof and is out of the scope of this work.  

 %%% Section E. Extension to other neural network models

\newpage
\addcontentsline{toc}{section}{References}
\ifdefined\isarxiv
\bibliographystyle{alpha}
\else
\bibliographystyle{plain}%{alpha}
\fi
\bibliography{ref}
\newpage

\end{document}